\documentclass[twoside,11pt]{article}

\usepackage[abbrvbib, preprint]{jmlr2e}

\date{}
\usepackage[utf8]{inputenc}\usepackage{macros}
\usepackage{graphicx}

              \DeclareMathOperator{\sgn}{sign}

\usepackage{chngcntr}
\usepackage{apptools}
\AtAppendix{\counterwithin{lemma}{section}}

\usepackage{lastpage}
\jmlrheading{?}{?}{1-\pageref{LastPage}}{2/26; Revised ?/?}{?/?}{?}{Yingjie Wang, Mokhtar Z. Alaya, Salim Bouzebda and Xinsheng Liu}

\ShortHeadings{Learning High-Dimensional Heavy-Tailed Locally Stationary Time Series}{Wang, Alaya, Bouzebda, and Liu}
\firstpageno{1}

\begin{document}

\title{Sparsified-Learning for High-Dimensional Heavy-Tailed Locally Stationary Time Series, Concentration and Oracle Inequalities}

\author{\name Yingjie Wang \email wyj202007@nuaa.edu.cn \\
       \addr Faculty of Mathematics and Physics, Huaiyin Institute of Technology, Huai'an, China
       \AND
       \name Mokhtar Z. Alaya \email alayaelm@utc.fr \\
       \addr Université de Technologie de Compiègne, LMAC (Laboratoire de Mathématiques Appliquées de Compiègne), CS 60 319 - 60 203 Compiègne Cedex, France
       \AND
        \name Salim Bouzebda \email salim.bouzebda@utc.fr \\
       \addr Université de Technologie de Compiègne, LMAC (Laboratoire de Mathématiques Appliquées de Compiègne), CS 60 319 - 60 203 Compiègne Cedex, France
       \AND
       \name Xinsheng Liu \email xsliu@nuaa.edu.cn \\
       \addr Faculty of Mathematics and Physics, Huaiyin Institute of Technology, Huai'an, China
       }

\editor{My editor}

\maketitle

\begin{abstract}Sparse learning is ubiquitous in many machine learning tasks. It aims to regularize the goodness-of-fit objective by adding a penalty term to encode structural constraints on the model parameters. In this paper, we develop a flexible sparse learning  framework tailored to high-dimensional heavy-tailed locally stationary time series (LSTS). The data-generating mechanism incorporates a regression function that changes smoothly over time and is observed under noise belonging to the class of sub-Weibull and regularly varying distributions. We introduce a sparsity-inducing penalized estimation procedure that combines additive modeling with kernel smoothing and define an additive kernel-smoothing hypothesis class. In the presence of locally stationary dynamics, we assume exponentially decaying $\beta$-mixing coefficients to derive concentration inequalities for kernel-weighted sums of locally stationary processes with heavy-tailed noise. We further establish nonasymptotic prediction-error bounds, yielding both slow and fast convergence rates under different sparsity structures, including Lasso and total variation penalization with the least-squares loss. To support our theoretical results, we conduct numerical experiments on simulated LSTS with sub-Weibull and Pareto noise,  highlighting how tail behavior affects prediction error across different covariate-dimensions as the sample size increases.
\end{abstract}


\section{Introduction} \label{sec:introduction}

Sparsified learning is an innovative approach that combines the principles of sparse learning and adaptive modeling to address the challenges posed by high-dimensional and complex datasets, for instance, see~\citep{Tibshirani1996, Yuan2006, Zou2005} among many others. It aims to capture the essential patterns and relationships within the data while promoting sparsity, interpretability, and computational efficiency~\citep{Fan2001, Koltchinskii2011, Negahban2011, Klopp2017}. 
The key idea behind sparsified learning is to identify and select a sparse subset of relevant features or variables that significantly impact the target variable. 
The resulting model becomes simpler and more interpretable by emphasizing sparsity while reducing overfitting and improving generalization performance. However, many times series exhibit non-stationary behaviors, which exist in many application fields, including finance \citep{Tanaka2017}, economics \citep{VOGT2012}, and environmental science \citep{Matsuda2018, Petronio2020}.

Locally stationary time series  (LSTS) are a class of stochastic processes that exhibit variation over time while maintaining relative stability within short time intervals \citep{DAHLHAUS2012}. This characteristic makes them valuable in time series analysis, particularly in tasks such as modeling and forecasting. LSTS offer a more precise framework for modeling time series data than stationary processes. They excel at capturing time-varying phenomena and can be estimated with greater efficiency. 

Heavy-tailed time series refers to time series data that exhibit extreme values or outliers that occur more frequently than expected under a normal distribution \citep{MR4174389}. The heavy-tailed behavior can be linked to a range of real-world factors, including financial market crashes, natural disasters, and specific social phenomena characterized by rare yet significant events. 
Data with heavy tails have been collected in many application fields, including economics \citep{Malevergne2006}, environment \citep{Reiss2001}, biology \citep{James2015}, and so on. Heavy-tailed time series affects the prediction accuracy, mainly because extreme events or outliers are predicted more frequently than under normal distributions \citep{Robert1998}.
This work aims to develop a new method to solve the challenges posed by heavy-tailed and locally stationary behavior in time series data. Using sparsity techniques to deal with heavy-tailed behaviors, this paper aims to significantly improve the efficiency and accuracy of modeling these complex data structures, thereby advancing the latest techniques in the specific field of time series analysis. 

\paragraph{Related works.} 
Sparse learning methods are used in regression models to handle high-dimensional data with many features, where most of the features are irrelevant or redundant.
The Lasso~\citep{Tibshirani1996} estimator is a regression technique that induces sparsity in the model by adding an $\ell_{1}$ penalty to the loss function. It is a convex relaxation of best subset selection, and it can be used to perform variable selection and regularization to enhance the prediction accuracy and interpretability of the resulting statistical model \citep{Xia2014,Norouzirad2018}. 
Total variation (TV) penalization can also be employed for sparse learning in stationary time series. TV penalization is a type of $\ell_{1}$-penalization that encourages the sparsity of the gradient of the signal \citep{Eickenberg2015}. By minimizing the total variation of the signal, TV penalization encourages the sparsity of the gradients, which in turn promotes sparsity in the solution \citep{Belilovsky2015, Li2020}. \cite{Baraud2001} studied the problem of estimating the unknown regression function in a $\beta$-mixing dependent framework. They build a penalized least squares estimator on a data-driven selected model with a nonnegative penalty function. 
Although the aforementioned studies have demonstrated encouraging outcomes, the techniques proposed in these studies require stringent assumptions about the stochastic process, specifically, assuming it to be a stationary process and linear regression. 

LSTS help analyze and forecast time series data that exhibit changing statistical properties. It provides a more flexible and realistic representation of the data than assuming global stationarity.  
The nonparametric models with a time-varying regression function and locally stationary covariates proposed by \cite{VOGT2012} and the asymptotic theory of nonparametric regression for a locally stationary functional time series studied in \cite{Daisuke2022}. In \cite{Dahlhaus2019}, some general theory is presented for locally stationary processes
based on the stationary approximation and the stationary derivative. A two-step estimation method that borrows the strengths of spline smoothing and the local polynomial smoothing method is developed by \cite{Hu2019} for a locally stationary process. The aforementioned works rely on the assumption of strictly independent and identically distributed (i.i.d.) tail behavior for their analysis. This significantly restricts the practical applicability of the developed theoretical results. In reality, many practical learning scenarios involve heavy-tailed data that occur naturally.

Various studies dealt with statistical learning with samples drawn from some heavy-tailed data. \cite{Wong2020} studied the (strict) stationarity to establish lasso guarantees for heavy-tailed time series. \cite{Abhishek2019} we establish risk bounds for the empirical risk minimization (ERM) applicable to data-generating processes that are both dependent and heavy-tailed. \cite{Sagnik2022} studied the optimal sparse estimation of high-dimensional heavy-tailed
time series. \cite{Takeyuki2022} considered sparse estimation of linear regression coefficients when covariates and noises are sampled from heavy-tailed distributions. 

\paragraph{Contributions.}
We propose a sparse learning framework for high-dimensional LSTS with {heavy-tailed} innovations, allowing both sub-Weibull and regularly varying (Pareto-type) tails under $\beta$-mixing dependence. The framework incorporates $\ell_1$ (Lasso) and weighted total variation penalties, enabling the flexible modeling of standard sparsity, structured sparsity, and block sparsity in non-stationary regression settings. We establish non-asymptotic concentration inequalities for kernel-weighted sums of locally stationary $\beta$-mixing processes with heavy-tailed noise, explicitly capturing localization effects, dependence decay, and tail behavior. Using the proposed concentration results, we derive slow-rate oracle inequalities for both penalties without restricted eigenvalue assumptions, quantifying the impact of local stationarity on convergence rates. Under suitable restricted eigenvalue assumptions, we further obtain fast-rate oracle inequalities for estimation and prediction errors in high-dimensional settings with heavy-tailed and dependent data. Proximal algorithms are used for efficient computation, and numerical experiments demonstrate the finite-sample performance of the proposed estimators across different tail regimes and sparsity patterns.
Together, these results provide a theoretical justification for sparse estimation in locally stationary models beyond the sub-Gaussian or strictly stationary regimes.

\paragraph{Layout of the paper.}

The structure of the paper is as follows. Section \ref{sec:preliminaries} presents the preliminaries of locally stationary processes and heavy-tailed distributions. In Section \ref{sub:estimation_procedure_minimization_of_penalized_empirical_risk}, we develop a sparse penalized estimation procedure. Section \ref{sub:least_square_loss} proposes concentration inequalities for locally stationary $\beta$-mixing heavy-tailed random variables.
Section \ref{sec.Oracle inequality} provides non-asymptotic oracle inequalities for different types of sparsity.
Section \ref{sec:numerical_experiments} illustrates the numerical experiments by proximal algorithms and experimental protocol. Finally, Section \ref{sec:conclusion} concludes the paper, highlighting the contributions of our work and discussing potential directions for future research.

\paragraph{Notation.}
The set $\R_+$ denotes the non-negative real numbers. 
For every $q>0$, we denote by $\norm{x}_q$ the usual $\ell_q$ norm of a vector $x\in \R^d$, namely $\norm{x}_q = (\sum_{j=1}^d|x_j|^q)^{1/q}$, and $\norm{x}_\infty = \max_{1\leq j\leq d}|x_j|$.
We also denote $\norm{x}_0 = |\{j: x_j \neq 0\}$, where $|A|$ stands for the cardinality of a finite set $A$.  We denote $A^\complement$ for the complement of a set $A$. 
For any $u \in \mathbb{R}^{d}$ and any $L \subset \{1, \ldots, d\}$, we denote $u_L$ as the vector in $\mathbb{R}^d$ satisfying $(u_L)_k = u_k$ for $k \in L$ and $(u_L)_k = 0$ for $k \in L^{\complement} = \{1, \ldots, d\} \setminus L$.
We write $\mathbf{1}$ (resp. $\mathbf{0}$) the vector having all coordinates equal to one (resp. zero).
We denote $\ind{}(\cdot)$ the indicator function taking the value $1$ if the condition in $(\cdot)$ is satisfied and $0$ otherwise.
For a real-valued random variable $S$, we use the notation $S^\tau$ to denote the truncated version of the random variable $S$, i.e., $S^\tau =S \ind{(S \leq \tau)}$. Finally, we denote by $\sgn(x)$ the set of sub-differentials of the function $x \mapsto |x|$, namely $\sgn(x) = \{ 1\}$ if $x > 0$, $\sgn(x) = \{ -1 \}$ if $x < 0$ and $\sgn(0) = [-1, 1]$.

\section{Background on heavy-tailed LSTS} \label{sec:preliminaries} 
Let $\{Z_{t,T}\}_{t=-\infty}^{\infty}$ be a stochastic time series with time index $t$ such that ${Z_{t, T}=(X_{t, T}^\top,Y_{t, T}})$, where the variable $X_{t, T}=(X_{t, T}^{1}, \ldots, X_{t, T}^{d})^\top$ is a $d$-vector $(d\geq1)$ of covariates and takes values in the compact input space $ \mathcal{X} \subseteq \mathbb{R}^{d}$ and $Y_{t, T}$ belongs to the output space $\mathcal{Y} \subseteq \mathbb R$. Define $\mathcal{F}$ be the set of measurable functions mapping from $[0,1]\times \mathcal{X}$ to $\mathcal{Y}$. We consider the nonparametric model
\begin{equation}
\label{reg_model}
 Y_{t, T}=m^{\star}\big(\frac{t}{T}, X_{t, T}\big)+\varepsilon_{t, T},\quad \text{ for } t=1, \ldots,T,
\end{equation} 
where $m^\star(\cdot, \cdot) \in \mathcal{F}$ stands for the conditional mean regression function of $Y_{t,T}|X_{t, T}$, that depends on the time and space directions. The model variables are assumed to be locally stationary processes (see Definition~\ref{def:locallystseries}). 

As usual in the literature on locally stationary processes, the regression function $m^\star$ does not depend on real-time $t$ but rather on a rescaled time $u = \frac tT$ (see~\cite{VOGT2012,DAHLHAUS2012}).
In the following, we recall some background information on locally stationary processes. LSTS are a fundamental concept in time series analysis and statistical modeling, providing a framework for understanding how the statistical properties of time series vary over time or across data segments. This concept is essential when dealing with time series data that exhibit non-stationary behavior.

\subsection{Locally stationary time series} \label{sub:background_on_locally_stationary_process}
We consider non-stationary processes with dynamics that change slowly over time and may thus behave as stationary at a local level.
For example, consider a continuous function $m:[0,1] \rightarrow$ $\mathbb{R}$ and a sequence of i.i.d. random variables $(\varepsilon_t)_{t\in \mathbb{N}}$. The stochastic process $X_{t, T}=m(t / T)+\varepsilon_t$, $t \in\{1, \ldots, T\}, T\in \mathbb{N}$ can be expected to behave "almost" stationary for $t \in$ $\{1, \ldots, T\}$ close to $t^*$, for some $t^* \in\{1, \ldots, T\}$, as in this case $m(t^* / T) \approx m(t / T)$, but this process is not weakly stationary. A more realistic concept that allows this kind of change is called local stationarity and was first introduced by \cite{MR1429916}, who approximated the spectral representation of the underlying stochastic process locally.
\begin{definition} [Locally stationary time series, see~\cite{VOGT2012}]
\label{def:locallystseries}
The time series $\{X_{t, T}\}_{t=1}^{T}$ is locally stationary if for each rescaled time point $u \in [0, 1]$ there exists an associated process $\{X_{t}(u)\}_{t \in \mathbb{Z}}$ with the following two properties:
\begin{itemize}
\setlength\itemsep{-0.1cm}
\item [(i)] $\{X_{t}(u)\}_{t \in \mathbb{Z}}$ is strictly stationary;
\item [(ii)] It holds that
\begin{equation}
\norm{X_{t, T}-X_{t}(u)} \leq\big(\big|\frac{t}{T}-u\big|+\frac{1}{T}\big) U_{t, T}(u) \quad  a.s., \label{equation2.2}
\end{equation}
where $U_{t, T}(u)$ is a process of positive variables satisfying 
$ \mathbb{E}[(U_{t, T}(u))^{\rho}]< C$
for some $ \rho > 0$ and $ C < \infty$ independent of $u, t $, and $T.$ Here, $\norm{\cdot}$denotes an arbitrary norm on $\mathbb{R}^{d}.$
\end{itemize}
\end{definition}

\begin{remark}\label{remark.U_{t,T}}
Since $\rho$th moments of the variables $U_{t, T}(u)$ are uniformly bounded, it holds that $U_{t, T}(u)=\bigO_{\mathbb{P}}(1)$, then we have
\begin{equation*}
\norm{X_{t, T}-X_{t}(u)} \leq \bigO_{\mathbb{P}}\big(\big|\frac{t}{T}-u\big|+\frac{1}{T}\big).
\end{equation*}   
\end{remark}

\begin{example}
A example of Locally stationary process is a time-varying autoregressive process, denoted as $tvAR(d)$ \citep{DAHLHAUS2012} and defined by
\begin{equation*}
Y_{t, T}=X_{t, T}+\sum_{j=1}^{d} m_{j}\big(\frac{t}{T}\big) X_{t-j, T}-\varepsilon_{t,T}, \quad t \in \mathbb{Z},
\end{equation*}
where $m_{j}\big(\frac{t}{T}\big) $ follows the curves $ m_{j}(\cdot):[0,1] \rightarrow(-1,1) $, $m_{j}(u)=m_{j}(0)$ for $u<0$ and $m_{j}(u)=m_{j}(1)$ for $u>1$. In a certain neighborhood, there exists a stationary process denoted as $ X_{t}(u_{0}) $ with a fixed time point $ u_{0}=t_{0} / n $ that satisfies the equation (\ref{equation2.2}). The stationary process $ X_{t}(u_{0})$ defined by 
\begin{equation*}
y_{t}(u_{0})=X_{t}(u_{0})+\sum_{j=1}^{d} m_{j}(u_{0}) X_{t-j}(u_{0})-\varepsilon_{t,T}, \quad t \in \mathbb{Z}.
\end{equation*}
\end{example}

Dealing with LSTS needs more conditions to get the theoretical guarantee. One of the most popular is the mixing condition, including $\alpha$-mixing and $\beta$-mixing, which are important concepts in statistics, particularly in the context of time series analysis and stochastic processes. Mixing conditions are used to characterize the dependence structure of random variables. It describes how quickly the dependence between observations decays with increasing time intervals~\citep{Rosenblatt1956}.

\begin{definition}[Mixing condition, see~\cite{bradley2005}]
\label{def:mixingcondition}
Let $ (\Omega, \mathcal{F}, P) $ be a probability space, and let $\mathcal{A}$, $\mathcal{B} $ be subfields of $ \mathcal{F} $. Let $\beta(\mathcal{A}, \mathcal{B})=\mathbb{E} \sup_{B \in \mathcal{B}}|\mathbb{P}(B)-\mathbb{P}(B \mid \mathcal{A})|,$
for an array $ \{Z_{t, T}:1\leq t \leq T\} $, define the coefficients
\begin{equation*}
\beta(k)=\sup_{t, T:1\leq t \leq T-k} \beta\big(\sigma(Z_{s, T},1\leq s \leq t), \sigma(Z_{s, T}, t+k \leq s \leq T)\big), 
\end{equation*}
where $ \sigma(Z) $ is the $\sigma$-field generated by $ Z $. The array $\{Z_{t, T}\} $ is said to be $ \beta $-mixing if $ \beta(k) \rightarrow 0$ as $ k \rightarrow \infty $.
\end{definition}

The coefficient $\beta(k)$ quantifies the level of dependence among events taking place within a span of $k$ time units. It introduces a temporal dependence structure that diminishes over time \citep{Wong2020, Anatolyev2020,bradley2005}. The $\beta $-mixing condition is a valuable tool in analyzing non-stationary time series data within the fields of statistics and machine learning \citep{Kuznetsov2018}.

\subsection{Heavy-tailed distribution} \label{sub:backgrounf_on_heavy_tailed_disitrbution}

A distribution is considered to be heavy-tailed if it has a heavier tail than any exponential distribution~\citep{Nair2022}. In this section, we present two types of tail distribution functions. Let us start with the definition of the tail-capturing distribution.
\begin{definition}[Tail-capturing distribution]
\label{def:tailcaptures}
Let $\textsc{I}: \R \rightarrow \R_+$ denote an increasing and continuous function with the property $\textsc{I}(\nu)=\bigO(\nu)$ as $\nu \to\infty $. We say $\textsc{I}$ captures the tail of random variable $H$ if 
\begin{equation*}
\mathbb{P}[|H| >\nu] \leq \exp (-\textsc{I}(\nu)), \quad \textrm{\rm for all}\quad \nu >0.
\end{equation*}
\end{definition}
Note that $\textsc{I}(\nu)$ can be a generic function. Clearly, $\textsc{I}_{br}(\nu) = -\log(\mathbb{P}[H > \nu])$ and $\textsc{I}_{bl}(\nu) = -\log(\mathbb{P}[H < -\nu])$ capture, respectively, the right tail and the left tail for any random variable $H$, and they are called the basic rate capturing function. 
It is simple to remark that if $H$ is a right heavy-tailed random variable, then $-H$ is left heavy-tailed. 
It is convenient to approximate the basic tail-capturing function $\textsc{I}$~\citep{Bakhshizadeh2023}. We detail an example of $\textsc{I}$ that are popular in application areas.

\begin{figure}[htbp]
\centering
\begin{minipage}{0.5\textwidth}
\centering
\includegraphics[width=0.9\textwidth]{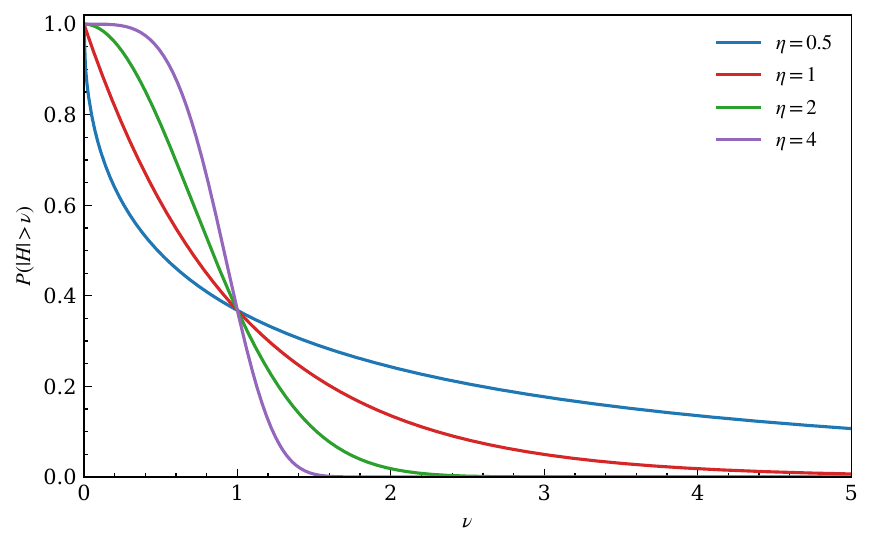}
\captionsetup{font=scriptsize}
\caption{Sub-Weibull distributions with $C=1$.}
\end{minipage}\hfill
\begin{minipage}{0.5\textwidth}
\centering
\includegraphics[width=0.9\textwidth]{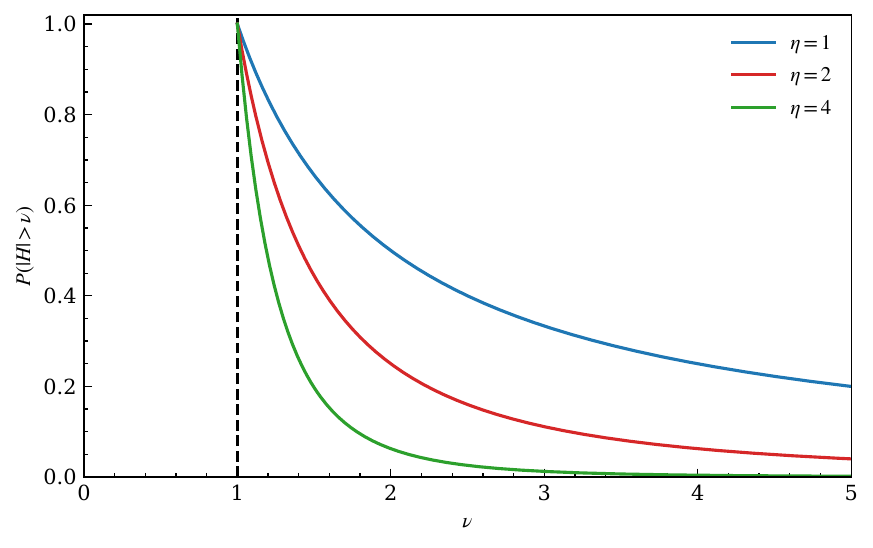}
\captionsetup{font=scriptsize}
\caption{Pareto distributions with $u=1$.}
\end{minipage}
\end{figure}

\begin{example}[Sub-Weibull distribution]\label{example3-(i)}
If $\textsc{I}(\nu) =(\frac{\nu}{C})^{\eta}$ for some $\eta >0$ and $C$ is a constant depending only on $\eta$, $H$ follows {\it sub-Weibull} distribution with the constant $(\eta,C)$, i.e., $\mathbb{P}[|H| >\nu] \leq \exp (-\big( \nu/C \big)^{\eta})$, for all $\nu \geq 0$. 
\end{example}

The tail decay of the sub-Weibull distribution is exponential, and the rate of decay of the sub-Weibull is controlled by the parameter $\eta$, making it possible to describe the tail behavior between the light tail and the extremely heavy tail. 
Sub-Gaussian and sub-exponential distributions are special cases of the sub-Weibull distribution with $\eta=2$ and $\eta=1$, respectively. 
An illustration of sub-Weibull distributions is represented in Figure 1 for different
values of the tail parameter $\eta$. We can see that the smaller $\eta$ the heavier the tail. Sub-Weibull distribution has important applications in many high-dimensional statistics and machine learning fields, especially when dealing with data with heavy-tailed rows.

Next, we introduce another class of heavy-tailed distributions, regularly varying heavy-tailed distributions, which are characterized by a power decay of the tail, that is, their tail probability decays more slowly, much slower than exponential distributions.

\begin{definition} [Regularly varying tail distribution]
\label{def:Regularvaryingtail}
The distribution function of the random variable H has a regularly varying tail with index $\eta>0$, if  
\begin{align*}
\mathbb{P}[|H|>\nu] = \nu^{-\eta} \textsc{L}(\nu),
\end{align*}
where $\textsc{L}(\nu)$ is a slowly varying function at $\infty$, i.e., for all $t > 0$, $L(t\nu) \sim L(\nu)$.
\end{definition}

The distribution function of $ H $ has a regularly varying right tail if $\mathbb{P}[H>\nu] = \nu^{-\eta} \textsc{L}(\nu)$. Similarly, the distribution function $H$ has a regularly varying left tail of index $ -\eta $, if $\mathbb{P}[H<-\nu] = \nu^{-\eta} \textsc{L}(\nu)$. 
\begin{example}[Pareto distribution] \label{example4-(ii)}
If the distribution of $H$ has regularly varying tail with index $-\eta$ and $\textsc{L}(\nu)=u^{\eta}$, for some $u >0$ then $H$ follows the Pareto distribution with tail index
\begin{align*}
\mathbb P[|H|> \nu ]=
\begin{cases}
(u / \nu)^{\eta } & \text { for } \quad \nu \geq u, \\ 0 & \text { for } \quad \nu<u.
\end{cases}
\end{align*}
\end{example}

Figure 2 illustrates the Pareto distribution with different parameter values $\eta$. It's worth noting that the parameter $\eta$ plays a pivotal role in determining the tail behavior of the distribution. Higher values of $\eta$ correspond to distributions with lighter tails, while lower values of $\eta$ yield distributions with heavier tails.

\section{Sparse penalized estimation procedure} \label{sub:estimation_procedure_minimization_of_penalized_empirical_risk}

We set $\mathcal{M}$ to be the  additive data-dependent hypothesis space defined by
\begin{equation*}
\mathcal{M} =\Big\{ m_{\bm{\theta}}(u, x)= \sum_{r=1}^{T}\sum_{j=1}^{d} \theta_{r, j} K_{h,1}\big(u-\frac{r}{T}\big)K_{h,2}(x^{j}-X_{r, T}^{j})  \Big\},
\end{equation*}
where 
\begin{align*}
\bm{\theta} = (\theta_{1\bullet}^\top, \ldots, \theta_{T\bullet}^\top)^\top = 
\big((\theta_{1,1},\ldots,\theta_{1,d}),(\theta_{2,1},\ldots,\theta_{2,d}),\ldots,(\theta_{T,1},\ldots,\theta_{T,d})\big)^{T} \in \mathbb{R}^{T d}.
\end{align*}
Here, $K_{h,i}(\cdot)$ is a scaled kernel function with a bandwidth $h>0$, and $K_{h,i}(v)=K_{i}(\frac{v}{h})$ with basic kernel $K_{i}(\cdot)$ for $i=1,2$. The parameter space $\mathcal{M}$ is a given subset of $\mathcal{F}$ and represents a data-dependent hypothesis space used for statistical modeling, especially in the framework of additive models. This space combines elements from both additive models and kernel methods, and it's linear in the parameter vector $\bm{\theta}$. The presence of two kernels with respect to the time ($K_{h,1}$) and space ($K_{h,2}$) directions is convenient to estimate the ground truth conditional mean function $m^\star$. Since the process is locally stationary, we give much attention to its information in a local bandwidth $h$ depending on the sample size $T,$ namely $h= h(T)$. For that reason, we shall appropriately choose the two kernels. 
Note that each candidate estimator $m_{\bm{\theta}} \in \mathcal{M}$ can be  expressed as \begin{align*}
m_{\bm{\theta}}(u, x)=\sum_{j=1}^{d}m_{\bm{\theta}}^{j}(u, x) \text{ where } m_{\bm{\theta}}^{j}(u, x)=\sum_{r=1}^{T}\theta_{r, j} K_{h,1}\big(u-\frac{r}{T}\big)K_{h,2}(x^{j}-X_{r, T}^{j}).  
\end{align*}
We can then have an additive estimator structure. The additive models for LSTS can effectively capture the dynamic feature of the regression function~\citep{Hu2019,Wang2022HuberAM}. In this work, we consider the penalized empirical risk with a square loss function $\ell$ defined on $[0,1] \times \mathcal{X}\times \mathcal{Y}$ and penalty regularization $\Omega:\R^{T d} \mapsto \R_+$ on $\bm{\theta}$. 
We first define the empirical risk as 
\begin{equation*}
R_{\rm emp}(m_{\bm{\theta}}) = \frac{1}{T} \sum_{t=1}^{T} \ell \big(m_{\bm{\theta}}\big(\frac{t}{T}, X_{t, T}\big), Y_{t, T} \big).
\end{equation*}

We now give the following definition.
\begin{definition}
\label{def.6}
The penalized empirical risk minimization of $m^\star$ writes as $\hat{m} = m_{\hat{\bm{\theta}}}$, where 
\begin{equation*}
\hat{\bm{\theta}} =  (\hat \theta_{1\bullet}^\top, \ldots, \hat\theta_{T\bullet}^\top)^\top \in \argmin_{\bm{\theta} \in \mathbb{R}^{T \times d}}\big\{R_{\rm emp}(m_{\bm{\theta}}) + \Omega_\lambda(\bm{\theta})\big\}.
\end{equation*}
\end{definition}
The hyper-parameter $\lambda>0$ controls the trade-off between the goodness-of-fit $R_{\rm emp}$ and the constraints on the learned parameter $\bm{\theta}$ through the penalization $\Omega(\bm{\theta})$, which leads to incorporating sparsity structure on $\bm{\theta}$. For Lasso penalization
\begin{align*}
\Omega_\lambda (\bm{\theta}) = \lambda \norm{\bm{\theta}}_{1}= \sum_{r=1}^{T}\norm{\theta_{r\bullet}}_{1} = \sum_{r=1}^{T}\sum_{j=1}^{d}|\theta_{r, j}|
\end{align*}
and weighted total variation, for $\lambda=(\lambda_{1},\cdots,\lambda_{d})\in \R^{d}_{+}$,
\begin{align*}
\Omega_{\lambda}(\bm{\theta})=\norm{\bm{\theta}}_{\TV,\lambda}=\sum_{r=1}^{T}\norm{\theta_{r\bullet}}_{\TV,\lambda}=\sum_{r=1}^{T}\sum_{j=2}^{d}\lambda_{j}|\theta_{r, j}-\theta_{r, (j-1)}|.
\end{align*}
\paragraph{Square loss.} 
Hereafter, we consider the squared loss function, $\ell(z)=z^{2},$ and thus define the penalized empirical risk minimizer of $m^\star$ writes as $\hat{m} = m_{\hat{\bm{\theta}}}$, where 
\begin{equation}
\label{explicit_estimator}
\hat{\bm{\theta}} \in \mathop{\arg\min}\limits_{\bm{\theta} \in \mathbb{R}^{T d}}\frac{1}{T} \sum_{t=1}^{T} \Big(Y_{t, T}- \sum_{r=1}^{T}\sum_{j=1}^{d} \theta_{r, j} K_{h,1}\big(\frac{t}{T}-\frac{r}{T}\big) K_{h,2}(X_{t,T}^{j}-X_{r, T}^{j})\Big)^{2} +\Omega_\lambda (\bm{\theta}).
\end{equation}
Let $\bY = (Y_{1,T}, \ldots, Y_{T,T})^\top \in \R^T$ and $\bK$ be the $T \times (T \times d)$ matrix such that for $t \in \{1,\ldots,T\}$ and $(b,j) \in \{1,\ldots,T\} \times \{1,\ldots,d\}$, $\bK=(K_{1\bullet}, \dots,K_{T\bullet})^\top$ and the $\{t,b,j\}$-element of $\bK$ is 
\begin{equation*}
  K_{t,b,j} = K_{h,1}\big(\frac{t}{T}-\frac{b}{T}\big)K_{h,2}(X_{t,T}^{j}-X_{b, T}^{j}).
\end{equation*}
Setting $\bM^\star =  \big(m^{\star}\big(\frac 1T, X_{1, T}\big), \cdots, m^{\star}\big(1, X_{T, T}\big)\big)^\top \in \mathbb{R}^T$ and $\beps = (\varepsilon_{1,T}, \ldots, \varepsilon_{T,T})^\top \in \mathbb{R}^T$, we have $\bY=\bM^{\star}+\beps$.
Let the empirical risk $R_{\rm emp}(m_{\bm{\theta}})=R_{T}(\cdot)$ defined for all $\bm{\theta} \in \mathbb{R}^{Td}$, such that $R_{T}(\bm{\theta})= \frac{1}{T} \norm{\bY -\bK\bm{\theta}}_{2}^{2}.$
Then problem~(\ref{explicit_estimator}) can be written as follows 
\begin{equation}
\label{estimator}
\hat{\bm{\theta}} = \mathop{\arg\min}\limits_{\bm{\theta} \in \mathbb{R}^{Td}} \big\{R_{T}(\bm{\theta})+ \Omega_\lambda (\bm{\theta})\big\}.    
\end{equation}
We provide bounds for the generalization error 
\begin{align*}
R(\hat m, m^\star) = \mathbb{E}\Big[\frac 1T \sum_{t=1}^T \big(\hat m\big(\frac tT, X_{t,T}\big) - m^\star\big(\frac tT, X_{t,T}\big)\big)^2\Big].
\end{align*}

\paragraph{Block sparsity.}
For all $\bm{\theta} \in \R^{Td}$, let $J(\bm{\theta})=\{J_{1},\ldots, J_{T}\}$ be the concatenation of the support sets, for the Lasso penalization, we define, for $ r \in \{1,\cdots, T\}$,
\begin{equation*}\label{equ_Lasso_ridge_spares}
J_r=J_r(\theta_{r\bullet}) = \{ j\in\{1, \ldots, d\}: \theta_{r, j} \neq 0\}, 
\end{equation*}
and for TV penalization,
\begin{equation*}\label{equ_TV_spares}
J_r=J_r(\theta_{r\bullet}) = \{ j\in\{2, \ldots, d\}: \theta_{r, j} \neq \theta_{r, (j-1)}\}.
\end{equation*}
Similarly, we set $J^{\complement}(\bm{\theta})=\{J^{\complement}_{1},\ldots,J^{\complement}_{T}\}$ be the complementary of $J(\bm{\theta})$. The cardinality of $J_r$, $|J_r|$, characterizes the sparsity of the vector $\theta_{r\bullet}$, namely the small $|J_r|$, the "sparser" $\theta_{r\bullet}$.
The value $ |J(\bm{\theta})| $ characterizes the sparsity of the vector $ \bm{\theta} $, given by $|J(\bm{\theta})| = \sum_{r=1}^{T} |J_r|$.
It counts the number of non-equal consecutive values of $ \bm{\theta} $. If $ \bm{\theta} $ is block-sparse, namely whenever $ | \mathcal{J}(\bm{\theta}) | \ll Td $ where $ \mathcal{J}(\bm{\theta}) = \{ r = 1, \ldots, T: \theta_{r,\bullet} \neq \mathbf{0}_{d} \} $ (meaning that few raw features are useful for prediction), then $ |J(\bm{\theta})| \leq | \mathcal{J}(\bm{\theta}) | \max_{r \in J(\bm{\theta})} |J_r| $, which means that $ |J(\bm{\theta})| $ is controlled by the block sparsity $ | \mathcal{J}(\bm{\theta}) | $.

\subsection{Assumptions} \label{sub:assumptions}

The necessary assumptions to ensure the results are listed below. To begin, we establish the essential condition for the data sequence, specifically focusing on the exponentially $\beta$-mixing condition. This condition is a key tool for describing the complex interdependence between data points.

\begin{assumption}
\label{Ass:1}
The process $\{X_{t,T}\}_{t \in z}$ is locally stationary in the sence of Definition~\ref{def:locallystseries}.
\end{assumption}
\begin{assumption}
\label{Ass:2}
The array $\{X_{t,T}, \varepsilon_{t,T}\}_{t \in z}$ is $ \beta $-mixing sequence with mixing coefficients $ \beta(k) \leq \exp (-\varphi k^{\eta_{1}}) $, for some $ \varphi>0, \eta_{1}>1$.
\end{assumption}

The exponentially $\beta$-mixing data has been employed as an underlying 
assumption in statistical learning (\cite{Xie2017}, \cite{Abhishek2019} and \cite{MR4102690}), to prove the consistency theorems for the lasso estimators of sparse linear regression models and establish risk bounds for the empirical risk minimization with both dependent and heavy-tailed data-generating processes. Additionally, $\beta$-mixing heavy-tailed time series refers to a stationary sequence of non-negative random variables with heavy tails and $\beta$-mixing dependence \citep{Miao2023}, it appears in some statistical and data analysis scenarios, especially when one is faced with the task of modeling or analyzing data sets characterized by extreme values and complex dependency patterns \citep{Wong2020}. 

\begin{assumption}
\label{Ass:KB1} \label{Ass:KB2}
The basic kernel $ K_{i},i=1,2$ is symmetric around zero, bounded by $C_{K_{i}}(i=1,2)$ for some $ C_{K_{i}}<\infty $ and has compact support, that is, $ K_{i}(v)=0$ for all $ |v|>1 $. Moreover, $ K_{i} $ is Lipschitz continuous, that is, $ |K_{i}(v)-K_{i}(v')| \leq L_{K_{i}}|v-v'| $ for some $ L_{K_{i}}<\infty,i=1,2$ and all $ v, v' \in \mathbb{R} $.
\end{assumption}

Note that throughout the paper the bandwidth $h$ of the kernel function is assumed to converge to zero at least at the polynomial rate, that is, there exists a small $0<\xi < 1$ such that $h = \bigO(T^{-\xi})$. The Assumption~\ref{Ass:KB1}, regarding kernel functions $K_{1}(\cdot)$ and $K_{2}(\cdot)$, is standard in the literature and satisfied by popular kernel functions, such as the (asymmetric) triangle and quadratic kernels \citep{Silverman1986,Vapnik2000}.

\section{Concentration inequalities for heavy-tailed LSTS} \label{sub:least_square_loss}
We propose concentration inequalities for locally stationary $\beta$-mixing sub-Weibull random variables and regularly varying random variables.
For the noise $\{\varepsilon_{t,T}\}_{t=1}^{T}$ and the kernel function $K_{h,i},i=1,2$, we define, for fixed $j \in \{1, \cdots, d\}$ and $r \in \{1, \cdots, T\}$, the sequence $W_{t,r,T}^{j}$ is
\begin{equation}\label{equation-W_t,T}
W_{t,r,T}^{j}=K_{h,1}\big(\frac{t}{T}-\frac{r}{T}\big) K_{h,2}(X_{t,T}^{j}-X_{r, T}^{j}) \varepsilon_{t,T}, \text{ for } t=1,\ldots,T.
\end{equation}
We firstly focus on the sub-Weibull distribution shown in Example~\ref{example3-(i)}.
\begin{proposition}[Locally stationary sub-Weibull distribution]
\label{proposition.Locally-stationary-sub-Weibull-distribution}
Let $\{\varepsilon_{t,T}\}_{t=1}^{T}$ follows the sub-Weibull distribution with constant $(\eta_{2},C_{\varepsilon})$, the sequence $\{W_{t,r,T}^{j}\}_{t}$ defined in (\ref{equation-W_t,T}). Assumption~\ref{Ass:1}-\ref{Ass:KB2} are satisfied. Let $h=\bigO(T^{-\xi})$ with $0<\xi<\frac{1}{2}$, for any $\gamma> 2C_{K}\sqrt{\log T / T } $ and $T>4$, we have
\begin{align*}
\mathbb  P\Big[\frac{1}{T}\big|&\sum_{t=1}^{T} W_{t,r,T}^{j}\big|\geq \gamma \Big]\\ 
&\leq \exp\Big(-(\frac{T\log T}{C_{\varepsilon}})^{\eta_{2}}\Big)\\
& \quad + T \exp \Big( -\frac{(\gamma T)^{\eta}}{(4C_{K}C_{\varepsilon})^{\eta}C_{1}} \Big) + \exp \Big( -\frac{\gamma^{2}T}{(4C_{K}C_{\varepsilon})^{2}C_{2}} \Big)\\
& \quad + T \exp \Big( -\frac{(\gamma T^{2}h)^{\eta}}{(4C_{K,L}(2T+1)C_{\varepsilon})^{\eta}C_{1}} \Big) + \exp \Big( -\frac{(\gamma h)^{2}T^{3}}{(4C_{K,L}(2T+1)C_{\varepsilon})^{2}C_{2}} \Big).
\end{align*} 
where $1/\eta = 1/\eta_{1} + 1/\eta_{2}$, $ \eta < 1$, $C_{K}=C_{K_{1}}C_{K_{2}}$, the constant $C_{K,L}$ depends on kernel bound and Lipschiz constant, the constants $C_{1}$, $C_{2}$ depend only on $\eta_{1}$, $\eta_{2}$ and $\varphi$.  
\end{proposition}
To deal with regularly varying heavy-tailed interactions, we prove the concentration inequality for the sums of locally stationary $\beta$-mixing regularly varying random variables. The following proposition is useful for the concentration inequality for the sums of locally stationary regularly varying heavy-tailed and is similar to Lemma 2.2 of \cite{Abhishek2019}.

\begin{proposition}[Stationary regularly varying heavy-tailed]\label{proposition.Regularly varying heavy-tailed}
Let $\{Z_{t, T}\}_{t=1}^{T}$ be a strictly stationary $\beta$-mixing sequence of zero mean real value random variables, follow the regularly varying heavy-tailed distributions with index $\eta_{2}>0$ and bounded slowly varying function $L(\cdot)$, see Definition~\ref{def:Regularvaryingtail}. The $\beta$-mixing coefficients satisfy $ \beta(k) \leq \exp (-\varphi k^{\eta_{1}}) $ with $ \varphi, \eta_{1}>1$. 
Let $0<\vartheta<\frac{(\eta_{1}-1)(\eta_{2}-1)}{1+(2\eta_{1}-1)\eta_{2}}$.
Then for $\varrho>1/T^{\vartheta}$, we have
\begin{align*}
\mathbb{P}\Big[\frac{1}{T}\Big|\sum_{t=1}^{T} Z_{t, T} \Big| \geq \varrho \Big] &\leq
\frac{12T^{(d_{1}-1/\eta_{1})(1-\eta_{2})} \varrho^{(d_{1}-1/\eta_{1})(1-\eta_{2})-1}}{ 2^{1-\eta_{2}}}L\Big(\frac{(\varrho T)^{d_{1}-1/\eta_{1}}}{2}\Big)\\
&\quad + \frac{6 \exp(-\varphi \varrho T)}{\varrho} +2\exp\Big({-\frac{1}{9T^{2d_{1}-1/\eta_{1}-1}\varrho^{2d_{1}-1/\eta_{1}-2}}}\Big).
\end{align*} 
where $\frac{\vartheta}{(1-\vartheta)(\eta_{2}-1)}+\frac{1}{\eta_{1}} < d_{1}< \frac{1-2\vartheta}{2(1-\vartheta)}+\frac{1}{2\eta_{1}}$. 
\end{proposition}
Next, we give the concentration inequality for the sums of locally stationary
$\beta$-mixing regularly varying random variables. 

\begin{proposition}[Locally Stationary regularly varying heavy-tailed]\label{proposition.Locally-Stationary-regularly-varying-heavy-tailed}
Let $\{\varepsilon_{t,T}\}_{t=1}^{T}$ follows regularly varying heavy-tailed with index $\eta_{2}>0$ and bounded slowly varying function $L(\cdot)$. The sequence $\{W_{t,r,T}^{j}\}_{t}$ be defined in (\ref{equation-W_t,T}) and Assumptions~\ref{Ass:1}-\ref{Ass:KB2} are satisfied. Let $h=\bigO(T^{-\xi})$ with $0<\xi<1$ and $0<\vartheta<\frac{(\eta_{1}-1)(\eta_{2}-1)}{1+(2\eta_{1}-1)\eta_{2}}$. Then for $\gamma>\frac{2C_{K,L}(2T+1)}{T^{1+\vartheta}h}$, we have
\begin{align*}
\mathbb P\Big[\frac{1}{T}\big|&\sum_{t=1}^{T} W_{t,r,T}^{j}\big|\geq \gamma \Big] \\
&\leq (T\log T)^{-\eta_{2}}L(T\log T)\\
& \quad +\frac{12T^{(d_{1}-1/\eta_{1})(1-\eta_{2})} \gamma^{(d_{1}-1/\eta_{1})(1-\eta_{2})-1}}{ 2^{1-\eta_{2}}(4C_{K})^{(d_{1}-1/\eta_{1})(1-\eta_{2})-1}}L\Big(\frac{(\gamma T/4C_{K})^{d_{1}-1/\eta_{1}}}{2}\Big)\\
&\quad + \frac{24C_{K} \exp(-\varphi \gamma T/4C_{K})}{\gamma } +2\exp\Big({-\frac{1}{9T^{2d_{1}-1/\eta_{1}-1}(\gamma/4C_{K})^{2d_{1}-1/\eta_{1}-2}}}\Big)\\
& \quad + \frac{12T^{2(d_{1}-1/\eta_{1})(1-\eta_{2})-1} (\gamma h)^{(d_{1}-1/\eta_{1})(1-\eta_{2})-1}}{ 2^{1-\eta_{2}}(4C_{K,L}(2T+1))^{(d_{1}-1/\eta_{1})(1-\eta_{2})-1}}L\Big(\frac{( \gamma T^{2}h)^{d_{1}-1/\eta_{1}}}{2(4C_{K,L}(2T+1))^{d_{1}-1/\eta_{1}}}\Big)\\
&\quad+ \frac{24C_{K,L}(2T+1)\exp(-\varphi (\frac{\gamma T^{2}h}{4 C_{K,L}(2T+1)}))}{\gamma Th}+2\exp\Big({-\frac{(4C_{K,L}(2T+1))^{2d_{1}-1/\eta_{1}-2}}{9T^{4d_{1}-2/\eta_{1}-3}(\gamma h)^{2d_{1}-1/\eta_{1}-2}}}\Big),
\end{align*}
where $C_{K}=C_{K_{1}}C_{K_{2}}$, $\varphi>0, \eta_{1}>1$, $\frac{\vartheta}{(1-\vartheta)(\eta_{2}-1)}+\frac{1}{\eta_{1}} < d_{1}< \frac{1-2\vartheta}{2(1-\vartheta)}+\frac{1}{2\eta_{1}}$ and the constant $C_{K,L}$ depend on kernel bound and Lipschiz constant. 
\end{proposition}

The following is the concentration inequality for the sums of locally stationary Pareto distributions, which is an example of a regularly varying heavy-tailed distribution.

\begin{proposition}[Locally stationary Pareto distribution]
\label{proposition.LS.pareto}
Let $\{\varepsilon_{t,T}\}_{t=1}^{T}$ follows the Pareto distribution with $\eta_{2}=4$ and $L(v)=u^{4}$, the constant $u > 0$. The sequence $\{W_{t,r,T}^{j}\}_{t}$ be defined in (\ref{equation-W_t,T}) and Assumptions~\ref{Ass:1}-\ref{Ass:KB2} are satisfied with $\eta_{1}=4$. Let $h=\bigO(T^{-\xi})$ with $0<\xi<1$ and $0<\vartheta<9/29$. Then for any $\gamma>\frac{2 C_{K,L}(2T+1)}{T^{1+\vartheta}h}$, we have
\begin{align*}
\mathbb P\Big[\frac{1}{T}\big|&\sum_{t=1}^{T} W_{t,r,T}^{j}\big|\geq \gamma \Big] \\
&\leq (T\log T)^{-4}u^{4}
+\frac{96(4C_{K})^{3d_{1}+1/4}}{ T^{3d_{1}-3/4}\gamma^{3d_{1}+1/4}}u^{4} + \frac{24C_{K} \exp(-\varphi \gamma T/4C_{K})}{\gamma}\\ 
& \quad  +2\exp\Big(-\frac{T^{5/4-2d_{1}}(\gamma/4C_{K})^{9/4-2d_{1}}}{9}\Big)
+ \frac{96 (4C_{K,L}(2T+1))^{3d_{1}+1/4}}{T^{6d_{1}-1/2}(\gamma h)^{3d_{1}+1/4}}u^{4}\\
&\quad+ \frac{24C_{K,L}(2T+1)\exp(-\varphi (\frac{\gamma T^{2}h}{4 C_{K,L}(2T+1)}))}{\gamma Th}
+2\exp\Big(-\frac{T^{7/2-4d_{1}}(\gamma h)^{9/4-2d_{1}}}{9 (4C_{K,L}(2T+1))^{9/4-2d_{1}}}\Big),
\end{align*}
where $d_{1} \in (\frac{\vartheta}{3(1-\vartheta)}+\frac{1}{4},\frac{1-2\vartheta}{2(1-\vartheta)}+\frac{1}{8})$, $C_{K}=C_{K_{1}}C_{K_{2}}$, $\varphi>0$ and the constant $C_{K,L}$ depend on kernel bound and Lipschiz constant.
\end{proposition}

The concentration inequality presented in this section provides a non-asymptotic description for the fluctuation of the weighted kernel $\frac{1}{T}\sum_{t=1}^{T} W^{j}_{t,r,T}$,
which are central to the analysis of locally stationary time series with heavy-tailed
noise. The bounds decompose into several terms, each reflecting a different aspect
of the underlying stochastic structure, the tail behavior of the innovations, and the decay of
$\beta$-mixing dependence, and the localization introduced by the time and space kernels.

\section{Oracle inequalities}
\label{sec.Oracle inequality}
We provide the non-asymptotic oracle inequalities relating $ R(\hat{m}, m^{\star}) $ and $ R(m_\theta, m^{\star}) $. Here, $ R(m_\theta, m^{\star}) $ represents the risk under the optimal parameterization, corresponding to the minimal risk in the ideal setting. The oracle inequality provides theoretical upper bounds for the estimator's performance, ensuring that the excess risk of the estimator approximates the oracle risk under ideal conditions. By introducing penalization terms via the parameter $\bm{\theta}$, the approach aims to bridge the gap between theoretical guarantees and practical estimation. 

\subsection{Slow rates}
In this subsection, we state an oracle inequality with a slow rate of convergence, which bounds the prediction error in terms of the penalty value of the regression vectors. Oracle inequalities with slow rates provide weaker bounds, which means they might be less tight but more robust, and it have been extensively studied in various contexts, demonstrating their application in machine learning and statistical estimation\citep{steinwart2006oracle,mendelson2012general}.
\subsubsection{Sub-Weibull distribution}
\begin{theorem}[Lasso penalization]
\label{theorem:least_sq_oracle_ineq_lasso_pen_subweibull}
Let Assumptions~\ref{Ass:1}-\ref{Ass:KB2} hold, $\{\varepsilon_{t,T}\}_{t=1}^{T}$ follows the sub-Weibull distribution with constant $(\eta_{2},C_{\varepsilon})$ and $\Omega(\bm{\theta})$ is the Lasso penalization. Assume the sample size satisfies $ T \geq c(\log d)^{\frac{2}{\eta}-1}$ with $1/\eta = 1/\eta_{1} + 1/\eta_{2}$, $1/2 \leq \eta<1$ and $c>1$, set $\lambda = \sqrt{\frac{c\log d +\log T}{T^{1-2\xi}}}$, and the bandwidth $h=\bigO(T^{-\xi})$ with $0<\xi<1/2$. Then the estimator $\hat m$ in problem~(\ref{estimator}) verifies
\begin{equation*}
R(\hat m, m^\star) \leq \inf_{\bm{\theta} \in\R^{T d}} \big\{R(m_{\theta}, m^\star) + 2 \lambda \norm{\bm{\theta}}_{1}\big\}.
\end{equation*}
with a probability larger than $1 - d^{1 - c}$.

\end{theorem}

\begin{theorem}[Weighted TV penalization]
\label{theorem:least_sq_oracle_ineq_TV_pen_subweibull}
Let Assumptions~\ref{Ass:1}-\ref{Ass:KB2} hold, $\{\varepsilon_{t,T}\}_{t=1}^{T}$ follows the sub-Weibull distribution with constant $(\eta_{2},C_{\varepsilon})$ and $\Omega(\bm{\theta})$ is the weighted total-variation penalization. Assume the sample size satisfies $ T \geq c(\log d)^{\frac{2}{\eta}-1}$ with $1/\eta = 1/\eta_{1} + 1/\eta_{2}$, $1/2 \leq \eta<1$ and $c>1$, set $\lambda_{j} = (d-j+1)\sqrt{\frac{c\log d +\log T}{T^{1-2\xi}}}$, and the bandwidth $h=\bigO(T^{-\xi})$ with $0<\xi<1/2$. Then the estimator $\hat m$ in problem~(\ref{estimator}) verifies
\begin{equation*}
R(\hat m, m^\star) \leq \inf_{\bm{\theta} \in\R^{Td}} \big\{R(m_{\theta}, m^\star) + 2 \norm{\bm{\theta}}_{\TV,\lambda}\big\}.
\end{equation*}
with a probability larger than $1 - d^{1 - c}$, $c>1$, the constant $C_{1}$, $C_{2}$ depend only on $c$. 
\end{theorem}

\begin{remark}
For the Sub-Weibull distribution, we provide oracle inequalities depending on the sample size $T$ at a rate of $\bigO(1/T^{\frac{1}{2}-\xi})$.
This rate is slower than the error bounds for Lasso regression with sub-Weibull random vectors, which exhibit a convergence rate of $\bigO(1/T^{1/2})$ as established in \cite{Wong2020}. This indicates that strictly stationary sequences have a faster convergence than locally stationary sequences.
\end{remark}

\begin{remark}[Block-sparsity]
We consider the vector $\bm{\theta} \in \mathbb{R}^{Td}$ to have block sparsity. Let Assumptions~\ref{Ass:1}-\ref{Ass:KB2} hold, $\{\varepsilon_{t,T}\}_{t=1}^{T}$ follows the sub-Weibull distribution with constant $(\eta_{2},C_{\varepsilon})$, assume the sample size satisfies $ T \geq c(\log d)^{\frac{2}{\eta}-1}$ with $1/\eta = 1/\eta_{1} + 1/\eta_{2}$, $1/2 \leq \eta<1$ and $c>1$, and the bandwidth $h=\bigO(T^{-\xi})$ with $0<\xi<1/2$. 
For Lasso penalization and total variation penalization,
\begin{equation*}
 \Omega_\lambda(\bm{\theta})=\bigO \big(\lambda |\mathcal{J}(\bm{\theta})|
\max_{r=1,\ldots,T}|J_{r}|\max_{r=1,\ldots,T}(\theta_{r,max},\theta_{r,max}^{2}) \big),
\end{equation*}
with $\theta_{r,max}=\max_{1\leq j \leq d} |\theta_{r,j}|$, then with the probability larger than $1 - d^{1 - c}$, we have
\begin{equation*}
R(\hat m, m^\star) \leq \inf_{\bm{\theta} \in \mathbb{R}^{Td}}\big\{R(m_{\theta}, m^\star)+\bigO\big(\lambda |\mathcal{J}(\bm{\theta})|
\max_{r=1,\ldots,T}|J_{r}|\max_{r=1,\ldots,T}(\theta_{r,max},\theta_{r,max}^{2})\big)\big\},
\end{equation*}
where $J_{r}$ defined in Equation~(\ref{equ_Lasso_ridge_spares}) for Lasso penalization and Equation~(\ref{equ_TV_spares}) for total variation penalization.

\end{remark}

\subsubsection{Regularly varying heavy-tailed }
\begin{theorem}[Lasso penalization]
\label{theorem:least_sq_oracle_ineq_lasso_pen_ragularvarying}
Let Assumptions~\ref{Ass:1}-\ref{Ass:KB2} hold, $\{\varepsilon_{t,T}\}_{t=1}^{T}$ follows the regularly varying heavy-tailed with index $\eta_{2}>\frac{3\eta_{1}-1}{\eta_{1}-1}$ and bounded slowly varying function $L(\cdot)$ and $\Omega(\bm{\theta})$ is the Lasso penalization. Let $0<\vartheta< \frac{(\eta_{1}-1)(\eta_{2}-1)-2\eta_{1}}{1+(2\eta_{1}-1)\eta_{2}}$, assume the sample size satisfies 
$$ T > \big( \frac{d^{c(2+1/\eta_{1}-2d_{1})}}{(c\log d)^{(d_{1}-1/\eta_{1})(\eta_{2}-1)+1}}\big)^{\frac{1}{(\eta_{2}+3)d_{1}-(\eta_{2}+1)/\eta_{1}-3}},$$
and 
$$\lambda = \frac{2C_{K,L}(2T+1)}{T^{1+\vartheta-\xi}}$$
with the bandwidth $h=\bigO(T^{-\xi})$, $0<\xi<\vartheta$, then the estimator $\hat m$ in problem~(\ref{estimator}) verifies
\begin{equation*}
R(\hat m, m^\star) \leq \inf_{\bm{\theta} \in\R^{T d}} \big\{R(m_{\theta}, m^\star) + 2 \lambda \norm{\bm{\theta}}_{1}\big\}.
\end{equation*}
with a probability larger than $1 - d^{1 - c}$, where $c>1$, $\varphi, \eta_{1}>1$, $\frac{1+\vartheta}{(1-\vartheta)(\eta_{2}-1)}+\frac{1}{\eta_{1}} < d_{1}< \frac{1-2\vartheta}{2(1-\vartheta)}+\frac{1}{2\eta_{1}}$, the constant $C_{K,L}$ depend on kernel bound and Lipschiz constant.
\end{theorem}

\begin{corollary}[Pareto distribution with Lasso penalization]
\label{corol:pareto}
Suppose Assumptions~\ref{Ass:1}-\ref{Ass:KB2} hold with $\eta_{1}=4$. Let $\{\varepsilon_{t,T}\}_{t=1}^{T}$ follows the Pareto distribution with $\eta_{2}=4$ and $L(v)=u^{4}$, the constant $u > 0$. Let $0<\vartheta<1/29$, assume the sample size satisfies 
$$ T > \bigg( \frac{d^{c(9/4-2d_{1})}}{(c\log d)^{3d_{1}+1/4}}\bigg)^{1/(7d_{1}-17/4)},$$
and 
$$\lambda = \frac{2C_{K,L}(2T+1)}{T^{1+\vartheta-\xi}}$$
with $d_{1} \in (\frac{1+\vartheta}{3(1-\vartheta)}+\frac{1}{4},\frac{1-2\vartheta}{2(1-\vartheta)}+\frac{1}{8})$ and $c>1$, the bandwidth $h=\bigO(T^{-\xi})$ with  $0<\xi<\vartheta$, then the estimator $\hat m$ in problem~(\ref{estimator}) verifies
\begin{equation*}
R(\hat m, m^\star) \leq \inf_{\bm{\theta} \in\R^{T d}} \big\{R(m_{\theta}, m^\star) + 2 \lambda \norm{\bm{\theta}}_{1}\big\}.
\end{equation*}
with a probability larger than $1 - d^{1 - c}$, where $c>1$, $\varphi >1$, the constant $C_{K,L}$ depend on kernel bound and Lipschiz constant.
\end{corollary}

\begin{theorem}[Weighted TV penalization]\label{theorem:TV_Regularly_V_H}
Let Assumptions~\ref{Ass:1}-\ref{Ass:KB2} hold, $\{\varepsilon_{t,T}\}_{t=1}^{T}$ follows the regularly varying heavy-tailed with index $\eta_{2}>\frac{3\eta_{1}-1}{\eta_{1}-1}$ and bounded slowly varying function $L(\cdot)$ and $\Omega(\bm{\theta})$ is the weighted total-variation penalization. Let $0<\vartheta< \frac{(\eta_{1}-1)(\eta_{2}-1)-2\eta_{1}}{1+(2\eta_{1}-1)\eta_{2}}$, assume the sample size satisfies 
$$ T > \bigg( \frac{d^{c(2+1/\eta_{1}-2d_{1})}}{(c\log d)^{(d_{1}-1/\eta_{1})(\eta_{2}-1)+1}}\bigg)^{\frac{1}{(\eta_{2}+3)d_{1}-(\eta_{2}+1)/\eta_{1}-3}},$$
and 
$$\lambda_{j} \geq (d-j+1) \frac{2C_{K,L}(2T+1)}{T^{1+\vartheta-\xi}}$$
with the bandwidth $h=\bigO(T^{-\xi})$ with $0<\xi<\vartheta$, then the estimator $\hat m$ in problem~(\ref{estimator}) verifies
\begin{equation*}
R(\hat m, m^\star) \leq \inf_{\bm{\theta} \in\R^{T d}} \big\{R(m_{\theta}, m^\star) + 2 \norm{\bm{\theta}}_{\TV,\lambda}\big\}.
\end{equation*}
with a probability larger than $1 - d^{1 - c}$, where $c>1$, $\varphi, \eta_{1}>1$, $\frac{1+\vartheta}{(1-\vartheta)(\eta_{2}-1)}+\frac{1}{\eta_{1}} < d_{1}< \frac{1-2\vartheta}{2(1-\vartheta)}+\frac{1}{2\eta_{1}}$, the constant $C_{K,L}$ depend on kernel bound and Lipschiz constant.
\end{theorem}

\begin{remark}[Block-sparsity]
We consider the vector $\bm{\theta} \in \mathbb{R}^{Td}$ to have block sparsity. Let Assumptions~\ref{Ass:1}-\ref{Ass:KB2} hold, $\{\varepsilon_{t,T}\}_{t=1}^{T}$ follows the regularly varying heavy-tailed with index $\eta_{2}>0$ and bounded slowly varying function $L(\cdot)$, assume the sample size satisfies 
$$ T > \big( \frac{d^{c(2+1/\eta_{1}-2d_{1})}}{(c\log d)^{(d_{1}-1/\eta_{1})(\eta_{2}-1)+1}}\big)^{\frac{1}{(\eta_{2}+3)d_{1}-(\eta_{2}+1)/\eta_{1}-3}},$$ 
and the bandwidth $h=\bigO(T^{-\xi})$ with  $0<\xi<\vartheta$.
For Lasso penalization and total variation penalization,
\begin{equation*}
 \Omega_\lambda(\bm{\theta})=\bigO(\lambda |\mathcal{J}(\bm{\theta})|
\max_{r=1,\ldots,T}|J_{r}|\max_{r=1,\ldots,T}|\theta_{r,max},\theta_{r,max}^{2}|),
\end{equation*}
with $\theta_{r,max}=\max_{1\leq j \leq d} |\theta_{r,j}|$, then with the probability larger than $1 - d^{1 - c}$, we have
\begin{equation*}
R(\hat m, m^\star) \leq \inf_{\bm{\theta} \in \mathbb{R}^{Td}} \big\{R(m_{\theta}, m^\star)+ \bigO\big(\lambda |\mathcal{J}(\bm{\theta})|
\max_{r=1,\ldots,T}|J_{r}|\max_{r=1,\ldots,T}|\theta_{r,max},\theta_{r,max}^{2}|\big) \big\},
\end{equation*}
where $J_{r}$ defined in Equation~(\ref{equ_Lasso_ridge_spares}) for Lasso penalization and Equation~(\ref{equ_TV_spares}) for total variation penalization.

\end{remark}

We present a table summarizing the key aspects of high-dimensional estimation of the four theorems under the sub-Weibull or regularly varying distributions. The results show two types of penalties: Lasso and weighted total variation (TV), each addressing heavy-tailed data with an appropriate sample size and conditions for the penalty parameters. 
Each theorem involves specific constraints on parameters $\eta$, $\eta_1$, $\eta_2$, $d_1$, and $\vartheta$, which are necessary for the oracle inequality to hold.

\begin{table}[h!]
\renewcommand\arraystretch{2.8}
\centering
\resizebox{1.0\linewidth}{!}{
\begin{tabular}{|l|c|c|c|c|}
\hline
\multirow{2}{*}{\textbf{Properties}} & \multicolumn{2}{c|}{\textbf{Sub-Weibull Noise}} & \multicolumn{2}{c|}{\textbf{Regularly Varying Noise}} \\
\cline{2-5}
& \textbf{Lasso} & \textbf{Weighted TV} & \textbf{Lasso} & \textbf{Weighted TV} \\
\hline
Bandwidth $h=\bigO(T^{-\xi})$ & $0<\xi<1/2$ &  $0<\xi<1/2$ &  $0<\xi<\vartheta< \frac{(\eta_{1}-1)(\eta_{2}-1)-2\eta_{1}}{1+(2\eta_{1}-1)\eta_{2}}$  &  $0<\xi<\vartheta< \frac{(\eta_{1}-1)(\eta_{2}-1)-2\eta_{1}}{1+(2\eta_{1}-1)\eta_{2}}$ \\
\hline
Sample Size& 
$T \geq c(\log d)^{\frac{2}{\eta} - 1}$ & 
$T \geq c(\log d)^{\frac{2}{\eta} - 1}$ & 
$T > \bigg( \frac{d^{c(2 + 1/\eta_1 - 2d_1)}}{(c \log d)^{(d_1 - 1/\eta_1)(\eta_2 - 1) + 1}} \bigg)^{\frac{1}{(\eta_2 + 3) d_1 - (\eta_2 + 1)/\eta_1 - 3}}$ & 
$T > \bigg( \frac{d^{c(2 + 1/\eta_1 - 2d_1)}}{(c \log d)^{(d_1 - 1/\eta_1)(\eta_2 - 1) + 1}} \bigg)^{\frac{1}{(\eta_2 + 3) d_1 - (\eta_2 + 1)/\eta_1 - 3}}$ \\
\hline
Penalty Parameter & 
$\lambda = \sqrt{\frac{c \log d + \log T}{T^{1 - 2\xi}}}$ & 
$\lambda_j = (d - j + 1) \sqrt{\frac{c \log d + \log T}{T^{1 - 2\xi}}}$ & 
$\lambda = \frac{2C_{K,L}(2T + 1)}{T^{1 + \vartheta - \xi}}$ & 
$\lambda_j = (d - j + 1) \frac{2C_{K,L}(2T + 1)}{T^{1 + \vartheta - \xi}}$ \\
\hline
Probability Bound & 
$1 - d^{1 - c}$ & 
$1 - d^{1 - c}$ & 
$1 - d^{1 - c}$ & 
$1 - d^{1 - c}$ \\
\hline
\end{tabular}
}
\caption{Summary of theorems on Lasso and weighted TV Penalization under heavy-tailed noise}
\label{tab:theorem_summary}
\end{table}

\subsection{Fast rates}

The oracle inequality with a fast rate provides tighter bounds, leading to more precise performance measures, but these are only valid when strong assumptions are met. We impose the restricted eigenvalue condition on the matrix $\bK$ to establish fast oracle inequalities. In the case of high-dimensional estimators, restricted eigenvalue conditions characterize the sample complexity of precise recovery \citep{Bickel2009}. Restricted eigenvalue conditions are needed to guarantee nice statistical properties. Here, we present a condition equivalent to the restricted eigenvalue condition proposed by \cite{Bickel2009,hsu2016loss} and \cite{alaya2019binarsity}.

\begin{assumption}[Restricted eigenvalue condition]
\label{Assumption_RE_condition}
Let $J(\bm{\theta})$ is the sparsity of a vector of coefficients $\bm{\theta}$ with $0 \leq |J(\bm{\theta})| \leq J^{\star}$, the following condition holds:
\begin{equation*}
\kappa(\bK, J(\bm{\theta})) \triangleq \min_{J_0 \subseteq \{1, \ldots, Td\},\atop |J_0| \leq |J(\bm{\theta})|} \, \min_{\Delta \in S_{J_{0}}} \frac{\|\bK \Delta\|_2}{\sqrt{T}\|\Delta_{J_0}\|_2} > 0,
\end{equation*}
where $J(\bm{\theta})$ is the sparsity of a vector of coefficients $\bm{\theta}$.
\begin{itemize}
\item[(i)] For Lasso penalization, $S_{J_{0}} = \Big\{ \Delta \in \mathbb{R}^{Td} \setminus \{ \mathbf{0}\} \mid \|\Delta_{J_{0}^{\complement}}\|_1 \leq 3 \|\Delta_{J_{0}}\|_1 \Big\}$.

\item[(ii)] For weighted total variation penalization, 
\begin{equation*}
S_{J_{0}} = \Big\{ \Delta \in \mathbb{R}^{Td} \setminus \{ \mathbf{0}\} \mid 
\sum_{r=1}^{T} \|( \Delta_{r\bullet} )_{J_{0}^{\complement}}\|_{\TV,\lambda} \leq 3 \sum_{r=1}^{T} \|( \Delta_{r\bullet})_{J_{0}}\|_{\TV,\lambda} \Big\}.
\end{equation*}
\end{itemize}
\end{assumption}

\subsubsection{Sub-Weibull distribution}

\begin{theorem}[Lasso penalization]
\label{theorem:fast_lasso_pen_SubWeibull} 
Let Assumptions~\ref{Ass:1}-\ref{Ass:KB2} and Assumption~\ref{Assumption_RE_condition}-(i) hold, $\{\varepsilon_{t,T}\}_{t=1}^{T}$ follows the sub-Weibull distribution with constant $(\eta_{2},C_{\varepsilon})$ and $\Omega(\bm{\theta})$ is the Lasso penalization. Assume the sample size satisfies $ T \geq c(\log d)^{\frac{2}{\eta}-1}$ with $1/\eta = 1/\eta_{1} + 1/\eta_{2}$, $1/2 \leq \eta<1$ and $c>1$, set $\lambda = \sqrt{\frac{c\log d +\log T}{T^{1-2\xi}}}$, and the bandwidth $h=\bigO(T^{-\xi})$ with $0<\xi<1/2$. Assume $\kappa(\bm{K},J(\bm{\theta}))>0$, the vector $\hat{\bm{\theta}}$ satisfies
\begin{equation*}
\inf_{\bm{\theta} \in \mathbb{R}^{Td}} \norm{\hat{\bm{\theta}}-\bm{\theta}}_{2} \leq \inf_{\bm{\theta} \in \mathbb{R}^{Td}} \frac{3 (\sqrt{3}+1) \lambda \sqrt{J^{\star}} }{2 \kappa^{2}(\bm{K},J(\bm{\theta}))},
\end{equation*}
\begin{equation*}
\inf_{\bm{\theta} \in \mathbb{R}^{Td}} \frac{1}{T} \| K(\hat{\bm{\theta}}-\bm{\theta}) \|_{2}^{2} \leq \inf_{\bm{\theta} \in \mathbb{R}^{Td}} \frac{9 \lambda^{2} J^{\star}} {4 \kappa^{2}(\bm{K}, J(\bm{\theta}))}
\end{equation*}
and
\begin{equation*}
R(\hat m, m^\star) \leq \inf_{\bm{\theta} \in\R^{T d}} \Big\{R(m_{\theta}, m^\star) + \frac{9 \lambda^{2} J^{\star}}{16 \kappa^{2}(\bm{K}, J(\bm{\theta}))} \Big\}. 
\end{equation*}
with a probability larger than $1 - d^{1 - c}$.
\end{theorem}

\begin{theorem}[Total variation penalization]
\label{theorem:fast_TV_pen_SubWeibull}
Let Assumptions~\ref{Ass:1}-\ref{Ass:KB2} and Assumption~\ref{Assumption_RE_condition}-(ii) hold, $\kappa(\bm{K},J(\bm{\theta}))>0$, $\{\varepsilon_{t,T}\}_{t=1}^{T}$ follows the sub-Weibull distribution with constant $(\eta_{2},C_{\varepsilon})$ and $\Omega(\bm{\theta})$ is the weighted total-variation penalization. Assume the sample size satisfies $ T \geq c(\log d)^{\frac{2}{\eta}-1}$ with $1/\eta = 1/\eta_{1} + 1/\eta_{2}$, $1/2 \leq \eta<1$ and $c>1$, set $\lambda_{j} = (d-j+1)\sqrt{\frac{c\log d +\log T}{T^{1-2\xi}}}$, and the bandwidth $h=\bigO(T^{-\xi})$ with $0<\xi<1/2$,
\begin{equation*}
\inf_{\bm{\theta} \in \mathbb{R}^{Td}} \norm{\hat{\bm{\theta}}-\bm{\theta}}_{2} \leq \inf_{\bm{\theta} \in \mathbb{R}^{Td}} \frac{(\sqrt{3}+1)  \sqrt{288 J^{\star} }\max\limits_{r=1,\dots,T}\|(\lambda_{j})_{J_{r}(\bm{\theta})}\|_{\infty}}{ \kappa^{2}(\bm{K}, J(\bm{\theta})) },
\end{equation*}
\begin{equation*}
\inf_{\bm{\theta} \in \mathbb{R}^{Td}} \frac{1}{T} \| K(\hat{\bm{\theta}}-\bm{\theta}) \|_{2}^{2}
\leq \inf_{\bm{\theta} \in \mathbb{R}^{Td}} \frac{ \sqrt{288 J^{\star} }\max\limits_{r=1,\dots,T}\|(\lambda_{j})_{J_{r}(\bm{\theta})}\|_{\infty}}{ \sqrt{T} \kappa(\bm{K}, J(\bm{\theta})) }  
\end{equation*}
and
\begin{equation*}
R(\hat m, m^\star) \leq \inf_{\bm{\theta} \in\R^{T d}} \big\{R(m_{\theta}, m^\star)  + \frac{288 J^{\star}}{ \kappa^{2}(\bm{K}, J(\bm{\theta}))}\max_{r=1,\dots,T}\|(\lambda_{j})_{J_{r}(\bm{\theta})}\|_{\infty}^{2} \big\} 
\end{equation*}
with a probability larger than $1 - d^{1 - c}$, $c>1$, the constant $C_{1}$, $C_{2}$ depend only on $c$. 
\end{theorem}

\subsubsection{Regularly varying heavy-tailed}
\begin{theorem}[Lasso penalization]
\label{theorem:fast_lasso_pen_ragularvarying}
Let Assumptions~\ref{Ass:1}-\ref{Ass:KB2} and Assumption~\ref{Assumption_RE_condition}-(i) hold, $\kappa(\bm{K},J(\bm{\theta}))>0$, $\{\varepsilon_{t,T}\}_{t=1}^{T}$ follows the regularly varying heavy-tailed with index $\eta_{2}>\frac{3\eta_{1}-1}{\eta_{1}-1}$ and bounded slowly varying function $L(\cdot)$ and $\Omega(\bm{\theta})$ is the Lasso penalization. Let $0<\vartheta< \frac{(\eta_{1}-1)(\eta_{2}-1)-2\eta_{1}}{1+(2\eta_{1}-1)\eta_{2}}$, assume the sample size satisfies 
$$ T > \big( \frac{d^{c(2+1/\eta_{1}-2d_{1})}}{(c\log d)^{(d_{1}-1/\eta_{1})(\eta_{2}-1)+1}}\big)^{\frac{1}{(\eta_{2}+3)d_{1}-(\eta_{2}+1)/\eta_{1}-3}},$$
set $\lambda = \frac{2C_{K,L}(2T+1)}{T^{1+\vartheta-\xi}}$
with the bandwidth $h=\bigO(T^{-\xi})$, $0<\xi<\vartheta$,
\begin{equation*}
\inf_{\bm{\theta} \in \mathbb{R}^{Td}} \norm{\hat{\bm{\theta}}-\bm{\theta}}_{2} \leq \inf_{\bm{\theta} \in \mathbb{R}^{Td}} \frac{3 (\sqrt{3}+1) \lambda \sqrt{J^{\star}} }{2 \kappa^{2}(K,J(\bm{\theta}))},
\end{equation*}
\begin{equation*}
\inf_{\bm{\theta} \in \mathbb{R}^{Td}} \frac{1}{T} \| K(\hat{\bm{\theta}}-\bm{\theta}) \|_{2}^{2} \leq \inf_{\bm{\theta} \in \mathbb{R}^{Td}} \frac{9 \lambda^{2} J^{\star}} {4 \kappa^{2}(K, J(\bm{\theta}))}
\end{equation*}
and
\begin{equation*}
R(\hat m, m^\star) \leq \inf_{\bm{\theta} \in\R^{T d}} \big\{R(m_{\theta}, m^\star) + \frac{9 \lambda^{2} J^{\star}}{16 \kappa^{2}(K, J(\bm{\theta}))} \big\}. 
\end{equation*}
with a probability larger than $1 - d^{1 - c}$, where $c>1$, $\varphi, \eta_{1}>1$, $\frac{1+\vartheta}{(1-\vartheta)(\eta_{2}-1)}+\frac{1}{\eta_{1}} < d_{1}< \frac{1-2\vartheta}{2(1-\vartheta)}+\frac{1}{2\eta_{1}}$, the constant $C_{K,L}$ depend on kernel bound and Lipschiz constant.
\end{theorem}

\begin{theorem}[Total variation penalization]
\label{theorem:fast_TV_pen_ragularvarying}
Let Assumptions~\ref{Ass:1}-\ref{Ass:KB2} and Assumptions~\ref{Assumption_RE_condition}-(ii) hold, $\kappa(\bm{K},J(\bm{\theta}))>0$, $\{\varepsilon_{t,T}\}_{t=1}^{T}$ follows the regularly varying heavy-tailed with index $\eta_{2}>\frac{3\eta_{1}-1}{\eta_{1}-1}$ and bounded slowly varying function $L(\cdot)$ and $\Omega(\bm{\theta})$ is the weighted total-variation penalization. Let $0<\vartheta< \frac{(\eta_{1}-1)(\eta_{2}-1)-2\eta_{1}}{1+(2\eta_{1}-1)\eta_{2}}$, assume the sample size satisfies 
$$ T > \big( \frac{d^{c(2+1/\eta_{1}-2d_{1})}}{(c\log d)^{(d_{1}-1/\eta_{1})(\eta_{2}-1)+1}}\big)^{\frac{1}{(\eta_{2}+3)d_{1}-(\eta_{2}+1)/\eta_{1}-3}},$$
and 
$$\lambda_{j} = (d-j+1) \frac{2C_{K,L}(2T+1)}{T^{1+\vartheta-\xi}}$$
with the bandwidth $h=\bigO(T^{-\xi})$ with $0<\xi<\vartheta$,
\begin{equation*}
\inf_{\bm{\theta} \in \mathbb{R}^{Td}} \norm{\hat{\bm{\theta}}-\bm{\theta}}_{2} \leq \inf_{\bm{\theta} \in \mathbb{R}^{Td}} \frac{(\sqrt{3}+1)  \sqrt{288 J^{\star} }\max\limits_{r=1,\dots,T}\|(\lambda_{j})_{J_{r}(\bm{\theta})}\|_{\infty}}{ \kappa^{2}(K, J(\bm{\theta})) },
\end{equation*}
\begin{equation*}
\inf_{\bm{\theta} \in \mathbb{R}^{Td}} \frac{1}{T} \| K(\hat{\bm{\theta}}-\bm{\theta}) \|_{2}^{2}
\leq \inf_{\bm{\theta} \in \mathbb{R}^{Td}} \frac{ \sqrt{288 J^{\star} }\max\limits_{r=1,\dots,T}\|(\lambda_{j})_{J_{r}(\bm{\theta})}\|_{\infty}}{ \sqrt{T} \kappa(\bm{K}, J(\bm{\theta})) }  
\end{equation*}
and
\begin{equation*}
R(\hat m, m^\star) \leq \inf_{\bm{\theta} \in\R^{T d}} \big\{R(m_{\theta}, m^\star)  + \frac{288 J^{\star}}{ \kappa^{2}(K, J(\bm{\theta}))}\max_{r=1,\dots,T}\|(\lambda_{j})_{J_{r}(\bm{\theta})}\|_{\infty}^{2} \big\} 
\end{equation*}
with a probability larger than $1 - d^{1 - c}$, where $c>1$, $\varphi, \eta_{1}>1$, $\frac{1+\vartheta}{(1-\vartheta)(\eta_{2}-1)}+\frac{1}{\eta_{1}} < d_{1}< \frac{1-2\vartheta}{2(1-\vartheta)}+\frac{1}{2\eta_{1}}$, the constant $C_{K,L}$ depend on kernel bound and Lipschiz constant.
\end{theorem}

Theorems~\ref{theorem:fast_lasso_pen_SubWeibull}, ~\ref{theorem:fast_TV_pen_SubWeibull}, ~\ref{theorem:fast_lasso_pen_ragularvarying} and ~\ref{theorem:fast_TV_pen_ragularvarying} 
provide error bounds and oracle inequalities for Lasso and weighted total variation penalized estimators under regularly varying heavy-tailed noise. Both results require a sufficiently large sample size, specific penalty parameters, and rely on the restricted eigenvalue condition $\kappa(\bm{K}, J(\bm{\theta}))$ for effective estimation. 
The results demonstrate that, with high probability, each estimator achieves $\ell_2$ norm and prediction error bounds that are near-optimal given the sparsity level $J^{\star}$ and regularization constants. These findings are robust in high-dimensional, heavy-tailed settings, with reliability guaranteed by the probability bound $1 - d^{1 - c}$ as $d$ grows.

\section{Numerical experiments} \label{sec:numerical_experiments}
In this section, we report numerical experiments to study the finite-sample performance of the proposed estimators. We first outline the proximal algorithms used for optimization, followed by simulation results on synthetic datasets with heavy-tailed innovations.

\subsection{Optimization scheme: proximal algorithms} Proximal algorithms~\citep{10.1561/2400000003} for solving optimization problems with composite objective functions leverage a decomposition strategy. The objective function is typically expressed as the sum of a smooth differentiable term $ R_T(\bm{\theta}) $, which measures the empirical risk, and a non-smooth regularization term $ \Omega_\lambda(\bm{\theta}) $. 
The decomposition involves the following two primary steps in each iteration. $(i) $ Gradient Descent Update: update the parameter vector by minimizing the smooth part $ R_T(\bm{\theta}) $ using a gradient step $\bm{\theta}^{(k+1)} \rightarrow \bm{\theta}^{(k)} - \eta \nabla R_T(\bm{\theta}^{(k)})$
where $ \eta > 0 $ is the step size and $ \nabla R_T(\bm{\theta}^{(k)}) $ denotes the gradient. $(ii)$ Proximal Mapping: address the non-smooth part $ \Omega(\bm{\theta}) $ through a proximal operator $\bm{\theta}^{(k+1)} \rightarrow \text{prox}_{\eta  \Omega_\lambda}(\bm{\theta}^{(k+1)}),$
where the proximal operator is defined as
\begin{equation*}
\text{prox}_{\eta  \Omega_\lambda}(v) = \arg\min_x \Big\{ \frac{1}{2} \|x - v\|_2^2 + \eta  \Omega_\lambda(x) \Big\}.
\end{equation*}

This iterative scheme ensures that the optimization process alternates between the gradient descent for the smooth component and the proximal adjustment for the non-smooth regularization term. For Lasso regularization, the proximal operator corresponds to a soft-thresholding function applied element-wise:
\begin{equation*}
[\text{prox}_{\eta \lambda \|\cdot\|_1}(v)]_i =
\begin{cases}
v_i - \eta \lambda, & \text{if } v_i > \eta \lambda, \\
v_i + \eta \lambda, & \text{if } v_i < -\eta \lambda, \\
0, & \text{if } |v_i| \leq \eta \lambda.
\end{cases}
\end{equation*}

For Weighted Total Variation (WTV) regularization, 
the proximal operator requires for each $ r $, solve the following subproblem:

\begin{equation*}
\bm{\theta}_r^{(k+1)} = \arg\min_x \Big\{ \frac{1}{2} \|x - v_r\|_2^2 + \eta \sum_{j=2}^d \lambda_j |x_j - x_{j-1}| \Big\},
\end{equation*}
where $ v_r = \bm{\theta}^{(k)}_r $. This can be performed iteratively using: compute differences $ D_j = v_{r,j} - v_{r,j-1} $. Apply a weighted shrinkage operator $x_j = \text{shrinkage}(D_j, \eta \lambda_j),$
where
\begin{equation*}
\text{shrinkage}(z, \tau) =
\begin{cases}
z - \tau, & \text{if } z > \tau, \\
z + \tau, & \text{if } z < -\tau, \\
 0, & \text{if } |z| \leq \tau.
\end{cases}
\end{equation*}

\subsection{Simulation design}
\label{subsec:sim-design}
In this section, we illustrate the finite sample performance of the proposed
kernel additive estimators with sparse penalties in a controlled heavy-tailed, locally stationary setting that satisfies Assumptions~\ref{Ass:1}–\ref{Ass:KB2}. We consider the nonparametric regression model~(\ref{reg_model}) where for a given sample size $T$, dimension $d$, and tail parameter $\eta$, we minimize the penalized empirical risk
\begin{equation*}
\hat{\bm{\theta}} = \argmin_{\bm{\theta} \in \mathbb{R}^{Td}} 
\Big\{ \tfrac{1}{T}\|\bY - \bK \bm{\theta}\|_2^2 
+  \Omega_\lambda(\bm{\theta}) \Big\},
\end{equation*}
where $\Omega(\bm{\theta})$ is either the $\ell_1$-Lasso or the weighted total variation penalty. Although the theoretical analysis is developed for a general kernel additive framework, all simulation experiments are conducted within a time varying linear regression model. This choice can be viewed as a specialization of the general setting, which retains the locally stationary dependence structure and the sparsity promoting penalties considered in the theory, while remaining computationally tractable in high–dimensional regimes.

\paragraph{Locally stationary covariates.}
The covariates $\{X_{t,T}\}_{t=1}^T$ take values in $\mathbb{R}^d$ and are generated from an
$m$-dependent locally stationary process of the form
\begin{equation*}
X_{t,T}=m\Big(\frac{t}{T}\Big)+\sum_{r=0}^{m} a_r\!\Big(\frac{t}{T}\Big) Z_{t-r},\qquad
t=1,\ldots,T,
\end{equation*}
where $m(\cdot)\in\mathbb{R}^d$ and $a_r(\cdot)\in\mathbb{R}^d$ are smooth functions defined on
$[0,1]$, and $\{Z_t\}$ is an i.i.d.\ innovation sequence.
In the simulations, the functions $m(\cdot)$ and $a_r(\cdot)$ are specified in a
componentwise manner. Specifically, for each $j=1,\ldots,d$,
\begin{equation*}
m_j(u) = a_j \sin(2\pi u + \psi_j), \qquad
a_{r,j}(u) = b_{r,j} \sin(2\pi u + \varphi_{r,j}),
\qquad u\in[0,1],
\end{equation*}
where the amplitudes $a_j\in(0.3,0.8)$ and $b_{r,j}\in(0,0.9)$, as well as the phases
$\psi_j,\varphi_{r,j}\in[0,2\pi)$, are drawn independently at the beginning of each replication
and then kept fixed over time.
To guarantee stability, the coefficients are rescaled such that
\begin{equation*}
\sum_{r=0}^{m} |a_{r,j}(u)| \le \rho < 1
\quad
\text{uniformly for } u\in[0,1] \text{ and } j=1,\ldots,d,
\end{equation*}
where $\rho\in(0,1)$ is a fixed constant, taken as $\rho=0.95$ throughout the simulations.
Under this construction, the process $\{X_{t,T}\}$ is strictly $m$-dependent and therefore
$\beta$-mixing with $\beta(k)=0$ for all $k>m$.

\paragraph{True regression function.}
We consider
\begin{equation*}
  m^{\star}\!\Big(\frac{t}{T}, X_{t,T}\Big)
  \;=\;
  \sum_{j=1}^d m_t^{\star}(j)\, X_{t,T}^{(j)},
  \qquad t=1,\dots,T,
\end{equation*}
where the time–varying coefficients $m_t^{\star}(j)$ exhibit structure in the time and feature dimensions $(t,j)$. Specifically, let $m^{\star}(\cdot,\cdot)$ be encoded by an array $\{m_t^\star(j)\}_{t=1,\dots,T}^{j=1,\dots,d}$ defined as follows.
First, we construct a time-varying temporal pattern
\begin{equation*}
  p_t=
  \begin{cases}
    0.6 + \dfrac{-0.8}{T/2}\, (t-1), & t = 1,\dots,\lfloor T/2 \rfloor, \\[1.2ex]
    -0.2 + \dfrac{0.8}{T - \lfloor T/2 \rfloor}\, \big( t - \lfloor T/2 \rfloor - 1 \big),
    & t = \lfloor T/2 \rfloor + 1,\dots,T,
  \end{cases}
\end{equation*}
that is, a piecewise–linear cycle from $0.6$ down to $-0.2$ and back to
$0.6$. Then we define two contiguous active blocks in the feature index $j$:
\begin{align*}
  m_t^{\star}(j) &=
  \begin{cases}
    p_t, & j = 1,\dots,s_1,\\
    0.7\,p_t, & j = s_1+1,\dots,s_2,\\
    0, & j > s_2,
  \end{cases}
\end{align*}
where $s_1 = \max\{3,\lfloor 0.2 d \rfloor\}$ and
$s_2 = \max\{s_1+2, \lfloor 0.4 d \rfloor\}$.  

\paragraph{Noise families.}
The noise process $\{\varepsilon_{t,T}\}$ follows one of two heavy-tailed families: $(i)$ The sub-Weibull distribution defined in Example~\ref{example3-(i)}, where we examine $\eta \in \{0.8, 1.0, 1.5\}$ with $C=1$. $(ii)$ The Pareto distribution defined in Example~\ref{example4-(ii)}, where we  examine with tail index $\eta \in \{3.0, 3.7, 4.0\}$ and threshold $u=1$.

\paragraph{Kernel Functions.}
We adopt the compactly supported trigonometric kernels $K_1$ and $K_2$,
\begin{equation}  \label{eq:sim-kernels}
  K_1(v) = \frac{\pi}{4}\cos\!\Big(\frac{\pi v}{2}\Big)\mathbf{1}\{|v|\le1\},
  \qquad
  K_2(v) = \frac{1}{2}\bigl(1+\cos(\pi v)\bigr)\mathbf{1}\{|v|\le1\},
\end{equation}
which are symmetric, bounded and Lipschitz on $\mathbb{R}$ with support
contained in $[-1,1]$, so that Assumption~\ref{Ass:KB2} hold.

\paragraph{Tuning parameters and evaluation.}

We consider high–dimensional regimes with $d \in \{50, 100, 150, 200\}$.
For each $d$ and each noise family, we use a sequence of sample sizes of the
form $T \in \{T_0(d), T_0(d)+100,\dots,T_0(d)+900\}$, where
\begin{equation*}
  T_0(d) =
  \begin{cases}
    100 \Big\lceil 2 \sqrt{d} \log d /100 \Big\rceil, & \text{Sub--Weibull errors},\\[0.6ex]
    100 \Big\lceil 2\, d \log d /100 \Big\rceil, & \text{Pareto errors}.
  \end{cases}
\end{equation*}
In particular, for each $T$, we set $ h = C T^{-1/3}$, where $ C$ is a fixed constant greater than 0 (in the code, we work with moderate values, such as $ C = 0.6$). For Lasso, the penalty level $\lambda$ is chosen based on the theoretical rates:
\begin{equation*}
  \lambda =
  \begin{cases}
    \big( c \log d + \log T \big)^{1/2} T^{-(1-2\xi)/2},
    & \text{Sub-Weibull errors},\\[0.8ex]
    C_{\mathrm{KL}} T^{-(1+\vartheta-\xi)},
    & \text{Pareto errors},
  \end{cases}
\end{equation*}
and for weighted total variation, the penalty level $\lambda_{j}$ is chosen based on the theoretical rates:
\begin{equation*}
  \lambda_{j} =
  \begin{cases}
    (d-j+1)\big( c \log d + \log T \big)^{1/2} T^{-(1-2\xi)/2},
    & \text{Sub-Weibull errors},\\[0.8ex]
    (d-j+1)C_{\mathrm{KL}} T^{-(1+\vartheta-\xi)},
    & \text{Pareto errors},
  \end{cases}
\end{equation*}
where fixed constants $c, C_{\mathrm{KL}}>0$, small $\xi=1/3$ and $\vartheta>0$.

To evaluate the out-of-sample predictive performance, we independently
generate a fresh test sample of size $T_{\mathrm{test}}$ from the same
data-generating process and compute the generalization error
\begin{equation*}
  R(\hat m, m^\star) = \frac{1}{T_{\mathrm{test}}}\sum_{t=1}^{T_{\mathrm{test}}}
  \Bigl(\hat m\bigl(t/T_{\mathrm{test}}, X_{t,T}^{\mathrm{(test)}}\bigr)- m^\star\bigl(t/T_{\mathrm{test}}, X_{t,T}^{\mathrm{(test)}}\bigr)
  \Bigr)^2,
\end{equation*}
where $\hat m$ is the estimator corresponding to
$\hat{\bm{\theta}}^{\,\mathrm{Lasso}}$ or $\hat{\bm{\theta}}^{\,\mathrm{WTV}}$.

\begin{figure}[htbp]
\centering
\includegraphics[width=0.8\textwidth]{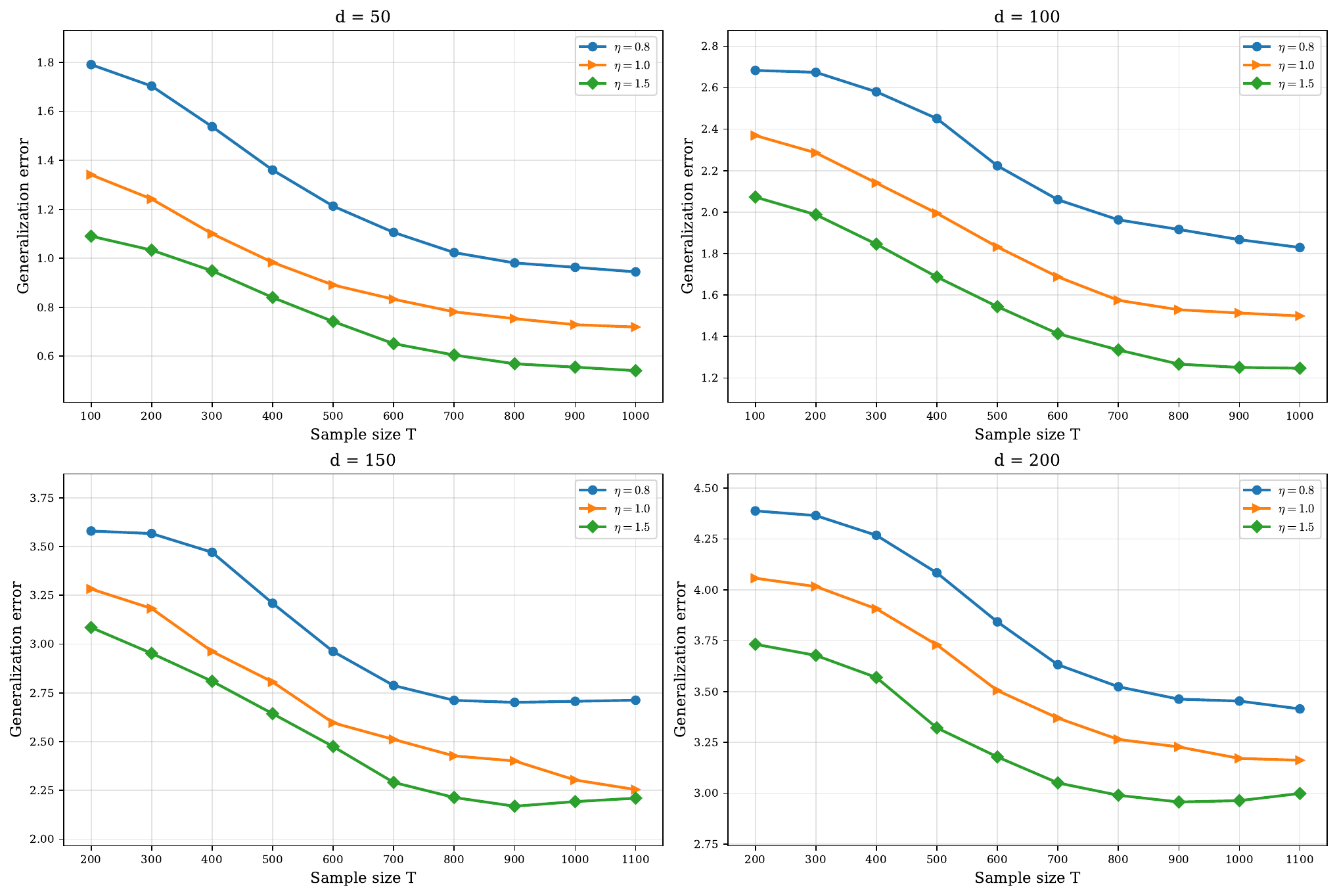}
\captionsetup{font=scriptsize}
\caption{Generalization error of the Lasso estimator under sub-Weibull noise.}
\label{fig:subweibull_lasso}
\end{figure}

\begin{figure}[htbp]
\centering
\includegraphics[width=0.8\textwidth]{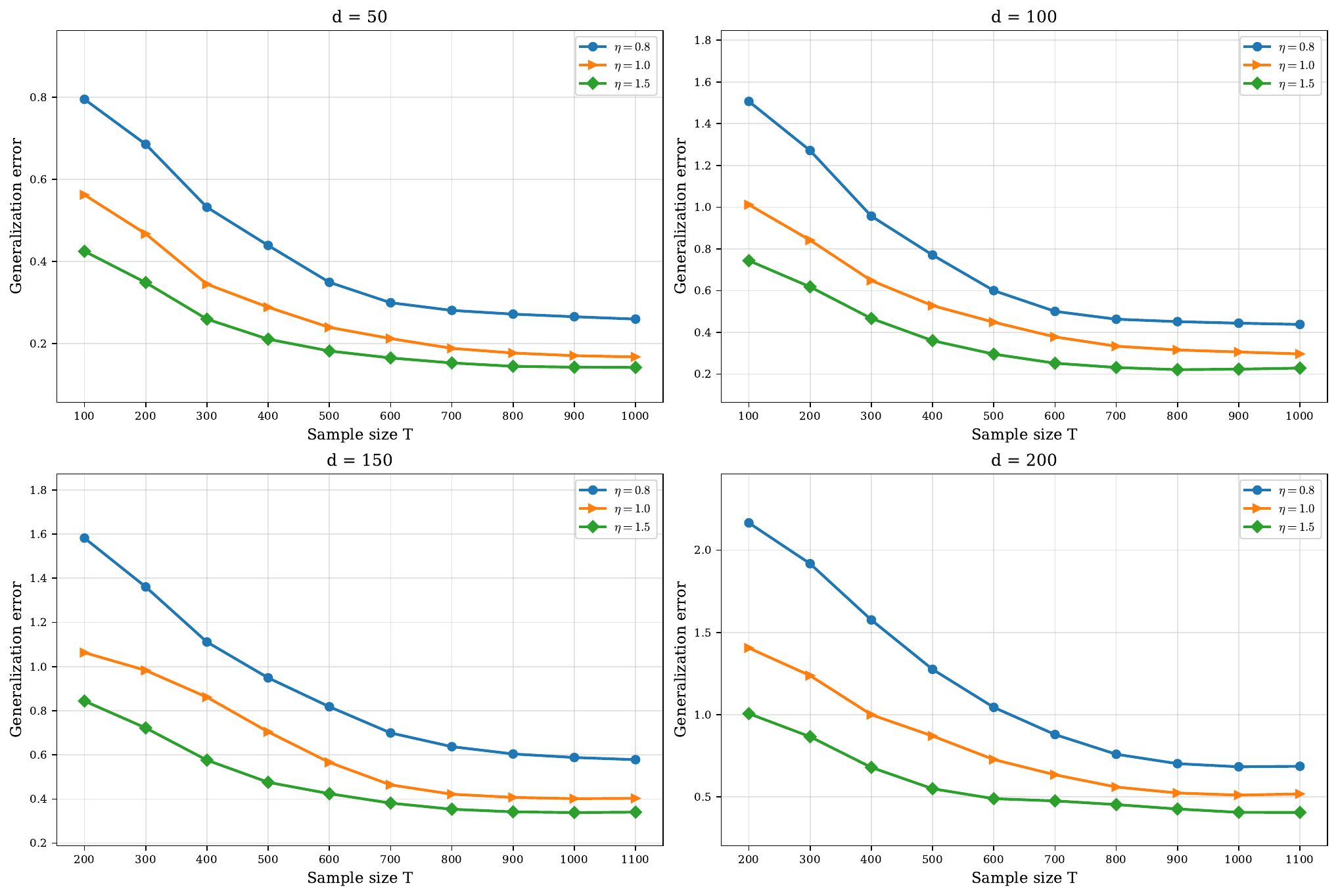}
\captionsetup{font=scriptsize}
\caption{Generalization error of weighted total variation estimator under Sub-Weibull noise.}
\label{fig:subweibull_wtv}
\end{figure}

\begin{figure}[htbp]
\centering
\includegraphics[width=0.8\textwidth]{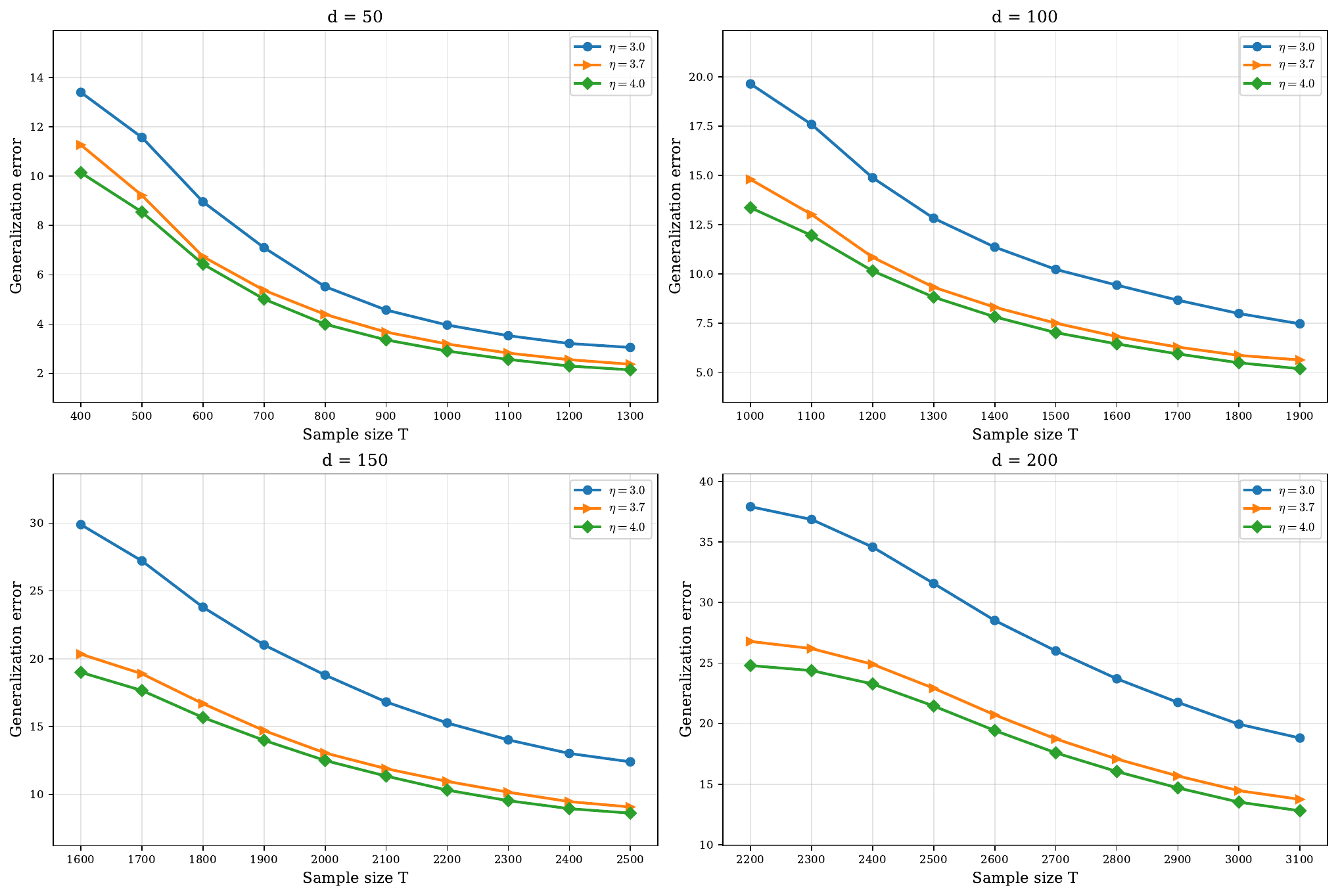}
\captionsetup{font=scriptsize}
\caption{Generalization error of the Lasso estimator under Pareto noise.}
\label{fig:pareto_lasso}
\end{figure}

\begin{figure}[htbp]
\centering
\includegraphics[width=0.8\textwidth]{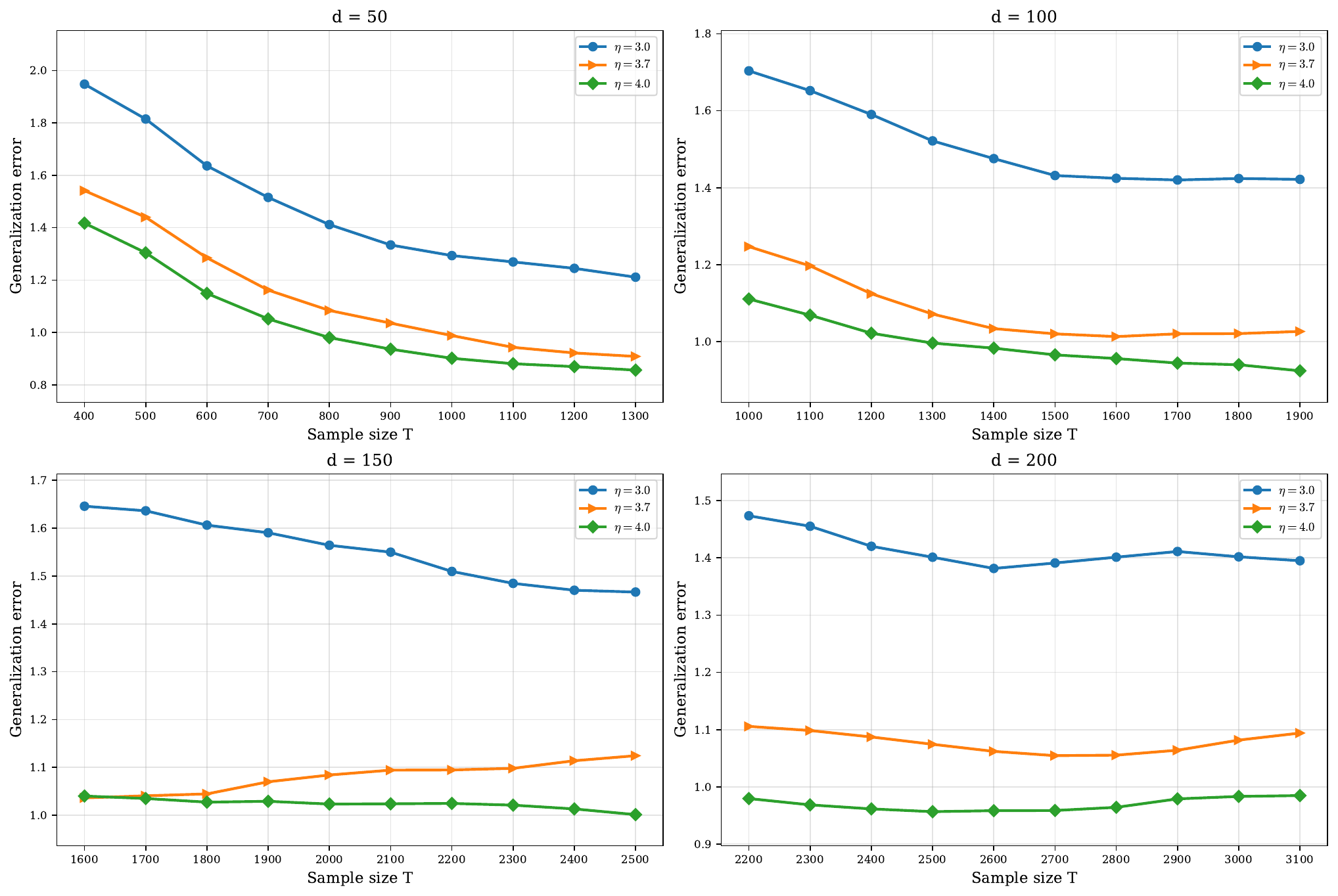}
\captionsetup{font=scriptsize}
\caption{Generalization error of weighted total variation estimator under Pareto noise.}
\label{fig:pareto_wtv}
\end{figure}

\paragraph{Effect of tail index.}
Across all scenarios, the tail behavior of the noise has a clear impact on the prediction error.
For both estimators and for all considered dimensions $d$, a larger tail index $\eta$, corresponding to a lighter-tailed distribution, consistently leads to a smaller generalization error.
Under sub-Weibull errors (Figures~\ref{fig:subweibull_lasso}-\ref{fig:subweibull_wtv}), the ordering of the curves with respect to $\eta$ is stable for both Lasso and WTV. 
Under Pareto errors (Figures~\ref{fig:pareto_lasso}-\ref{fig:pareto_wtv}), the same qualitative trend remains visible but becomes weaker, and the curves corresponding to different $\eta$ become less separated and may appear nearly indistinguishable in higher dimensions.
This reflects the increased difficulty of learning under regularly varying tails, where extreme observations and reduced concentration lead to stronger finite sample variability. 
This is fully consistent with our theoretical error bounds, which depend explicitly on the noise tail-growth parameter.

\paragraph{Effect of the dimension and sample size.}
For fixed $(d,\eta)$, the prediction error generally decreases as the sample size $T$ increases.
Larger dimensions lead to higher error levels, which reflect the increased difficulty of estimation in high-dimensional settings.
For moderate sample sizes, the error curves sometimes display small non-monotonic fluctuations or short plateaus, especially under heavy-tailed noise.
This behavior is typical in locally stationary problems, where finite sample variability, kernel smoothing, and regularization interact to determine the bias-variance tradeoff when data are limited.

\paragraph{Comparison between estimators.}
In most settings, the WTV estimator achieves a lower prediction error than the Lasso, with the improvement being most pronounced under Sub-Weibull noise and for moderate to large sample sizes.
This behavior is consistent with the fact that the weighted total variation penalty can better exploit the block structure of the true coefficient surface along the feature dimension.
Under Pareto noise, the advantage of WTV is still present but less uniform, particularly in higher dimensions, where the two methods can exhibit similar performance for moderate sample sizes.
Across both noise families, WTV remains competitive and often improves upon Lasso, which aligns with the intended structural regularization and provides empirical support for the theoretical analysis.

Overall, the experiments indicate that heavy-tailed noise leads to systematically larger prediction errors, which is consistent with the theoretical rates derived in the paper.
In addition, incorporating structural information along the feature dimension through the WTV penalty often improves predictive performance, especially under Sub-Weibull noise and for larger sample sizes.
Under Pareto noise, the advantage of WTV is less uniform but remains visible in many settings, suggesting that structural regularization can still provide robustness benefits in challenging heavy-tailed regimes.

\subsection{Verification of the Assumptions}
We briefly verify that the simulation design is compatible with
Assumptions~\ref{Ass:1}--\ref{Ass:KB2} underlying our theoretical analysis.
\paragraph{Local stationarity.}
Recall that the covariate process is generated as
\begin{equation*}
X_{t,T}=m\big(\frac{t}{T}\big)+\sum_{r=0}^{m} a_r\!\big(\frac{t}{T}\big) Z_{t-r},\qquad t=1,\ldots,T,  
\end{equation*}
where $\{Z_t\}_{t\in\mathbb{Z}}$ is an i.i.d.\ innovation sequence and the coefficient
functions $m(\cdot)$ and $a_r(\cdot)$ are smooth on $[0,1]$.
For each fixed $u\in[0,1]$, define the associated process
\begin{equation*}
X_t(u):=m(u)+\sum_{r=0}^{m} a_r(u)\,Z_{t-r},\qquad t\in\mathbb{Z}.
\end{equation*}
Since the coefficients are constant for fixed $u$ and $\{Z_t\}$ is i.i.d.,
$\{X_t(u)\}_{t\in\mathbb{Z}}$ is a finite-order linear filter of $\{Z_t\}$ and hence
strictly stationary, verifying condition~(i) of Definition~\ref{def:locallystseries}.

To verify condition~$(ii)$, let $u_t=t/T$. By the triangle inequality,
\begin{align*}
\norm{X_{t,T}-X_t(u)} \le \norm{m(u_t)-m(u)}+\sum_{r=0}^{m}\norm{a_r(u_t)-a_r(u)}\,|Z_{t-r}|.
\end{align*}
In the simulations, $m(\cdot)$ and $a_r(\cdot)$ are componentwise sinusoidal and therefore
uniformly Lipschitz on $[0,1]$. Hence, there exist constants $L_m,L_{a_r}<\infty$ such that
\begin{equation*}
\norm{m(u_t)-m(u)}\le L_m |u_t-u|,\qquad\norm{a_r(u_t)-a_r(u)}\le L_{a_r}|u_t-u|.
\end{equation*}
It follows that
\begin{equation*}
\norm{X_{t,T}-X_t(u)}\le|u_t-u|\Big(L_m+\sum_{r=0}^{m}L_{a_r}|Z_{t-r}|\Big)\le\Big(|u_t-u|+\tfrac{1}{T}\Big) U_{t,T}(u),
\end{equation*}
where
\begin{equation*}
U_{t,T}(u)
:=
L_m+\sum_{r=0}^{m}L_{a_r}|Z_{t-r}|.
\end{equation*}
Since $\{Z_t\}$ admits finite moments of all orders under the considered noise
distributions, there exists $\rho>0$ such that $\mathbb{E}\big[(U_{t,T}(u))^{\rho}\big]<C$ uniformly in $t,T$ and $u$. This verifies condition~(ii) of Definition~\ref{def:locallystseries}, and hence $\{X_{t,T}\}$ is locally stationary.

\paragraph{$\beta$-mixing property.}
By construction, $X_{t,T}$ depends only on the finite collection
$\{Z_t,\ldots,Z_{t-m}\}$. Consequently, the process $\{X_{t,T}\}_{t\in\mathbb{Z}}$ is strictly $m$-dependent, that is, the $\sigma$-fields generated by $\{X_{t,T}\}_{t\le s}$ and $\{X_{t,T}\}_{t\ge s+k}$ are independent for all $k>m$. It follows that the $\beta$-mixing coefficients of $\{X_{t,T}\}$ satisfy $\beta_X(k)=0$ with $ k>m$.
The regression noise sequence $\{\varepsilon_{t,T}\}$ is generated independently across time and independently of the covariates. Therefore, the joint array 
$\{(X_{t,T},\varepsilon_{t,T})\}_{t\in\mathbb{Z}}$ is also strictly $m$-dependent, and its $\beta$-mixing coefficients satisfy $\beta(k)=0$ with $k>m$.
Hence, Assumption~\ref{Ass:2} is satisfied. Although the simulation design is based on finite dependence, the theoretical analysis only requires exponentially decaying $\beta$-mixing coefficients. Hence the simulation setting is compatible with, and in fact stronger than, the assumed dependence condition.

\paragraph{Kernel regularity.)}
The basic kernels $K_1$ and $K_2$ in (\ref{eq:sim-kernels}) are symmetric around zero, bounded, have compact support contained in $[-1,1]$, and are Lipschitz continuous on $\mathbb{R}$. Hence, Assumption~\ref{Ass:KB1} is also satisfied under our simulation design.

\section{Conclusion} \label{sec:conclusion}
We introduce a flexible sparse learning framework designed for two classes of heavy-tailed distributions: Sub-Weibull distributions and regularly varying tail distributions, focusing on high-dimensional data modeling under local stationarity. We derive oracle inequalities under the least squares loss for both Lasso penalization and weighted total variation penalization. 
Under Assumptions, we first establish a class of oracle inequalities with relatively slow convergence rates, effectively linking prediction error to the regularization terms of the regression vector. Furthermore, under restricted eigenvalue conditions, we derive oracle inequalities that exhibit faster convergence rates. These theoretical results demonstrate that the error bounds for sparse estimation can be substantially improved, thereby enhancing the robustness and predictive accuracy of the model. 
The proposed framework is capable of accommodating different forms of heavy-tailed behavior and captures complex sparsity structures through adaptive regularization. It shows strong adaptability and wide applicability, particularly in high-dimensional settings characterized by locally stationary. This work provides a new perspective for constructing sparse learning models that are not only theoretically sound but also practically effective.

\appendix

\section{Proofs of concentration inequalities} \label{sec:Proofs_of_Proposition}

\subsection{Proof of Proposition~\ref{proposition.Locally-stationary-sub-Weibull-distribution}}
Let $ \Phi\Big(\frac{r}{T}, X_{r,T}^{j} \Big)=\frac{1}{T}\sum_{t=1}^{T} W_{t,r,T}^{j}=\frac{1}{T}\sum_{t=1}^{T}K_{h,1}\big(\frac{t}{T}-\frac{r}{T}\big) K_{h,2}(X_{t,T}^{j}-X_{r, T}^{j}) \varepsilon_{t,T}, \text{ for } t=1,\ldots,T$. Set $\tau_{T}=T \log T $, We write
\begin{equation*}
\Phi \Big(\frac{r}{T}, X_{r,T}^{j} \Big) = \Phi_{1} \Big(\frac{r}{T}, X_{r,T}^{j} \Big) + \Phi_{2} \Big(\frac{r}{T}, X_{r,T}^{j} \Big),  
\end{equation*} 
where 
\begin{align*}
\Phi_{1} \Big(\frac{r}{T}, X_{r,T}^{j} \Big)=\frac{1}{T} \sum_{t=1}^{T} K_{h,1}\big(\frac{t}{T}-\frac{r}{T}\big) K_{h,2}(X_{t,T}^{j}-X_{r, T}^{j}) \varepsilon_{t, T} \ind{\big(\big|\varepsilon_{t, T}\big| \leq \tau_{T}\big)}, \\ 
\Phi_{2} \Big(\frac{r}{T}, X_{r,T}^{j} \Big)=\frac{1}{T} \sum_{t=1}^{T} K_{h,1}\big(\frac{t}{T}-\frac{r}{T}\big) K_{h,2}(X_{t,T}^{j}-X_{r, T}^{j}) \varepsilon_{t, T} \ind{\big(\big|\varepsilon_{t, T}\big|>\tau_{T}\big)}.    
\end{align*}
It follows that 
\begin{equation}\label{equationbeta11}
\mathbb P \Big(\big|\Phi\Big(\frac{r}{T}, X_{r,T}^{j} \Big)\big| > 2\gamma\Big) \leq\mathbb  P \Big(\big|\Phi_{1}\Big(\frac{r}{T}, X_{r,T}^{j} \Big)\big| > \gamma\Big)+ \mathbb P \Big(\big|\Phi_{2}\Big(\frac{r}{T}, X_{r,T}^{j} \Big)\big| > \gamma\Big). 
\end{equation}
For $ \Phi_{2} \Big(\frac{r}{T}, X_{r,T}^{j} \Big) $, defining $ b_{T}=\sqrt{\log T / T }$, then we have $\tau_{T} > b_{T}$ and for any $\gamma \geq C_{K} b_{T}$, where $C_{K}=C_{K_{1}}C_{K_{2}}$, it has that 
\begin{equation}
\begin{aligned}
\mathbb P\Big(\Big|\Phi_{2} \Big(\frac{r}{T}, X_{r,T}^{j} \Big) \Big| \geq \gamma \Big)
& \leq\mathbb  P\Big( \big|\Phi_{2} \Big(\frac{r}{T}, X_{r,T}^{j} \Big) \big| \geq C_{K} b_{T} \Big)\\
& =\mathbb  P\Big(\Big|\frac{1}{T} \sum_{t=1}^{T} K_{h,1}\big(\frac{t}{T}-\frac{r}{T}\big) K_{h,2}(X_{t,T}^{j}-X_{r, T}^{j}) \varepsilon_{t, T} \ind{\big(\big|\varepsilon_{t, T}\big|>\tau_{T}\big)}\big| \geq C_{K} b_{T} \Big)\\
& \leq\mathbb  P\Big(\Big|\frac{1}{T} C_{K} \sum_{t=1}^{T} \varepsilon_{t, T} \ind{\big(\big|\varepsilon_{t, T}\big|>\tau_{T}\big)}\Big| \geq C_{K} b_{T} \Big) \\
& \leq\mathbb P\Big(\Big| \varepsilon_{t, T} \ind{\big(\big|\varepsilon_{t, T}\big|>\tau_{T}\big)}\Big| \geq b_{T} \Big) \\
& \leq \mathbb P\Big(\big|\varepsilon_{t,T}\big|>\tau_{T} \Big)\ind{(\tau_{T}>b_{T})}+ \mathbb P\Big(\big|\varepsilon_{t,T}\big|>\tau_{T} \Big)\ind{(\tau_{T} \leq b_{T})}\\
& \leq \mathbb P\Big(\big|\varepsilon_{t,T}\big|>\tau_{T}, \text{ for some }  1 \leq t \leq T \Big) \\
& \leq \exp (-\big( \tau_{T}/C_{\varepsilon} \big)^{\eta_{2}})\\
& \leq \exp\big(-(\frac{T\log T}{C_{\varepsilon}})^{\eta_{2}}\big).
\end{aligned}\label{equationbeta2}
\end{equation}
We now turn to the analysis of $ \Phi_{1} \Big(\frac{r}{T}, X_{r,T}^{j} \Big) $. From Assumption~\ref{Ass:1}, $\{X_{t,T}\}_{t=1}^{T}$ is a locally stationary sequence, which can be approximated locally by a strictly stationary sequence $\{X_{t}(u )\}_{t \in \mathbb{Z}}$. Since $K_{1}$ and $K_{2}$ are Lipschitz and bounded, with Remark~\ref{remark.U_{t,T}}, i.e., $\norm{X_{t, T}-X_{t}(u)} \leq \big(\big|\frac{t}{T}-u\big|+\frac{1}{T}\big) U_{t, T}(u) \leq C_{U} \big(\big|\frac{t}{T}-u\big|+\frac{1}{T}\big)$, where $u \in [0,1]$ and $C_{U}$ is a constant, we can infer that
\begin{align*}
& \big|K_{h,1}\big(\frac{t}{T}-\frac{r}{T}\big) K_{h,2}(X_{t,T}^{j}-X_{r, T}^{j})-K_{h,1}\big(u-\frac{r}{T}\big) K_{h,2}(X_{t}^{j}(u)-X_{r, T}^{j})\big|\\
&= \big|K_{h,1}\big(\frac{t}{T}-\frac{r}{T}\big) K_{h,2}(X_{t,T}^{j}-X_{r, T}^{j})-K_{h,1}\big(\frac{t}{T}-\frac{r}{T}\big) K_{h,2}(X_{t}^{j}(u)-X_{r, T}^{j}) \\
&\quad+ K_{h,1}\big(\frac{t}{T}-\frac{r}{T}\big) K_{h,2}(X_{t}^{j}(u)-X_{r, T}^{j})-K_{h,1}\big(u-\frac{r}{T}\big) K_{h,2}(X_{t}^{j}(u)-X_{r, T}^{j})\big| \\
&\leq K_{h,1}\big(\frac{t}{T}-\frac{r}{T}\big) \big|K_{h,2}(X_{t,T}^{j}-X_{r, T}^{j})- K_{h,2}(X_{t}^{j}(u)-X_{r, T}^{j})\big|\\
&\quad+ \big|K_{h,1}\big(\frac{t}{T}-\frac{r}{T}\big)-K_{h,1}\big(u-\frac{r}{T}\big)\big| K_{h,2}(X_{t}^{j}(u)-X_{r, T}^{j}) \\
&\leq K_{h,1}\big(\frac{t}{T}-\frac{r}{T}\big) \frac{L_{K_{2}}}{h}|X_{t,T}^{j}-X_{t}^{j}(u)|]+\frac{L_{K_{1}}}{h}|\frac{t}{T}-u| K_{h,2}(X_{t}^{j}(u)-X_{r, T}^{j}) \\
&\leq \frac{C_{K_{1}}L_{K_{2}}}{h} \big(\big|\frac{t}{T}-u\big|+\frac{1}{T}\big) U_{t, T}(u) + \frac{C_{K_{2}}L_{K_{1}}}{h}|\frac{t}{T}-u|\\
& \leq \frac{C_{U}C_{K_{1}}L_{K_{2}}}{h} \big(\big|\frac{t}{T}-u\big|+\frac{1}{T}\big) + \frac{C_{K_{2}}L_{K_{1}}}{h}|\frac{t}{T}-u|\\
&\leq \frac{C_{K,L}}{h} \big(1+\frac{1}{T}\big)+\frac{C_{K,L}}{h}\\
& \leq \frac{C_{K,L}(2T+1)}{Th},
\end{align*}
where $C_{K,L}=\max\{C_{U}C_{K_{1}}L_{K_{2}},C_{K_{2}}L_{K_{1}}\}$.
Defining
\begin{equation*}
\tilde{\Phi}_{1}(\frac{r}{T}, X_{r, T}^{j})=\frac{1}{T } \sum_{t=1}^{T} K_{h,1}\big(u-\frac{r}{T}\big) K_{h,2}(X_{t}^{j}(u)-X_{r, T}^{j}) \varepsilon_{t,T} \ind{\big(\big|\varepsilon_{t,T}\big| \leq \tau_{T}\big)}, 
\end{equation*}
we have 
\begin{align*}
&\Phi_{1} \big(\frac{r}{T}, X_{r, T}^{j} \big)-\tilde{\Phi}_{1} \big(\frac{r}{T}, X_{r, T}^{j}\big) \\
&\leq \frac{1}{T } \sum_{t=1}^{T} \Big(K_{h,1}\big(\frac{t}{T}-\frac{r}{T}\big) K_{h,2}(X_{t,T}^{j}-X_{r, T}^{j})-K_{h,1}\big(u-\frac{r}{T}\big) K_{h,2}(X_{t}^{j}(u)-X_{r, T}^{j})\Big)\varepsilon_{t,T} \ind{\big(\big|\varepsilon_{t,T}\big| \leq \tau_{T}\big)}\\
&\leq \frac{1}{T } \sum_{t=1}^{T} \frac{C_{K,L}(2T+1)}{Th} \varepsilon_{t,T} \ind{\big(\big|\varepsilon_{t,T}\big| \leq \tau_{T}\big)},
\end{align*}
we can have that
\begin{align*}
\Phi_{1} \big(\frac{r}{T}, X_{r, T}^{j} \big)&= \Phi_{1} \big(\frac{r}{T}, X_{r, T}^{j} \big)-\tilde{\Phi}_{1} \big(\frac{r}{T}, X_{r, T}^{j}\big) +\tilde{\Phi}_{1} \big(\frac{r}{T}, X_{r, T}^{j}\big) \\
&\leq \frac{1}{T } \sum_{t=1}^{T} \frac{C_{K,L}(2T+1)}{Th} \varepsilon_{t,T} \ind{\big(\big|\varepsilon_{t,T}\big| \leq \tau_{T}\big)} +\tilde{\Phi}_{1} \big(\frac{r}{T}, X_{r, T}^{j}\big).
\end{align*} 
It follows that
\begin{equation}\label{equationbeta13}
\mathbb P \Big( \Big|\Phi_{1} \Big(\frac{t}{T}, X_{t,T}^{j} \Big)\Big| \geq \gamma \Big)
\leq Q_{T}+\tilde{Q}_{T},
\end{equation}
where 
\begin{equation*}
Q_{T}=\mathbb P \bigg( |\frac{1}{T } \sum_{t=1}^{T} \frac{C_{K,L}(2T+1)}{Th} \varepsilon_{t,T} \ind{\big(\big|\varepsilon_{t,T}\big| \leq \tau_{T}\big)} |\geq \frac{\gamma}{{2}} \bigg),
\end{equation*}
and
\begin{equation*}
\tilde{Q}_{T}=\mathbb P \Big( |\tilde{\Phi}_{1} \big(\frac{r}{T}, X_{r, T}^{j}\big)| \geq \frac{\gamma}{2} \Big).
\end{equation*}
To bound $ \tilde{Q}_{T} $, we write
\begin{equation*}
\tilde{Q}_{T}\leq \mathbb P\Big(\frac{1}{T}\Big|\sum_{t=1}^{T} Z_{t, T}(u,X_{t}^{j}(u))\Big|\geq \frac{\gamma}{2}\Big), 
\end{equation*}
with
\begin{equation*}
Z_{t, T}(u, X_{t}^{j}(u))= K_{h,1}\big(u-\frac{r}{T}\big) K_{h,2}(X_{t}^{j}(u)-X_{r, T}^{j}) \varepsilon_{t,T} \ind{\big(\big|\varepsilon_{t,T}\big| \leq \tau_{T}\big)}.     
\end{equation*}
Note that $K_{1}$ and $K_{2}$ are bounded, from Assumption~\ref{Ass:2} and Proposition~\ref{lemma_sub_beta_mixing}, we have $\{ \varepsilon_{t,T} \} $ is $\beta$-mixing sequence, i.e., $\{ \varepsilon_{t,T} \} $ follows the $\beta$-mixing sub-Weibull distribution with mixing coefficients $ \beta(k) \leq \exp (-\varphi k^{\eta_{1}}) $, for some $ \varphi>0, \eta_{1}>1$. We now bound $\tilde{Q}_{T}$ with the help of the exponential inequality in Lemma~\ref{lemma.sumofmixingsubweibull}. For $T>4$ and $\gamma>\frac{2C_{K}}{T}$, i.e., $\frac{\gamma}{2C_{K}}>1/T$, we have
\begin{align}\label{equationbeta14}
\tilde{Q}_{T}&\leq \mathbb P\Big(\frac{1}{T}\Big|\sum_{t=1}^{T} Z_{t, T}(u, X_{t}^{j}(u))\Big|\geq \frac{\gamma}{2}\Big)\\ \nonumber
& \leq \mathbb P\Big(\frac{C_{K}}{T}\Big|\sum_{t=1}^{T}\varepsilon_{t,T} \ind{\Big(\big|\varepsilon_{t,T}\big| \leq \tau_{T}\Big)}\Big|\geq \frac{\gamma}{2}\Big)\\
&\leq T \exp \bigg( -\frac{(\gamma T)^{\eta}}{(2C_{K}C_{\varepsilon})^{\eta}C_{1}} \bigg) + \exp \bigg( -\frac{\gamma^{2}T}{(2C_{K}C_{\varepsilon})^{2}C_{2}} \bigg),\nonumber
\end{align}
where $1/\eta = 1/\eta_{1} + 1/\eta_{2}$, $ \eta < 1$, $C_{K}=C_{K_{1}}C_{K_{2}}$, $C_{\varepsilon}$ is the sub-Weibull constant and the constants $C_{1}$, $C_{2}$ depend only on $\eta_{1}$, $\eta_{2}$ and $\varphi$. 

We now bound $ Q_{T}$ with the help of the exponential inequality in Lemma~\ref{lemma.sumofmixingsubweibull}. For $T>4$ and $\gamma>\frac{2C_{K,L}(2T+1)}{T^{2}h}$, i.e., $\frac{\gamma Th}{2C_{K,L}(2T+1)}>1/T$, we have
\begin{align}\label{equationbeta15}
Q_{T}&= \mathbb P \bigg(\frac{1}{T} |\sum_{t=1}^{T} \frac{C_{K,L}(2T+1)}{Th} \varepsilon_{t,T} \ind{\Big(\big|\varepsilon_{t,T}\big| \leq \tau_{T}\Big)}|\geq \frac{\gamma}{2}\bigg)\\ \nonumber
&\leq T \exp \bigg( -\frac{(\gamma T^{2}h)^{\eta}}{(2C_{K,L}(2T+1)C_{\varepsilon})^{\eta}C_{1}} \bigg) + \exp \bigg( -\frac{(\gamma h)^{2}T^{3}}{(2C_{K,L}(2T+1)C_{\varepsilon})^{2}C_{2}} \bigg).
\end{align}

From (\ref{equationbeta11})-(\ref{equationbeta15}), let $h=\bigO(T^{-\xi})$ with $0<\xi<\frac{1}{2}$, for 
\begin{align*}
    \gamma>\max\{C_{K}\sqrt{\log T / T },\frac{2C_{K}}{T},\frac{2C_{K,L}(2T+1)}{T^{2}h}\}= C_{K}\sqrt{\log T / T } 
\end{align*}
and $T>4$, we further get that
\begin{align*}
\mathbb P (\big|\Phi\Big(&\frac{t}{T}, X_{t,T}^{j} \Big)\big| > 2\gamma)\\
&\leq \exp\big(-(\frac{T\log T}{C_{\varepsilon}})^{\eta_{2}}\big)\\
& \quad + T \exp \bigg( -\frac{(\gamma T)^{\eta}}{(2C_{K}C_{\varepsilon})^{\eta}C_{1}} \bigg) + \exp \bigg( -\frac{\gamma^{2}T}{(2C_{K}C_{\varepsilon})^{2}C_{2}} \bigg)\\
& \quad + T \exp \bigg( -\frac{(\gamma T^{2}h)^{\eta}}{(2C_{K,L}(2T+1)C_{\varepsilon})^{\eta}C_{1}} \bigg) + \exp \bigg( -\frac{(wh)^{2}T^{3}}{(2C_{K,L}(2T+1)C_{\varepsilon})^{2}C_{2}} \bigg).
\end{align*} 
Then, let $h=\bigO(T^{-\xi})$ with $0<\xi<\frac{1}{2}$, for $\gamma> 2C_{K}\sqrt{\log T / T }$ and $T>4$, we have
\begin{align*}
\mathbb P (\big|\Phi\Big(&\frac{r}{T}, X_{r, T}^{j}\Big)\big| > \gamma)\\
&\leq \exp\big(-(\frac{T\log T}{C_{\varepsilon}})^{\eta_{2}}\big)\\
& \quad + T \exp \bigg( -\frac{(\gamma T)^{\eta}}{(4C_{K}C_{\varepsilon})^{\eta}C_{1}} \bigg) + \exp \bigg( -\frac{\gamma^{2}T}{(4C_{K}C_{\varepsilon})^{2}C_{2}} \bigg)\\
& \quad + T \exp \bigg( -\frac{(\gamma T^{2}h)^{\eta}}{(4C_{K,L}(2T+1)C_{\varepsilon})^{\eta}C_{1}} \bigg) + \exp \bigg( -\frac{(wh)^{2}T^{3}}{(4C_{K,L}(2T+1)C_{\varepsilon})^{2}C_{2}} \bigg).
\end{align*}

\subsection{Proof of Proposition~\ref{proposition.Regularly varying heavy-tailed}}
Let $Z_{t,T}^{N}$ denote the truncated random variable $ Z_{t,T}$ such that 
\begin{equation*}
Z_{t,T}^{N}= \max \Big(\min \Big(Z_{t,T}, N\Big),-N\Big) \text{ w.r.t. } N \in N^{*} .   
\end{equation*}
Then define $\Sigma_{T}^{N}:=\sum_{t=1}^{T} Z_{t,T}^{N}$ and  
consider the partition of the samples into blocks of length $ B, I_{i}=\{1+(i-1) B, \ldots, i B\} $ for $ i=1,2, \ldots,2\mu $ where $ \mu =[T /(2B)] $. Also let $ I_{2\mu +1}=\Big\{2\mu B+1, \ldots, T\Big\} $. Define for a finite set $ I $ of positive integers, define $ \Sigma_{T}^{N}(I)= \sum_{t \in I} Z_{t,T}^{N} $. Then we can write, for $ l \leq T $
\begin{align*}
\Sigma_{l} & =\sum_{t=1}^{l}Z_{t,T} = \sum_{t=1}^{l}\Big(Z_{t,T}-Z_{t,T}^{N}\Big)+\sum_{t=1}^{l} Z_{t,T}^{N} \\
& =\sum_{t=1}^{l}\big(Z_{t,T}-Z_{t,T}^{N}\big)+\sum_{j \leq[l / B]} \Sigma_{T}^{N}\big(I_{2j}\big)+\sum_{j \leq[l/B]} \Sigma_{T}^{N}\big(I_{2j-1}\big)+\sum_{t=B[l / B]+1}^{l} Z_{t,T}^{N}.
\end{align*}
Then, we have
\begin{align*}
\Big|\Sigma_{l}\Big| &\leq \sum_{t=1}^{l}\big| Z_{t,T}-Z_{t,T}^{N}\big|+\Big|\sum_{j \leq[l/B]} \Sigma_{T}^{ N}\big(I_{2j}\big)\big|+\big|\sum_{j \leq[l/B]} \Sigma_{T}^{ N}\big(I_{2j-1}\big)\big|+2B N.
\end{align*}
Thus,
\begin{align*}
\sup_{l \leq T}\big|\Sigma_{l}\big| &\leq \sum_{t=1}^{T}\big|Z_{t,T}-Z_{t,T}^{N}\big|+\sup_{l \leq T}\big|\sum_{j \leq[l / B]} \Sigma_{T}^{ N}\big(I_{2j}\big)\big|+\sup_{l \leq T}\big|\sum_{j \leq[l / B]} \Sigma_{T}^{ N}\big(I_{2j-1}\big)\big|+2B N\\
&\leq \sum_{t=1}^{T}\big|Z_{t,T}-Z_{t,T}^{N}\big|+\sup_{l \leq T}\big|\sum_{j \leq[l / B]}\big(\Sigma_{T}^{ N}\big(I_{2j}\big)-\Sigma_{T}^{ N*}\big(I_{2j}\big)\big)\big|+\sup_{l \leq T}\big|\sum_{j \leq[l / B]} \Sigma_{T}^{ N*}\big(I_{2j}\big)\big|\\
&\quad+\sup_{l \leq T}\big|\sum_{j \leq[l / B]}\big(\Sigma_{T}^{ N}\big(I_{2j-1}\big)-\Sigma_{T}^{ N*}\big(I_{2j-1}\big)\big)\big|+\sup_{l \leq T}\big|\sum_{j \leq[l / B]} \Sigma_{T}^{ N*}\big(I_{2j-1}\big)\big|+2B N.
\end{align*}
Using Markov's inequality and Lemma \ref{Karamata_theorem}, for $\eta_{2} \geq 2$, we infer readily 
\begin{align*}
\mathbb{P}\Big(&\sum_{t=1}^{T}\Big|Z_{t,T}-Z_{t,T}^{N}\Big|   \\
& \leq \frac{1}{\varrho} \big[\sum_{t=1}^{T}\int_{0}^{\infty} \Big(\mathbb{P} (\big| Z_{t,T} - Z^N_{t,T} \big| >x)\ind{\{Z_{t,T}<-N\}} + \mathbb{P} (\big| Z_{t,T} - Z^N_{t,T} \big| >x)\ind{\{|Z_{t,T}|<N\}}\\
&\qquad + \mathbb{P} (\big| Z_{t,T} - Z^N_{t,T} \big| >x)\ind{\{Z_{t,T}>N\}}\Big) dx \big]\\
&= \frac{1}{\varrho} \big[\sum_{t=1}^{T}\int_{0}^{\infty} \mathbb{P} (\big| Z_{t,T} +N \big| >x)\ind{\{Z_{t,T}<-N\}} dx  + \int_{0}^{\infty} \mathbb{P} ( Z_{t,T} - N >x)\ind{\{Z_{t,T}>N\}} dx ]\\
&= \frac{1}{\varrho} \big[\sum_{t=1}^{T}\int_{0}^{\infty} \mathbb{P} (-Z_{t,T}>x+N)\ind{\{-Z_{t,T}>N\}}dx + \int_{N}^{\infty}\mathbb{P} ( Z_{t,T}>x)\ind{\{Z_{t,T}>N\}} dx \big]\\
&= \frac{1}{\varrho} \big[\sum_{t=1}^{T}\int_{N}^{\infty} \mathbb{P} (-Z_{t,T}>x)\ind{\{-Z_{t,T}>N\}}dx  + \int_{N}^{\infty}\mathbb{P} ( Z_{t,T}>x)\ind{\{Z_{t,T}>N\}} dx \big]\\
&\leq \frac{2}{\varrho} \sum_{t=1}^{T} \int_{N}^{\infty} \mathbb{P}\Big(\Big|Z_{t,T}\Big| \geq x\Big) dx \\ 
&\leq \frac{2T}{\varrho} \int_{N}^{\infty} x^{-\eta_{2}}L(x) dx\\
&=\frac{2T}{\varrho}\frac{1}{\eta_{2}-1}N^{1-\eta_{2}}L(N).
\end{align*}
Let $ Z_{t, T}^{*N}$, $t \in I $, be an independent random variables and have the same distribution as $ Z_{t, T}^{N}$, $t \in I $, and $ \Sigma_{T}^{ N*}(I)=\sum_{t \in I} Z_{t, T}^{*N}$. Using Lemma~\ref{lemma_lemma5_of_Dedecker}, we have
\begin{align*}
\mathbb E\Big[\Big|\Sigma_{T}^{ N}\Big(I_{2j}\Big)-\Sigma_{T}^{ N*}\Big(I_{2j}\Big)\Big|\Big]
&=\mathbb E\Big[\Big|\sum_{t \in I_{2j}} Z_{t, T}^{N}-\sum_{t \in I_{2j}} Z_{t, T}^{*N}\Big|\Big]\\
&=\mathbb E\Big[\Big|\sum_{t \in I_{2j}} (Z_{t, T}^{N}-Z_{t, T}^{*N})\Big|\Big]\\
&\leq B\tau(B).
\end{align*}
Then using Markov's inequality, we have
\begin{align*}  
\mathbb{P}\Big(\sup_{l \leq T}\Big|&\sum_{t \leq[l / B]}\Big(\Sigma_{T}^{ N}\Big(I_{2j}\Big)-\Sigma_{T}^{ N*}\Big(I_{2j}\Big)\Big)\Big| \geq \varrho \Big) \\ 
&\leq \frac{\mathbb{E}\Big[\sup_{l \leq T}\Big|\sum_{t \leq[l / B]}\Big(\Sigma_{T}^{ N}\Big(I_{2j}\Big)-\Sigma_{T}^{ N*}\Big(I_{2j}\Big)\Big)\Big|\Big]}{\varrho} \\ 
&\leq \frac{\mathbb{E}\Big[\sup_{l\leq T} \sum_{t \leq \mu}\Big|\Sigma_{T}^{ N}\Big(I_{2j}\Big)-\Sigma_{T}^{ N*}\Big(I_{2j}\Big)\Big|\Big]}{\varrho}\\
&\leq \frac{\mu B\tau(B)}{\varrho}.
\end{align*}
The same results holds for $\Big\{\Sigma_{T}^{ N}\Big(I_{2j-1}\Big)-\Sigma_{T}^{ N*}\Big(I_{2j-1}\Big)\Big\}_{j=1, \ldots, k} $. So for any $\varrho \geq 2BN$, we have
\begin{equation}
\begin{aligned}
\mathbb{P}\Big(\sup_{l \leq T}\Big|\Sigma_{l}\Big| \geq 6\varrho \Big) &\leq \frac{2T}{\varrho}\frac{1}{\eta_{2}-1}N^{1-\eta_{2}}L(N)+\frac{2\mu B\tau(B)}{\varrho}  \\
&\quad + \mathbb{P}\Big(\sup_{l\leq T}\Big|\sum_{j \leq[l / B]} \Sigma_{T}^{ N*}\Big(I_{2j}\Big)\Big| \geq \varrho \Big)\\
&\quad +\mathbb{P}\Big(\sup_{l\leq T}\Big|\sum_{j \leq[l / B]} \Sigma_{T}^{ N*}\Big(I_{2j-1}\Big)\Big| \geq \varrho \Big).
\end{aligned}\label{equation.17}
\end{equation}
Since the variable $\Sigma_{T}^{ N*}\Big(I_{2j}\Big)$ are independent and centered, $\Big|\Sigma_{T}^{ N*}\Big(I_{2j}\Big)\Big|\leq BN $, by the Hoeffding's inequality in Lemma~\ref{lemma_Hoeffding's inequality}, we obtain the bound
\begin{eqnarray} \label{equation.18}
\mathbb{P}\bigg(\sup_{l \leq T}\Big|\sum_{j \leq[l / B]} \Sigma_{T}^{ N*}\Big(I_{2j}\Big)\Big| \geq \varrho \bigg)
\leq \exp\bigg({-\frac{\varrho^{2}}{2\mu B^{2} N^{2}}}\bigg).
\end{eqnarray}
Similarly, we also obtain
\begin{equation}\label{equation.19}
\mathbb{P}\bigg(\sup_{l \leq T}\Big|\sum_{j \leq[l / B]} \Sigma_{T}^{ N*}\Big(I_{2j-1}\Big)\Big| \geq \varrho \bigg) \leq  \exp(-\frac{\varrho^{2}}{2\mu B^{2} N^{2}}).
\end{equation}
From (\ref{equation.17})-(\ref{equation.19}), we have
\begin{equation*}
\mathbb{P}\bigg(\sup_{l \leq T}\Big|\Sigma_{l}\Big| \geq 6\varrho \bigg) \leq\frac{2T}{\varrho}\frac{c_{K}}{\eta_{2}-1}N^{1-\eta_{2}}L(N)+\frac{2\mu B\tau(B)}{\varrho}
+2\exp\bigg({-\frac{\varrho^{2}}{2\mu B^{2} N^{2}}}\bigg).
\end{equation*}
As $2B\mu \leq T \leq 3B\mu$ and the the process $ \{Z_{t,T}\} $ is exponentially $ \tau $-mixing \citep{Chwialkowski2014}, i.e., for a  for some $\varphi, \eta_{1}>1$, $\eta_{2}\geq 2$, $\tau(B) \leq e^{-\varphi B^{\eta_{1}}} $. Then we have
\begin{align}\nonumber
\mathbb{P}\bigg(\sup_{l \leq T}\Big|&\Sigma_{l}\Big| \geq \varrho \bigg) \\\nonumber
&\leq \frac{12T}{\varrho}\frac{1}{\eta_{2}-1}N^{1-\eta_{2}}L(N)+\frac{12\mu B \exp(-\varphi B^{\eta_{1}})}{\varrho}
+2\exp\bigg({-\frac{\varrho^{2}}{72\mu B^{2} N^{2}}}\bigg)\\
&\leq \frac{12T}{\varrho}N^{1-\eta_{2}}L(N)+ \frac{6T \exp(-\varphi B^{\eta_{1}})}{\varrho}
+2\exp\bigg({-\frac{\varrho^{2}}{36TBN^{2}}}\bigg).
\end{align}
For $\varrho>1$, we now choosing $N=\frac{\varrho^{d_{1}}}{2 \varrho^{\frac{1}{\eta_{1}}}}=\frac{\varrho^{d_{1}-1/\eta_{1}}}{2}$ and $B=\varrho^{\frac{1}{\eta_{1}}},$
with $d_{1} \in (0,1)$, we have $2BN \leq \varrho $, and
\begin{align*}
\mathbb{P}\bigg(\sup_{l \leq T}&\Big|\Sigma_{l}\Big| \geq \varrho \bigg)\\
& \leq \frac{12T}{\varrho} (\frac{\varrho^{d_{1}}}{2 \varrho^{\frac{1}{\eta_{1}}}})^{1-\eta_{2}}L(\frac{\varrho^{d_{1}}}{2 \varrho^{\frac{1}{\eta_{1}}}})+ \frac{6T \exp(-\varphi \varrho)}{\varrho} +2\exp\big({-\frac{\varrho^{2}}{36T\varrho^{\frac{1}{\eta_{1}}}(\frac{\varrho^{d_{1}}}{2 \varrho^{\frac{1}{\eta_{1}}}})^{2}}}\big)\\
& \leq \frac{12T \varrho^{(d_{1}-1/\eta_{1})(1-\eta_{2})-1}}{ 2^{1-\eta_{2}}}L(\frac{\varrho^{d_{1}-1/\eta_{1}}}{2})
+ \frac{6T \exp(-\varphi \varrho)}{\varrho} +2\exp\big({-\frac{1}{9T\varrho^{2d_{1}-1/\eta_{1}-2}}}\big).
\end{align*}
Then for $\varrho>1/T^{\vartheta} >1/T$, $0<\vartheta<1$, we have
\begin{align*}
\mathbb{P}\bigg(\frac{1}{T}\Big|\sum_{t=1}^{T} Z_{t, T} \Big| \geq \varrho \bigg) 
&\leq \frac{12T^{(d_{1}-1/\eta_{1})(1-\eta_{2})} \varrho^{(d_{1}-1/\eta_{1})(1-\eta_{2})-1}}{ 2^{1-\eta_{2}}}L(\frac{(\varrho T)^{d_{1}-1/\eta_{1}}}{2})\\
&\quad + \frac{6 \exp(-\varphi \varrho T)}{\varrho} +2\exp\big({-\frac{1}{9T^{2d_{1}-1/\eta_{1}-1}\varrho^{2d_{1}-1/\eta_{1}-2}}}\big).
\end{align*}
Note that we want each term above tends to $0$ when $T$ is very large, for the term
\begin{align*}
&\frac{12T^{(d_{1}-1/\eta_{1})(1-\eta_{2})} \varrho^{(d_{1}-1/\eta_{1})(1-\eta_{2})-1}}{ 2^{1-\eta_{2}}}L(\frac{(\varrho T)^{d_{1}-1/\eta_{1}}}{2})\\
&\simeq T^{(d_{1}-1/\eta_{1})(1-\eta_{2})}T^{(1-(d_{1}-1/\eta_{1})(1-\eta_{2}))\vartheta}\\
&\simeq T^{(1-\vartheta)(d_{1}-1/\eta_{1})(1-\eta_{2})+\vartheta},
\end{align*}
it means that we want $(1-\vartheta)(d_{1}-1/\eta_{1})(1-\eta_{2})+\vartheta<0$, i.e., $d_{1}>\frac{\vartheta}{(1-\vartheta)(\eta_{2}-1)}+\frac{1}{\eta_{1}}$.
Obviously, the term
\begin{align*}
\frac{6 \exp(-\varphi \varrho T)}{\varrho} \simeq \frac{T^{\vartheta}}{\exp(T^{1-\vartheta})},
\end{align*}
tends to $0$ when $T$ is very large. For the term
\begin{align*}
2\exp\big({-\frac{1}{9T^{2d_{1}-1/\eta_{1}-1}\varrho^{2d_{1}-1/\eta_{1}-2}}}\big)
\simeq 2\exp(-T^{1-2d_{1}+1/\eta_{1}-(2-2d_{1}+1/\eta_{1})\vartheta}),
\end{align*}
it means that $1-2d_{1}+1/\eta_{1}-(2-2d_{1}+1/\eta_{1})\vartheta>0$, i.e., $d_{1}<\frac{1-2\vartheta}{2(1-\vartheta)}+\frac{1}{2\eta_{1}}$. 
Then we have $\frac{\vartheta}{(1-\vartheta)(\eta_{2}-1)}+\frac{1}{\eta_{1}} < d_{1}< \frac{1-2\vartheta}{2(1-\vartheta)}+\frac{1}{2\eta_{1}}$ with $0<\vartheta<\frac{(\eta_{1}-1)(\eta_{2}-1)}{1+(2\eta_{1}-1)\eta_{2}}$, we get
\begin{align*}
\mathbb{P}\bigg(\frac{1}{T}\Big|\sum_{t=1}^{T} Z_{t, T} \Big| \geq \varrho \bigg) 
&\leq \frac{12T^{(d_{1}-1/\eta_{1})(1-\eta_{2})} \varrho^{(d_{1}-1/\eta_{1})(1-\eta_{2})-1}}{ 2^{1-\eta_{2}}}L(\frac{(\varrho T)^{d_{1}-1/\eta_{1}}}{2})\\
&+ \frac{6 \exp(-\varphi \varrho T)}{\varrho} +2\exp\big({-\frac{1}{9T^{2d_{1}-1/\eta_{1}-1}\varrho^{2d_{1}-1/\eta_{1}-2}}}\big).
\end{align*}

\subsection{Proof of Proposition~\ref{proposition.Locally-Stationary-regularly-varying-heavy-tailed}}
Let $ \Phi\Big(\frac{r}{T}, X_{r,T}^{j} \Big)=\frac{1}{T}\sum_{t=1}^{T} W_{t,r,T}^{j}=\frac{1}{T}\sum_{t=1}^{T}K_{h,1}\big(\frac{t}{T}-\frac{r}{T}\big) K_{h,2}(X_{t,T}^{j}-X_{r, T}^{j}) \varepsilon_{t,T}, \text{ for } t=1,\ldots,T$. Set $\tau_{T}=T \log T $, we write
\begin{equation*}
\Phi \Big(\frac{r}{T}, X_{r,T}^{j} \Big) = \Phi_{1} \Big(\frac{r}{T}, X_{r,T}^{j} \Big) + \Phi_{2} \Big(\frac{r}{T}, X_{r,T}^{j} \Big),  
\end{equation*} 
where 
\begin{align*}
\Phi_{1} \Big(\frac{r}{T}, X_{r,T}^{j} \Big)=\frac{1}{T} \sum_{t=1}^{T} K_{h,1}\big(\frac{t}{T}-\frac{r}{T}\big) K_{h,2}(X_{t,T}^{j}-X_{r, T}^{j}) \varepsilon_{t, T} \mathbb{I}\Big(\Big|\varepsilon_{t, T}\Big| \leq \tau_{T}\Big), \\ 
\Phi_{2} \Big(\frac{r}{T}, X_{r,T}^{j} \Big)=\frac{1}{T} \sum_{t=1}^{T} K_{h,1}\big(\frac{t}{T}-\frac{r}{T}\big) K_{h,2}(X_{t,T}^{j}-X_{r, T}^{j}) \varepsilon_{t, T} \mathbb{I}\Big(\Big|\varepsilon_{t, T}\Big|>\tau_{T}\Big) .  
\end{align*}
It follows that 
\begin{equation}\label{equationrvh22}
\mathbb P (\big|\Phi\Big(\frac{r}{T}, X_{r,T}^{j} \Big)\big| > 2\gamma) \leq\mathbb  P (\big|\Phi_{1}\Big(\frac{r}{T}, X_{r,T}^{j} \Big)\big| > \gamma)+ \mathbb P (\big|\Phi_{2}\Big(\frac{r}{T}, X_{r,T}^{j} \Big)\big| > \gamma). 
\end{equation}
For $ \Phi_{2} \Big(\frac{r}{T}, X_{r,T}^{j} \Big) $, defining $ b_{T}=\sqrt{\log T / T }$, then for any $\gamma \geq C_{K} \sqrt{\log T / T }$, where $C_{K}=C_{K_{1}}C_{K_{2}}$, it has that 
\begin{equation}\label{equationrvh23}
\begin{aligned}
\mathbb P\Big(\Big|\Phi_{2} &\Big(\frac{r}{T}, X_{r,T}^{j} \Big) \Big| \geq \gamma \Big) \\
& \leq\mathbb  P\Big( \big|\Phi_{2} \Big(\frac{r}{T}, X_{r,T}^{j} \Big) \big| \geq C_{K} b_{T} \Big)\\
& =\mathbb  P\Big(\Big|\frac{1}{T} \sum_{t=1}^{T} K_{h,1}\big(\frac{t}{T}-\frac{r}{T}\big) K_{h,2}(X_{t,T}^{j}-X_{r, T}^{j}) \varepsilon_{t, T} \mathbb{I}\Big(\Big|\varepsilon_{t, T}\Big|>\tau_{T}\Big)\Big| \geq C_{K} b_{T} \Big)\\
& \leq\mathbb  P\big(\Big|\frac{1}{T} T C_{K} \varepsilon_{t, T} \mathbb{I}\Big(\Big|\varepsilon_{t, T}\Big|>\tau_{T}\Big)\Big| \geq C_{K} b_{T} \big) \\
& \leq \mathbb P\big(\big|\varepsilon_{t T}\big|>\tau_{T}, \text{ for some } 1 \leq t \leq T \Big) \\
& \leq \tau_{T}^{-\eta_{2}}L(\tau_{T})\\
& \leq (T\log T)^{-\eta_{2}}L((T\log T)).
\end{aligned}
\end{equation}
We now turn to the analysis of $ \Phi_{1} \Big(\frac{r}{T}, X_{r,T}^{j} \Big) $. From Assumption~\ref{Ass:1}, $\{X_{t,T}\}_{t=1}^{T}$ is a locally stationary sequence, which can be approximated locally by a strictly stationary sequence $\{X_{t}(u )\}_{t \in \mathbb{Z}}$. Since $K_{1}$ and $K_{2}$ are Lipschitz and bounded, with Remark~\ref{remark.U_{t,T}}, i.e., $\norm{X_{t, T}-X_{t}(u)} \leq \big(\big|\frac{t}{T}-u\big|+\frac{1}{T}\big) U_{t, T}(u) \leq C_{U} \big(\big|\frac{t}{T}-u\big|+\frac{1}{T}\big)$, where $u \in [0,1]$ and $C_{U}$ is a constant, we can infer that
\begin{align*}
& \big|K_{h,1}\big(\frac{t}{T}-\frac{r}{T}\big) K_{h,2}(X_{t,T}^{j}-X_{r, T}^{j})-K_{h,1}\big(u-\frac{r}{T}\big) K_{h,2}(X_{t}^{j}(u)-X_{r, T}^{j})\big|\\
&= \big|K_{h,1}\big(\frac{t}{T}-\frac{r}{T}\big) K_{h,2}(X_{t,T}^{j}-X_{r, T}^{j})-K_{h,1}\big(\frac{t}{T}-\frac{r}{T}\big) K_{h,2}(X_{t}^{j}(u)-X_{r, T}^{j}) \\
&\quad+ K_{h,1}\big(\frac{t}{T}-\frac{r}{T}\big) K_{h,2}(X_{t}^{j}(u)-X_{r, T}^{j})-K_{h,1}\big(u-\frac{r}{T}\big) K_{h,2}(X_{t}^{j}(u)-X_{r, T}^{j})\big| \\
&\leq K_{h,1}\big(\frac{t}{T}-\frac{r}{T}\big) \big|K_{h,2}(X_{t,T}^{j}-X_{r, T}^{j})- K_{h,2}(X_{t}^{j}(u)-X_{r, T}^{j})\big|\\
&\quad+ \big|K_{h,1}\big(\frac{t}{T}-\frac{r}{T}\big)-K_{h,1}\big(u-\frac{r}{T}\big)\big| K_{h,2}(X_{t}^{j}(u)-X_{r, T}^{j}) \\
&\leq K_{h,1}\big(\frac{t}{T}-\frac{r}{T}\big) \frac{L_{K_{2}}}{h}|X_{t,T}^{j}-X_{t}^{j}(u)|]+\frac{L_{K_{1}}}{h}|\frac{t}{T}-u| K_{h,2}(X_{t}^{j}(u)-X_{r, T}^{j}) \\
&\leq \frac{C_{K_{1}}L_{K_{2}}}{h} \big(\big|\frac{t}{T}-u\big|+\frac{1}{T}\big) U_{t, T}(u) + \frac{C_{K_{2}}L_{K_{1}}}{h}|\frac{t}{T}-u|\\
& \leq \frac{C_{U}C_{K_{1}}L_{K_{2}}}{h} \big(\big|\frac{t}{T}-u\big|+\frac{1}{T}\big) + \frac{C_{K_{2}}L_{K_{1}}}{h}|\frac{t}{T}-u|\\
&\leq \frac{C_{K,L}}{h} \big(1+\frac{1}{T}\big)+\frac{C_{K,L}}{h}\\
& \leq \frac{C_{K,L}(2T+1)}{Th},
\end{align*}
where $C_{K,L}=\max\{C_{U}C_{K_{1}}L_{K_{2}},C_{K_{2}}L_{K_{1}}\}$.
Defining
\begin{equation*}
\tilde{\Phi}_{1}(\frac{r}{T}, X_{r, T}^{j})=\frac{1}{T } \sum_{t=1}^{T} K_{h,1}\big(u-\frac{r}{T}\big) K_{h,2}(X_{t}^{j}(u)-X_{r, T}^{j})\big) \varepsilon_{t,T} \ind{\Big(\big|\varepsilon_{t,T}\big| \leq \tau_{T}\Big)}, 
\end{equation*}
we have 
\begin{align*}
&\Phi_{1} \big(\frac{r}{T}, X_{r, T}^{j} \big)-\tilde{\Phi}_{1} \big(\frac{r}{T}, X_{r, T}^{j}\big)\\
& \leq \frac{1}{T } \sum_{t=1}^{T} \Big(K_{h,1}\big(\frac{t}{T}-\frac{r}{T}\big) K_{h,2}(X_{t,T}^{j}-X_{r, T}^{j})-K_{h,1}\big(u-\frac{r}{T}\big) K_{h,2}(X_{t}^{j}(u)-X_{r, T}^{j})\Big)\varepsilon_{t,T} \ind{\Big(\big|\varepsilon_{t,T}\big| \leq \tau_{T}\Big)}\\
&\leq \frac{1}{T } \sum_{t=1}^{T} \frac{C_{K,L}(2T+1)}{Th} \varepsilon_{t,T} \ind{\Big(\big|\varepsilon_{t,T}\big| \leq \tau_{T}\Big)},
\end{align*}
We can have that
\begin{align*}
\Phi_{1} \big(\frac{r}{T}, X_{r, T}^{j} \big)&= \Phi_{1} \big(\frac{r}{T}, X_{r, T}^{j} \big)-\tilde{\Phi}_{1} \big(\frac{r}{T}, X_{r, T}^{j}\big) +\tilde{\Phi}_{1} \big(\frac{r}{T}, X_{r, T}^{j}\big) \\
&\leq \frac{1}{T } \sum_{t=1}^{T} \frac{C_{K,L}(2T+1)}{Th} \varepsilon_{t,T} \ind{\big(\big|\varepsilon_{t,T}\big| \leq \tau_{T}\big)} +\tilde{\Phi}_{1} \big(\frac{r}{T}, X_{r, T}^{j}\big).
\end{align*} 
It follows that
\begin{equation}\label{equationbeta13bis}
\mathbb P \Big( \Big|\Phi_{1} \Big(\frac{t}{T}, X_{t,T}^{j} \Big)\Big| \geq \gamma \Big)
\leq Q_{T}+\tilde{Q}_{T},
\end{equation}
where 
\begin{equation*}
Q_{T}=\mathbb P \big( |\frac{1}{T } \sum_{t=1}^{T} \frac{C_{K,L}(2T+1)}{Th} \varepsilon_{t,T} \ind{\big(\big|\varepsilon_{t,T}\big| \leq \tau_{T}\big)} |\geq \frac{\gamma}{{2}} \big),
\end{equation*}
and
\begin{equation*}
\tilde{Q}_{T}=\mathbb P \big( |\tilde{\Phi}_{1} \big(\frac{r}{T}, X_{r, T}^{j}\big)| \geq \frac{\gamma}{2} \big).
\end{equation*}
To bound $ \tilde{Q}_{T} $, we write
\begin{equation*}
\tilde{Q}_{T}=\mathbb P \big(|\tilde{\Phi}_{1} \big(\frac{r}{T}, X_{r, T}^{j}\big)| \geq \frac{\gamma}{2} \big) \leq \mathbb P\Big(\frac{1}{T}\Big|\sum_{t=1}^{T} Z_{t, T}(u,  X_{t}^{j}(u))\Big|\geq \frac{\gamma}{2}\Big),  
\end{equation*}
with
\begin{equation*}
Z_{t, T}(\frac{r}{T}, X_{r, T}^{j})= K_{h,1}\big(u-\frac{r}{T}\big) K_{h,2}(X_{t}^{j}(u)-X_{r, T}^{j}) \varepsilon_{t,T} \ind{\Big(\big|\varepsilon_{t,T}\big| \leq \tau_{T}\Big)}.  
\end{equation*}
Note that $K_{1}$ and $K_{2}$ are bounded, from Assumption~\ref{Ass:2} and Lemma~\ref{lemma_sub_beta_mixing}, we have $\{ \varepsilon_{t,T} \} $ is $\beta$-mixing sequence, i.e. $\{ \varepsilon_{t,T} \} $ follows the $\beta$-mixing sub-Weibull distribution with mixing coefficients $ \beta(k) \leq \exp \bigg(-\varphi k^{\eta_{1}}\bigg) $, for some $ \varphi, \eta_{1}>1$. We now bound $ \tilde{Q}_{T}$ with the help of Proposition~\ref{proposition.Regularly varying heavy-tailed}, for $\gamma>  
2C_{K}/T^{\vartheta}$ with $0<\vartheta<\frac{(\eta_{1}-1)(\eta_{2}-1)}{1+(2\eta_{1}-1)\eta_{2}}$, then we have
\begin{align}\label{equationbeta25}
\tilde{Q}_{T}&\leq \mathbb P\Big(\frac{1}{T}\Big|\sum_{t=1}^{T} Z_{t, T}(\frac{r}{T}, X_{r, T}^{j})\Big|\geq \frac{\gamma}{2}\Big)\\ \nonumber
& \leq \mathbb P\Big(\frac{C_{K}}{T}\Big|\sum_{t=1}^{T}\varepsilon_{t,T} \ind{\Big(\big|\varepsilon_{t,T}\big| \leq \tau_{T}\Big)}\Big|\geq \frac{\gamma}{2}\Big)\\ \nonumber
&\leq \frac{12T^{(d_{1}-1/\eta_{1})(1-\eta_{2})} (\gamma/2C_{K})^{(d_{1}-1/\eta_{1})(1-\eta_{2})-1}}{ 2^{1-\eta_{2}}}L(\frac{(\gamma T/2C_{K})^{d_{1}-1/\eta_{1}}}{2})\\
&\quad + \frac{6 \exp(-\varphi (\gamma/2C_{K}) T)}{(\gamma/2C_{K})} +2\exp\big({-\frac{1}{9T^{2d_{1}-1/\eta_{1}-1}(\gamma/2C_{K})^{2d_{1}-1/\eta_{1}-2}}}\big),
\end{align}
where $\frac{\vartheta}{(1-\vartheta)(\eta_{2}-1)}+\frac{1}{\eta_{1}} < d_{1}< \frac{1-2\vartheta}{2(1-\vartheta)}+\frac{1}{2\eta_{1}}$, $C_{K}=C_{K_{1}}C_{K_{2}}$ and $\eta_{2}\geq 2$.

We now bound $ Q_{T}$ with the help of Proposition~\ref{proposition.Regularly varying heavy-tailed}, for $\gamma> \frac{2C_{K,L}(2T+1)}{T^{1+\vartheta}h}$ with $0<\vartheta<\frac{(\eta_{1}-1)(\eta_{2}-1)}{1+(2\eta_{1}-1)\eta_{2}}$, then we have
\begin{align}\label{equationrvh26}
Q_{T}&= \mathbb P \bigg(\frac{1}{T} |\sum_{t=1}^{T} \frac{C_{K,L}(2T+1)}{Th} \varepsilon_{t,T} \ind{\Big(\big|\varepsilon_{t,T}\big| \leq \tau_{T}\Big)}|\geq \frac{\gamma}{2}\bigg)\\ \nonumber
& \leq \frac{12T^{(d_{1}-1/\eta_{1})(1-\eta_{2})} (\frac{\gamma Th}{2 C_{K,L}(2T+1)})^{(d_{1}-1/\eta_{1})(1-\eta_{2})-1}}{ 2^{1-\eta_{2}}}L(\frac{(\frac{\gamma T^{2}h}{2 C_{K,L}(2T+1)})^{d_{1}-1/\eta_{1}}}{2})\\
&\quad + \frac{6 \exp(-\varphi (\frac{\gamma T^{2}h}{2 C_{K,L}(2T+1)}))}{(\frac{\gamma Th}{2 C_{K,L}(2T+1)})} +2\exp\big({-\frac{1}{9T^{2d_{1}-1/\eta_{1}-1}(\frac{\gamma Th}{2 C_{K,L}(2T+1)})^{2d_{1}-1/\eta_{1}-2}}}\big).\\ \nonumber
& \leq \frac{12T^{2(d_{1}-1/\eta_{1})(1-\eta_{2})-1} (\gamma h)^{(d_{1}-1/\eta_{1})(1-\eta_{2})-1}}{ 2^{1-\eta_{2}}(2C_{K,L}(2T+1))^{(d_{1}-1/\eta_{1})(1-\eta_{2})-1}}L(\frac{(\frac{\gamma T^{2}h}{2 C_{K,L}(2T+1)})^{d_{1}-1/\eta_{1}}}{2})\\
&\quad + \frac{12 C_{K,L}(2T+1)\exp(-\varphi (\frac{\gamma T^{2}h}{2 C_{K,L}(2T+1)}))}{\gamma Th}+2\exp\big({-\frac{(2C_{K,L}(2T+1))^{2d_{1}-1/\eta_{1}-2}}{9T^{4d_{1}-2/\eta_{1}-3}(\gamma h)^{2d_{1}-1/\eta_{1}-2}}}\big),
\end{align}
where $\frac{\vartheta}{(1-\vartheta)(\eta_{2}-1)}+\frac{1}{\eta_{1}} < d_{1}< \frac{1-2\vartheta}{2(1-\vartheta)}+\frac{1}{2\eta_{1}}$ and $\eta_{2}\geq 2$. 

From (\ref{equationrvh22})-(\ref{equationrvh26}), for $\gamma>\max\{C_{K}\sqrt{\log T / T },2C_{K}/T^{\vartheta},\frac{2C_{K,L}(2T+1)}{T^{1+\vartheta}h}\} =\frac{2C_{K,L}(2T+1)}{T^{1+\vartheta}h}$ (since $0<\vartheta<\frac{\eta_{1-1}}{4\eta_{1}-1}<\frac{1}{4}$), we further get that
\begin{align*}
\mathbb{P} \Big(\big|\Phi\Big(&\frac{r}{T}, X_{r,T}^{j} \Big)\big| > 2\gamma\Big) \\
&\leq (T\log T)^{-\eta_{2}}L((T\log T))\\
& \quad +\frac{12T^{(d_{1}-1/\eta_{1})(1-\eta_{2})} (\gamma/2C_{K})^{(d_{1}-1/\eta_{1})(1-\eta_{2})-1}}{ 2^{1-\eta_{2}}}L(\frac{(\gamma T/2C_{K})^{d_{1}-1/\eta_{1}}}{2})\\
&\quad + \frac{6 \exp(-\varphi (\gamma/2C_{K}) T)}{(\gamma/2C_{K})} +2\exp\big({-\frac{1}{9T^{2d_{1}-1/\eta_{1}-1}(\gamma/2C_{K})^{2d_{1}-1/\eta_{1}-2}}}\big)\\
& \quad + \frac{12T^{2(d_{1}-1/\eta_{1})(1-\eta_{2})-1} (\gamma h)^{(d_{1}-1/\eta_{1})(1-\eta_{2})-1}}{ 2^{1-\eta_{2}}(2C_{K,L}(2T+1))^{(d_{1}-1/\eta_{1})(1-\eta_{2})-1}}L(\frac{(\frac{\gamma T^{2}h}{2 C_{K,L}(2T+1)})^{d_{1}-1/\eta_{1}}}{2})\\
&\quad+ \frac{12 C_{K,L}(2T+1)\exp(-\varphi (\frac{\gamma T^{2}h}{2 C_{K,L}(2T+1)}))}{\gamma Th}+2\exp\big({-\frac{(2C_{K,L}(2T+1))^{2d_{1}-1/\eta_{1}-2}}{9T^{4d_{1}-2/\eta_{1}-3}(\gamma h)^{2d_{1}-1/\eta_{1}-2}}}\big).
\end{align*} 
Then, for $\gamma> \frac{2C_{K,L}(2T+1)}{T^{1+\vartheta}h}$, we have
\begin{align*}
\mathbb{P} \Big(\big|\Phi\Big(&\frac{r}{T}, X_{r,T}^{j} \Big)\big| > \gamma\Big) \\
&\leq (T\log T)^{-\eta_{2}}L(T\log T)\\
& \quad +\frac{12T^{(d_{1}-1/\eta_{1})(1-\eta_{2})} (\gamma/4C_{K})^{(d_{1}-1/\eta_{1})(1-\eta_{2})-1}}{ 2^{1-\eta_{2}}}L(\frac{(\gamma T/4C_{K})^{d_{1}-1/\eta_{1}}}{2})\\
&\quad + \frac{6 \exp(-\varphi (\gamma/4C_{K}) T)}{(\gamma/4C_{K})} +2\exp\big({-\frac{1}{9T^{2d_{1}-1/\eta_{1}-1}(\gamma/4C_{K})^{2d_{1}-1/\eta_{1}-2}}}\big)\\
& \quad + \frac{12T^{2(d_{1}-1/\eta_{1})(1-\eta_{2})-1} (\gamma h)^{(d_{1}-1/\eta_{1})(1-\eta_{2})-1}}{ 2^{1-\eta_{2}}(4C_{K,L}(2T+1))^{(d_{1}-1/\eta_{1})(1-\eta_{2})-1}}L(\frac{(\frac{\gamma T^{2}h}{4C_{K,L}(2T+1)})^{d_{1}-1/\eta_{1}}}{2})\\
&\quad+ \frac{24C_{K,L}(2T+1)\exp(-\varphi (\frac{\gamma T^{2}h}{4 C_{K,L}(2T+1)}))}{\gamma Th}+2\exp\big({-\frac{(4C_{K,L}(2T+1))^{2d_{1}-1/\eta_{1}-2}}{9T^{4d_{1}-2/\eta_{1}-3}(\gamma h)^{2d_{1}-1/\eta_{1}-2}}}\big).
\end{align*} 
Hence, we obtain
\begin{align*}
\mathbb{P} \Big(\big|\Phi\Big(&\frac{r}{T}, X_{r,T}^{j} \Big)\big| > \gamma\Big) \\
&\leq (T\log T)^{-\eta_{2}}L(T\log T)\\
& \quad +\frac{12T^{(d_{1}-1/\eta_{1})(1-\eta_{2})} \gamma^{(d_{1}-1/\eta_{1})(1-\eta_{2})-1}}{ 2^{1-\eta_{2}}(4C_{K})^{(d_{1}-1/\eta_{1})(1-\eta_{2})-1}}L(\frac{(\gamma T/4C_{K})^{d_{1}-1/\eta_{1}}}{2})\\
&\quad + \frac{24C_{K} \exp(-\varphi \gamma T/4C_{K})}{\gamma } +2\exp\big({-\frac{1}{9T^{2d_{1}-1/\eta_{1}-1}(\gamma/4C_{K})^{2d_{1}-1/\eta_{1}-2}}}\big)\\
& \quad + \frac{12T^{2(d_{1}-1/\eta_{1})(1-\eta_{2})-1} (\gamma h)^{(d_{1}-1/\eta_{1})(1-\eta_{2})-1}}{ 2^{1-\eta_{2}}(4C_{K,L}(2T+1))^{(d_{1}-1/\eta_{1})(1-\eta_{2})-1}}L(\frac{( \gamma T^{2}h)^{d_{1}-1/\eta_{1}}}{2(4C_{K,L}(2T+1))^{d_{1}-1/\eta_{1}}})\\
&\quad+ \frac{24C_{K,L}(2T+1)\exp(-\varphi (\frac{\gamma T^{2}h}{4 C_{K,L}(2T+1)}))}{\gamma Th}+2\exp\big({-\frac{(4C_{K,L}(2T+1))^{2d_{1}-1/\eta_{1}-2}}{9T^{4d_{1}-2/\eta_{1}-3}(\gamma h)^{2d_{1}-1/\eta_{1}-2}}}\big).
\end{align*} 
Note that we want each term above tends to $0$ when $T$ is very large, let $h=\bigO(T^{-\xi})$, $0<\xi<1$, then we have $\gamma=\bigO(T^{\xi-\vartheta})$ and $\gamma h=\bigO(T^{-\vartheta})$. 

For the term
\begin{align*}
&\frac{12T^{(d_{1}-1/\eta_{1})(1-\eta_{2})} \gamma^{(d_{1}-1/\eta_{1})(1-\eta_{2})-1}}{ 2^{1-\eta_{2}}(4C_{K})^{(d_{1}-1/\eta_{1})(1-\eta_{2})-1}}L(\frac{(\gamma T/4C_{K})^{d_{1}-1/\eta_{1}}}{2})\\
& \simeq T^{(d_{1}-1/\eta_{1})(1-\eta_{2})} T^{(\xi-\vartheta)((d_{1}-1/\eta_{1})(1-\eta_{2})-1)}\\
& \simeq T^{(d_{1}-1/\eta_{1})(1-\eta_{2})+(\xi-\vartheta)((d_{1}-1/\eta_{1})(1-\eta_{2})-1)},
\end{align*}
we want $(d_{1}-1/\eta_{1})(1-\eta_{2})+(\xi-\vartheta)((d_{1}-1/\eta_{1})(1-\eta_{2})-1)<0$, i.e., $\xi >\vartheta-\frac{(d_{1}-1/\eta_{1})(\eta_{2}-1)}{1+(d_{1}-1/\eta_{1})(\eta_{2}-1)}$, since $\vartheta-\frac{(d_{1}-1/\eta_{1})(\eta_{2}-1)}{1+(d_{1}-1/\eta_{1})(\eta_{2}-1)}<0$ with $\frac{\vartheta}{(1-\vartheta)(\eta_{2}-1)}+\frac{1}{\eta_{1}} < d_{1}$, obviously, it is satisfied.

For the term
\begin{align*}
2\exp\big({-\frac{1}{9T^{2d_{1}-1/\eta_{1}-1}(\gamma/4C_{K})^{2d_{1}-1/\eta_{1}-2}}}\big) 
\simeq 2\exp(-T^{1-2d_{1}+1/\eta_{1}+(\xi-\vartheta)(2-2d_{1}+1/\eta_{1})}),
\end{align*}
we want $1-2d_{1}+1/\eta_{1}+(\xi-\vartheta)(2-2d_{1}+1/\eta_{1})>0$, i.e., $\xi>\vartheta-\frac{1-2d_{1}+1/\eta_{1}}{2-2d_{1}+1/\eta_{1}}$, since $\vartheta-\frac{1-2d_{1}+1/\eta_{1}}{2-2d_{1}+1/\eta_{1}}<0$ with $d_{1}< \frac{1-2\vartheta}{2(1-\vartheta)}+\frac{1}{2\eta_{1}}$, obviously, it is satisfied. 

For the term 
\begin{align*}
& \frac{12T^{2(d_{1}-1/\eta_{1})(1-\eta_{2})-1} (\gamma h)^{(d_{1}-1/\eta_{1})(1-\eta_{2})-1}}{ 2^{1-\eta_{2}}(4C_{K,L}(2T+1))^{(d_{1}-1/\eta_{1})(1-\eta_{2})-1}}L(\frac{( \gamma T^{2}h)^{d_{1}-1/\eta_{1}}}{2(4C_{K,L}(2T+1))^{d_{1}-1/\eta_{1}}})\\
& \simeq \frac{T^{2(d_{1}-1/\eta_{1})(1-\eta_{2})-1} T^{-\vartheta((d_{1}-1/\eta_{1})(1-\eta_{2})-1)}}{ T^{(d_{1}-1/\eta_{1})(1-\eta_{2})-1}}\\
& \simeq T^{(d_{1}-1/\eta_{1})(1-\eta_{2})+\vartheta(1-(d_{1}-1/\eta_{1})(1-\eta_{2}))},
\end{align*}
we want $(d_{1}-1/\eta_{1})(1-\eta_{2})+\vartheta(1-(d_{1}-1/\eta_{1})(1-\eta_{2}))<0$, obviously, it is satisfied with $\frac{\vartheta}{(1-\vartheta)(\eta_{2}-1)}+\frac{1}{\eta_{1}} < d_{1}$.

For the term 
\begin{align*}
2\exp\big({-\frac{(4C_{K,L}(2T+1))^{2d_{1}-1/\eta_{1}-2}}{9T^{4d_{1}-2/\eta_{1}-3}(\gamma h)^{2d_{1}-1/\eta_{1}-2}}}\big)
& \simeq 2\exp\big({-\frac{T^{2d_{1}-1/\eta_{1}-2}}{T^{4d_{1}-2/\eta_{1}-3} T^{-\vartheta(2d_{1}-1/\eta_{1}-2)}}}\big)\\
& \simeq 2\exp\big(-T^{1-2d_{1}+1/\eta_{1} +\vartheta(2d_{1}-1/\eta_{1}-2)}),
\end{align*}
we want $1-2d_{1}+1/\eta_{1} +\vartheta(2d_{1}-1/\eta_{1}-2)>0$, obviously, it is satisfied with $d_{1}< \frac{1-2\vartheta}{2(1-\vartheta)}+\frac{1}{2\eta_{1}}$. Obviously, other terms tend to $0$ at large T. 

Then for $\gamma>\frac{2C_{K,L}(2T+1)}{T^{1+\vartheta}h}$, $0<\vartheta<\frac{(\eta_{1}-1)(\eta_{2}-1)}{1+(2\eta_{1}-1)\eta_{2}}$, $\frac{\vartheta}{(1-\vartheta)(\eta_{2}-1)}+\frac{1}{\eta_{1}} < d_{1}< \frac{1-2\vartheta}{2(1-\vartheta)}+\frac{1}{2\eta_{1}}$, let $h=\bigO(T^{-\xi})$ with $0<\xi<1$, we have 
\begin{align*}
\mathbb{P} \Big(\big|\Phi\Big(&\frac{r}{T}, X_{r,T}^{j} \Big)\big| > \gamma\big) \\
&\leq (T\log T)^{-\eta_{2}}L(T\log T)\\
& \quad +\frac{12T^{(d_{1}-1/\eta_{1})(1-\eta_{2})} \gamma^{(d_{1}-1/\eta_{1})(1-\eta_{2})-1}}{ 2^{1-\eta_{2}}(4C_{K})^{(d_{1}-1/\eta_{1})(1-\eta_{2})-1}}L(\frac{(\gamma T/4C_{K})^{d_{1}-1/\eta_{1}}}{2})\\
&\quad + \frac{24C_{K} \exp(-\varphi \gamma T/4C_{K})}{\gamma } +2\exp\big({-\frac{1}{9T^{2d_{1}-1/\eta_{1}-1}(\gamma/4C_{K})^{2d_{1}-1/\eta_{1}-2}}}\big)\\
& \quad + \frac{12T^{2(d_{1}-1/\eta_{1})(1-\eta_{2})-1} (\gamma h)^{(d_{1}-1/\eta_{1})(1-\eta_{2})-1}}{ 2^{1-\eta_{2}}(4C_{K,L}(2T+1))^{(d_{1}-1/\eta_{1})(1-\eta_{2})-1}}L(\frac{( \gamma T^{2}h)^{d_{1}-1/\eta_{1}}}{2(4C_{K,L}(2T+1))^{d_{1}-1/\eta_{1}}})\\
&\quad+ \frac{24C_{K,L}(2T+1)\exp(-\varphi (\frac{\gamma T^{2}h}{4 C_{K,L}(2T+1)}))}{\gamma Th}+2\exp\big({-\frac{(4C_{K,L}(2T+1))^{2d_{1}-1/\eta_{1}-2}}{9T^{4d_{1}-2/\eta_{1}-3}(\gamma h)^{2d_{1}-1/\eta_{1}-2}}}\big).
\end{align*}

\section{Proofs of oracle inequalities} \label{sec:proofs}

\subsection{Proof of Theorem~\ref{theorem:least_sq_oracle_ineq_lasso_pen_subweibull} : oracle inequality for sub-Weibull distribution with Lasso} \label{subsec:proof_of_theorem_theorem:least_sq_oracle_ineq_lasso_pen}
By the minimizing property of $\bm{\theta}$, it follows that
\begin{equation*}
\frac{1}{T}\|\bY-\bK \hat{\bm{\theta}}\|_{2}^{2}+\lambda \norm{\hat{\bm{\theta}}}_1 \leq \frac{1}{T}\|\bY-\bK\bm{\theta}\|_{2}^{2}+\lambda \norm{\bm{\theta}}_1,
\end{equation*}
which, using that $ Y_{t,T}=m^{\star}\big(\frac{t}{T}, X_{t, T}\big)+\varepsilon_{t, T} $, $t=1,\ldots,T$, yields
\begin{equation*}
\frac{1}{T}\norm{\bM^\star+\beps-\bK\hat{\bm{\theta}}}_{2}^{2}+\lambda \norm{\hat{\bm{\theta}}}_1 \leq \frac{1}{T}\norm{\bM^\star+\beps -\bK \bm{\theta}}_{2}^{2}+\lambda \norm{\bm{\theta}}_1,
\end{equation*}
where $\bM^\star = \big(m^{\star}\big(\frac 1T, X_{1, T}\big), \cdots, m^{\star}\big(1, X_{T, T}\big)\big)^\top \in \mathbb{R}^T$ and $\beps = (\varepsilon_{1,T}, \ldots, \varepsilon_{T,T})^\top.$ 
Or, equivalently,
\begin{equation*}
\begin{aligned}
&\frac{1}{T}\norm{\bM^\star-\bK \hat{\bm{\theta}}}_{2}^{2}+\frac{1}{T} \norm{\beps}_{2}^{2}+\frac{2}{T}\langle \bM^\star-\bK \hat{\bm{\theta}},\beps \rangle+\lambda \norm{\hat{\bm{\theta}}}_1 \\
&\leq \frac{1}{T}\norm{\bM^\star-\bK \bm{\theta}}_{2}^{2}+\frac{1}{T} \norm{\beps}_{2}^{2}+\frac{2}{T}\langle \bM^\star-\bK \bm{\theta},\beps\rangle+\lambda \norm{\bm{\theta}}_1,
\end{aligned}
\end{equation*}
we have
\begin{equation*}
\frac{1}{T}\norm{\bM^\star-\bK \hat{\bm{\theta}}}|_{2}^{2}
\leq \frac{1}{T}\norm{\bM^\star-\bK \bm{\theta}}_{2}^{2}+ \frac{2}{T}\langle \bK (\hat{\bm{\theta}}-\bm{\theta}),\beps \rangle+\lambda (\norm{\bm{\theta}}_1-\norm{\hat{\bm{\theta}}}_1).
\end{equation*}
So to bound $ \frac{1}{T} \|\bM^\star-\bK \hat{\bm{\theta}} \|_{2}^{2} $, one must bound $ B_{1}=\frac{1}{T}\norm{\bM^\star -\bK \bm{\theta}}_{2}^{2}$, $B_{2}=\frac{1}{T}\langle \bK(\hat{\bm{\theta}}- \bm{\theta}),\beps \rangle $ and $B_{3}=\lambda(\Omega(\bm{\theta})-\Omega(\hat{\bm{\theta}}))$.
For the $B_{2}=\frac{1}{T}\langle \bK(\hat{\bm{\theta}}- \bm{\theta}),\beps\rangle $, we have 
\begin{eqnarray*}
B_{2}&=&\frac{2}{T}\langle \bK(\hat{\bm{\theta}}- \bm{\theta}),\beps \rangle \\
&=& \Big|\frac{2}{T} \sum_{t=1}^{T} \sum_{r=1}^{T} \sum_{j=1}^{d} K_{h,1}\big(\frac{t}{T}-\frac{r}{T}\big) K_{h,2}(X_{t,T}^{j}-X_{r,T}^{j})(\hat{\theta}_{r, j}-\theta_{r, j}) \varepsilon_{t,T}\Big|\\ 
& \leq& \sum_{r=1}^{T} \sum_{j=1}^{d} \frac{2}{T} \sum_{t=1}^{T} \big|K_{h,1}\big(\frac{t}{T}-\frac{r}{T}\big) K_{h,2}(X_{t,T}^{j}-X_{r,T}^{j}) \varepsilon_{t,T}\big|\big|\hat{\theta}_{r, j}-\theta_{r, j}\big|.
\end{eqnarray*}
Let us consider the event $\mathscr{U}_{T}^{\lambda}=\bigcap_{r=1}^{T} \bigcap_{j=1}^{d}\mathscr{U}_{r,j}^{\lambda},$where
\begin{equation*}
\mathscr{U}_{r,j}^{\lambda} = \Big\{\frac{2}{T}\sum_{t=1}^{T} \big|K_{h,1}\big(\frac{t}{T}-\frac{r}{T}\big) K_{h,2}(X_{t,T}^{j}-X_{r,T}^{j})\varepsilon_{t,T}\big| \leq \lambda \Big\}.
\end{equation*}
Note that on $\mathscr{U}^\lambda_{T}$, one has 
\begin{align*}
B_{2}=\frac{2}{T}\langle \bK(\hat{\bm{\theta}}- \bm{\theta}),\beps \rangle
& \leq \sum_{r=1}^{T} \sum_{j=1}^{d} \lambda |\hat{\theta}_{r, j}-\theta_{r, j}|\\
& \leq \sum_{r=1}^{T} \sum_{j=1}^{d} \lambda|\hat{\theta}_{r, j}| +  \sum_{r=1}^{T} \sum_{j=1}^{d} \lambda|\theta_{r, j}|\\
&= \lambda (\norm{\hat{\bm{\theta}}}_1 + \norm{\bm{\theta}}_1).
\end{align*}
Putting things together, we have
\begin{equation*}
\frac{1}{T}\norm{\bM^\star-\bK \hat{\bm{\theta}}}_{2}^{2}
\leq \frac{1}{T}\norm{\bM^\star-\bK \bm{\theta}}_{2}^{2}+2 \lambda \norm{\bm{\theta}}_1.
\end{equation*}
It means as
\begin{equation*}
R(\hat m, m^\star) \leq \inf_{\bm{\theta} \in\R^{T d}} \big\{R(m_{\theta}, m^\star) + 2 \lambda \norm{\bm{\theta}}_1\big\}.
\end{equation*}
Consider the $\{\varepsilon_{t,T}\}_{t=1}^{T}$ follows $\beta$-mixing sub-Weibull distribution in Example~\ref{example3-(i)}, apply Proposition~\ref{proposition.Locally-stationary-sub-Weibull-distribution}, let $1/\eta=1/\eta_{1}+1/\eta_{2}$ and bandwith $h=\bigO(T^{-\xi}) \leq C_{h}T^{-\xi}$ with $0<\xi<\frac{1}{2}$ and constant $C_{h}>0$, for any $\lambda \geq  2C_{K}\sqrt{\log T / T} $ and $T>4$, we find that the probability of the complementary event ${\mathscr{U}^\lambda_{T}}$ is
\begin{align*}
\mathbb P[({\mathscr{U}_{T}^{\lambda}})^{\complement}]
&\leq \sum_{r=1}^{T} \sum_{j=1}^{d} \mathbb P \big(\frac{2}{T}\sum_{t=1}^{T}  |K_{h,1}(\frac{t}{T}-\frac{r}{T}) K_{h,2}(X_{t,T}^{j}-X_{r,T}^{j})\varepsilon_{t,T} | \leq \lambda \big)\\
&\leq Td \Big[ \exp\big(-(\frac{T\log T}{C_{\varepsilon}})^{\eta_{2}}\big)\\
& \quad\qquad + T \exp \big( -\frac{(\lambda T)^{\eta}}{(8C_{K}C_{\varepsilon})^{\eta}C_{1}} \big) + \exp \big( -\frac{\lambda^{2}T}{(8C_{K}C_{\varepsilon})^{2}C_{2}} \big)\\
& \quad\qquad + T \exp \big( -\frac{(\lambda T^{2}h)^{\eta}}{(8C_{K,L}(2T+1)C_{\varepsilon})^{\eta}C_{1}} \big) + \exp \big( -\frac{(\lambda h)^{2}T^{3}}{(8C_{K,L}(2T+1)C_{\varepsilon})^{2}C_{2}} \big) \Big]\\
&\leq Td \Big[ \exp\big(-(\frac{T\log T}{C_{\varepsilon}})^{\eta_{2}}\big)+ T \exp \big( -\frac{(\lambda Th)^{\eta}}{\max\{8C_{K}C_{\varepsilon},16C_{K,L}C_{\varepsilon}\}^{\eta}C_{1}} \big) \\
& \quad\qquad+ \exp \big( -\frac{(\lambda h)^{2}T}{\max\{8C_{K}C_{\varepsilon},16C_{K,L}C_{\varepsilon}\}^{2}C_{2}} \big) \Big]\\
&\leq Td \Big[\exp\big(-(\frac{T\log T}{C_{\varepsilon}})^{\eta_{2}}\big) + T \exp \big( -\frac{(\lambda Th)^{\eta}}{C_{max}^{\eta}C_{1}} \big) + \exp \big( -\frac{(\lambda h)^{2}T}{C_{max}^{2}C_{2}} \big) \Big]\\
&\leq Td \Big[ T \exp \big( -\frac{(\lambda Th)^{\eta}}{C_{3}} \big) + \exp \big( -\frac{(\lambda h)^{2}T}{C_{4}} \big) \Big]\\
&\leq Td \Big[ T \exp \big( - (C_{h}\lambda T^{1-\xi})^{\eta}/C_{3} \big) + \exp \big( -C_{h}^{2}\lambda^{2} T^{1-2\xi} /C_{4}\big) \Big],
\end{align*}
where $1/\eta=1/\eta_{1}+1/\eta_{2} > 1$, $C_{max}=\max\{8C_{K}C_{\varepsilon},16C_{K,L}C_{\varepsilon}\}$, the constants $C_{1}$, $C_{2}$ depend only on $\eta_{1}$, $\eta_{2}$ and $\varphi$, the constants $C_{3}$ depend only on $C_{K}$, $C_{\varepsilon}$, $C_{K,L}$ and $C_{1}$, the constants $C_{4}$ depend only on $C_{K}$, $C_{\varepsilon}$, $C_{K,L}$ and $C_{2}$.
If we set,
\begin{equation*}
\lambda \geq \max \bigg\{ \frac{(c\log d+\log T^{2} )^{1/\eta}}{T^{1-\xi}} ,\sqrt{\frac{c\log d+\log T }{T^{1-2\xi}}} \bigg\},
\end{equation*}
then the probability above is at most $d\exp(-c \log d )=d^{1-c}$.
Note that the constant $c>1$ can be made arbitrarily large but affects the constants $C_{h}$, $C_{3}$ and $C_{4}$ above. We want
\begin{equation*}
\sqrt{\frac{c\log d+\log T }{T^{1-2\xi}}} \geq \frac{(c\log d+2\log T )^{1/\eta}}{T^{1-\xi}},
\end{equation*}
and
\begin{equation*}
T^{\eta/2}(c\log d+\log T )^{\eta/2} \geq c\log d+2\log T^{2},
\end{equation*}
which is implied by
\begin{align*}
T^{\eta/2}(c\log d+\log T )^{\eta/2} \geq c\log d+\log T \text{ (the first condition)} \\ 
\text{and } T^{\eta/2}(c\log d+\log T )^{\eta/2} \geq \log T \text{ (the second condition)},
\end{align*}
the second condition is met for any $T>4$. 
For the first condition, we have
\begin{align*}
T^{\frac{\eta}{2-\eta}} \geq c\log d+\log T,
\end{align*}
if we set $ \frac{1}{2} \leq \eta <1$, $T^{\frac{\eta}{2-\eta}}$ is  significantly larger than $\log T$, then if $T \geq (c\log d)^{\frac{2-\eta}{\eta}}$, we can get $\lambda \geq \sqrt{\frac{c\log d +\log T}{T^{1-2\xi}}}$, obviously, the condition $\lambda \geq 2C_{K}\sqrt{\log T / T } $ is satisfied. Then we have
\begin{equation*}
\frac{1}{T}\norm{\bM^\star-\bK \hat{\bm{\theta}}}_{2}^{2}
\leq \frac{1}{T}\norm{\bM^\star-\bK \bm{\theta}}_{2}^{2}+2 \lambda \norm{\bm{\theta}}_1.
\end{equation*}
It means as
\begin{equation*}
R(\hat m, m^\star) \leq \inf_{\bm{\theta} \in\R^{T d}} \big\{R(m_{\theta}, m^\star) + 2 \lambda \norm{\bm{\theta}}_1\big\}.
\end{equation*}

\subsection{Proof of Theorem~\ref{theorem:least_sq_oracle_ineq_TV_pen_subweibull}: oracle inequality for sub-Weibull distribution with weighted total variation penalization} \label{subec:proof_of_thorem_theorem:least_sq_oracle_ineq_tv_pen}
We consider the following penalized optimization problem
\begin{equation}
\hat{\bm{\theta}} \in \mathop{\arg\min}\limits_{\bm{\theta} \in \mathbb{R}^{T \times d}} \big\{\frac{1}{T} \|\bY-\bK\bm{\theta}\|_{2}^{2}+ \norm{\bm{\theta}}_{\TV,\lambda}\big\}\label{TVestimator}.
\end{equation}
Similar as Theorem \ref{theorem:least_sq_oracle_ineq_lasso_pen_subweibull}, by the minimizing property of $\bm{\theta}$, it follows that
\begin{equation*}
\frac{1}{T}\norm{\bM^\star+\beps- \bK\hat{\bm{\theta}}}_{2}^{2}+ \norm{\hat{\bm{\theta}}}_{\TV,\lambda} \leq \frac{1}{T}\norm{\bM^\star+ \beps-K \bm{\theta}}_{2}^{2}+ \norm{\bm{\theta}}_{\TV,\lambda}.
\end{equation*}
Equivalently, we have 
\begin{eqnarray*}
&&\frac{1}{T}\norm{\bM^\star-\bK \hat{\bm{\theta}}}_{2}^{2}+\frac{1}{T} \norm{\beps}_{2}^{2}+\frac{2}{T}\langle \bM^\star-\bK \hat{\bm{\theta}},\beps \rangle+ \norm{\hat{\bm{\theta}}}_{\TV,\lambda} \\
&&\leq \frac{1}{T}\norm{\bM^\star-\bK \bm{\theta}}_{2}^{2}+\frac{1}{T} \norm{\beps}_{2}^{2}+\frac{2}{T}\langle \bM^\star-\bK \bm{\theta},\beps \rangle+ \norm{\bm{\theta}}_{\TV,\lambda},
\end{eqnarray*}
we have
\begin{align*}
&\frac{1}{T}\norm{\bM^\star-\bK \hat{\bm{\theta}}}_{2}^{2}\leq \frac{1}{T}\norm{\bM^\star-\bK \bm{\theta}}_{2}^{2}+ \frac{2}{T}\langle \bK (\hat{\bm{\theta}}-\bm{\theta}),\beps \rangle+ \norm{\bm{\theta}}_{\TV,\lambda}-\norm{\hat{\bm{\theta}}}_{\TV,\lambda}.
\end{align*}
Define the block diagonal matrix $\mathbf{D}={\rm diag}(D_{1},\ldots,D_{T})$ is the $Td \times Td$ matrix with the $d \times d$ matrix $D_{r}$ for $r=1,\ldots,T$,
\begin{equation}\label{equation_D_r}
D_{r}=\bigg[\begin{array}{ccccc}1&0& & &0\\ -1&1& & & \\ & & \ddots & \ddots & \\0& & & -1&1\end{array}\bigg] \in \mathbb R^{d} \times  \mathbb R^{d}.
\end{equation} 
Then, we remark that for all $\theta_{r\bullet} \in \mathbb R^{d}$,
\begin{eqnarray*}
\norm{\bm{\theta}}_{\TV,\lambda}&=&\sum_{r=1}^{T}\norm{\theta_{r\bullet}}_{\TV,\lambda}=\sum_{r=1}^{T}\sum_{j=2}^{d}\lambda_{j}|\theta_{r, j}-\theta_{r, (j-1)}|\\
&=&\sum_{r=1}^{T}\sum_{j=1}^{d}\lambda_{j}\norm{D_{r}\theta_{r\bullet}}_{1},
\end{eqnarray*}
Moreover, we define $\mathbf{V}$ as the inverse of matrix $\mathbf{D}$, i.e., $\mathbf{VD=I}$, where $\mathbf{V}={\rm diag}(V_{1},\ldots,V_{T})$ is the $Td \times Td$ matrix with the $(d \times d)$ lower triangular matrix matrix $V_{r}$, and the entries $ \Big(V_{r}\Big)_{s, j}=0$ if $ s<j $ and $ \Big(V_{r}\Big)_{s, j}=1$ otherwise. For $\lambda_{j} > 0$, we consider the event 
\begin{equation}\label{equation_event_TV}
\mathscr{U}_{T}^{\lambda_{j}}=\bigcap_{r=1}^{T} \bigcap_{j=1}^{d}\mathscr{U}_{r,j}^{\lambda_{j}},
 \text{ where }
\mathscr{U}_{r,j}^{\lambda_{j}} = \Big\{\frac{2}{T} \big|\varepsilon^{T}(K_{r\bullet}V_{r})_{j}\big| \leq \lambda_{j} \Big\}.
\end{equation}
Note on $\mathscr{U}_{T}^{\lambda_{j}}$, one has 
\begin{eqnarray*}
B_{2}&=&\frac{2}{T}\langle \bK(\hat{\bm{\theta}}- \bm{\theta}),\beps \rangle\\ 
&=&\frac{2}{T} \Big| \beps^{T} \bK V \cdot D(\hat{\bm{\theta}}- \bm{\theta}) \Big|\\
&=& \frac{2}{T} \Big|\sum_{r=1}^{T} \varepsilon^{T}K_{r\bullet}V_{r} (D_{r}(\hat{\theta}_{r\bullet}- \theta_{r\bullet}))\Big|\\ 
&=& \frac{2}{T} \Big|\sum_{r=1}^{T}\sum_{j=1}^{d} \varepsilon^{T}(K_{r\bullet}V_{r})_{j} (D_{r}(\hat{\theta}_{r\bullet}- \theta_{r\bullet}))_{j}\Big|\\ 
& \leq& \sum_{r=1}^{T} \sum_{j=1}^{d} \lambda_{j} |(D_{r}(\hat{\theta}_{r\bullet}- \theta_{r\bullet}))|\\
&\leq& \sum_{r=1}^{T}\norm{(\hat{\theta}_{r\bullet}- \theta_{r\bullet})}_{\TV,\lambda}\\
&=& \norm{\hat{\bm{\theta}}-\bm{\theta}}_{\TV,\lambda}.
\end{eqnarray*}
Similar as the Lemma 3.4.1 of \cite{alaya2016segmentation}, for all $\theta, \theta' \in \mathbb{R}^{dT}$, one has $\Omega(\theta + \theta') \leq \Omega(\theta) + \Omega(\theta')$ and $\Omega(-\theta) \leq \Omega(\theta)$,
putting things together, we have
\begin{eqnarray*}
\frac{1}{T}\norm{\bM^\star-\bK \hat{\bm{\theta}}}_{2}^{2}
&\leq& \frac{1}{T}\norm{\bM^\star-\bK \bm{\theta}}_{2}^{2}+ \norm{\bm{\theta}}_{\TV,\lambda}+ \norm{\hat{\bm{\theta}}}_{\TV,\lambda}+ \norm{\bm{\theta}}_{\TV,\lambda}- \norm{\hat{\bm{\theta}}}_{\TV,\lambda}\\
&=&\frac{1}{T}\norm{\bM^\star-\bK \bm{\theta}}_{2}^{2}+2\norm{\bm{\theta}}_{\TV,\lambda}.
\end{eqnarray*}

Consider the $\{\varepsilon_{t,T}\}_{t=1}^{T}$ follows $\beta$-mixing sub-Weibull distribution in Example~\ref{example3-(i)}, apply Proposition~\ref{proposition.Locally-stationary-sub-Weibull-distribution}, let $1/\eta=1/\eta_{1}+1/\eta_{2}$ and bandwith $h=\bigO(T^{-\xi}) \leq C_{h}T^{-\xi}$ with $0<\xi<\frac{1}{2}$ and constant $C_{h}>0$, for any $\lambda \geq  2C_{K}\sqrt{\log T / T} $ and $T>4$, we find that the probability of the complementary event ${\mathscr{U}^{\lambda_{j}}_{T}}$ is
\begin{align*} \mathbb P[({\mathscr{U}_{T}^{\lambda_{j}}})^{\complement}]
& \leq \sum_{r=1}^{T} \sum_{j=1}^{d} \mathbb P[ {(\mathscr{U}_{r,j}^{\lambda_{j}})}^\complement]\\
& \leq \sum_{r=1}^{T} \sum_{j=1}^{d} \mathbb{P}\big( \sum_{q=j}^{d} \frac{2}{T}\sum_{t=1}^{T}  |K_{h,1}(\frac{t}{T}-\frac{r}{T}) K_{h,2}(X_{t,T}^{q}-X_{r,T}^{q})\varepsilon_{t,T} | \leq \lambda_{j} \big) \\
&\leq \sum_{r=1}^{T} \sum_{j=1}^{d} \mathbb{P} \big((d-j+1)\frac{2}{T}\sum_{t=1}^{T}  |K_{h,1}(\frac{t}{T}-\frac{r}{T}) K_{h,2}(X_{t,T}^{j}-X_{r,T}^{j})\varepsilon_{t,T} | \leq \lambda_{j} \big),
\end{align*}
Refer to the proof of Theorem~\ref{theorem:least_sq_oracle_ineq_lasso_pen_subweibull}, we have that if the sample size satisfies $ T \geq c(\log d)^{\frac{2}{\eta}-1}$ with $1/2 \leq \eta<1$ and $c>1$, set $\lambda_{j} \geq (d-j+1)\sqrt{\frac{c\log d +\log T}{T^{1-2\xi}}}$, and the bandwidth $h=\bigO(T^{-\xi})$ with $0<\xi<1/2$, the probability of the complementary event $\mathscr{U}_{T}^{\lambda_{j}}$, i.e. $\mathbb P[(\mathscr{U}_{T}^{\lambda_{j}})^{\complement}]$, with a probability larger than $1 - d^{1 - c}$, we have  
\begin{equation*}
R(\hat m, m^\star) \leq \inf_{\bm{\theta} \in\R^{T d}} \big\{R(m_{\theta}, m^\star) + 2 \lambda\norm{\bm{\theta}}_{\TV,\lambda}\big\}.
\end{equation*}

\subsection{Proof of Theorem~\ref{theorem:least_sq_oracle_ineq_lasso_pen_ragularvarying}: oracle inequality for regular varying heavy-tailed distribution with Lasso}
Similar as the proof of Theorem~\ref{theorem:least_sq_oracle_ineq_lasso_pen_subweibull}, note that on $\mathscr{U}^\lambda_{T}$, one has 
\begin{align*}
B_{2}=\frac{1}{T}\langle \bK(\hat{\bm{\theta}}- \bm{\theta}),\beps \rangle
 \leq \sum_{r=1}^{T} \sum_{j=1}^{d} \lambda|\hat{\theta}_{r, j}| +  \sum_{r=1}^{T} \sum_{j=1}^{d} \lambda|\theta_{r, j}|= \lambda (\norm{\hat{\bm{\theta}}}_1 + \norm{\bm{\theta}}_1).
\end{align*}
Putting things together, on $\mathscr{U}^\lambda_{T} $, we get $\frac{1}{T}\norm{\bM^\star-\bK \hat{\bm{\theta}}}_{2}^{2}
\leq \frac{1}{T}\norm{\bM^\star-\bK \bm{\theta}}_{2}^{2}+2 \lambda \norm{\bm{\theta}}_1.$
It means as $R(\hat m, m^\star) \leq \inf_{\bm{\theta} \in\R^{T d}} \big\{R(m_{\theta}, m^\star) + 2 \lambda \norm{\bm{\theta}}_1\big\}.$

Consider the $\{\varepsilon_{t,T}\}_{t=1}^{T}$ follows regular varying heavy-tailed distribution in Definition~\ref{def:Regularvaryingtail}, apply Proposition~\ref{proposition.Locally-Stationary-regularly-varying-heavy-tailed}, let $h=\bigO(T^{-\xi})$ with $0<\xi<\vartheta$, for any $\gamma>\frac{2C_{K,L}(2T+1)}{T^{1+\vartheta}h}$, then we find that the probability of the complementary event ${\mathscr{U}^\lambda_{T}} $ is
\begin{align*}
\mathbb P[({\mathscr{U}_{T}^{\lambda}})^{\complement}]
&\leq \sum_{r=1}^{T} \sum_{j=1}^{d} \mathbb P \bigg(\frac{1}{T}\sum_{t=1}^{T}  |K_{h,1}(\frac{t}{T}-\frac{r}{T}) K_{h,2}(X_{t,T}^{j}-X_{r,T}^{j})\varepsilon_{t,T} | \leq \lambda \bigg)\\
& \leq Td \Big[(T\log T)^{-\eta_{2}}L(T\log T)\\
& \quad\qquad+\frac{12T^{(d_{1}-1/\eta_{1})(1-\eta_{2})} \lambda^{(d_{1}-1/\eta_{1})(1-\eta_{2})-1}}{ 2^{1-\eta_{2}}(4C_{K})^{(d_{1}-1/\eta_{1})(1-\eta_{2})-1}}L(\frac{(\lambda T/4C_{K})^{d_{1}-1/\eta_{1}}}{2})\\
&\quad\qquad + \frac{24C_{K} \exp(-\varphi \lambda T/4C_{K})}{\lambda } +2\exp\big({-\frac{1}{9T^{2d_{1}-1/\eta_{1}-1}(\lambda/4C_{K})^{2d_{1}-1/\eta_{1}-2}}}\big)\\
& \quad\qquad + \frac{12T^{2(d_{1}-1/\eta_{1})(1-\eta_{2})-1} (\lambda h)^{(d_{1}-1/\eta_{1})(1-\eta_{2})-1}}{ 2^{1-\eta_{2}}(4C_{K,L}(2T+1))^{(d_{1}-1/\eta_{1})(1-\eta_{2})-1}}L(\frac{( \lambda T^{2}h)^{d_{1}-1/\eta_{1}}}{2(4C_{K,L}(2T+1))^{d_{1}-1/\eta_{1}}})\\
&\quad\qquad + \frac{24C_{K,L}(2T+1)\exp(-\varphi (\frac{\lambda T^{2}h}{4 C_{K,L}(2T+1)}))}{\lambda Th}+2\exp\big({-\frac{(4C_{K,L}(2T+1))^{2d_{1}-1/\eta_{1}-2}}{9T^{4d_{1}-2/\eta_{1}-3}(\lambda h)^{2d_{1}-1/\eta_{1}-2}}}\big)\Big].
\end{align*}
Therefore, 
\begin{align*}
\mathbb P[({\mathscr{U}_{T}^{\lambda}})^{\complement}]
& \simeq Td \Big[\frac{L(T\log T)}{(T\log T)^{\eta_{2}}}\\
& \quad\qquad +\frac{C_{1}}{ T^{(d_{1}-1/\eta_{1})(\eta_{2}-1)} \lambda^{1-(d_{1}-1/\eta_{1})(\eta_{2}-1)} }L( (\lambda T )^{d_{1}-1/\eta_{1}})\\
&\quad\qquad + \frac{C_{2} \exp(-\varphi \lambda T)}{\lambda } +\exp\big(-C_{3}T^{1+1/\eta_{1}-2d_{1}}\lambda^{2+1/\eta_{1}-2d_{1}}\big)\\
& \quad\qquad +\frac{C_{4}}{ T^{(d_{1}-1/\eta_{1})(\eta_{2}-1)} (\lambda h)^{1+(d_{1}-1/\eta_{1})(\eta_{2}-1)} }L( (\lambda Th )^{d_{1}-1/\eta_{1}})\\
&\quad\qquad + \frac{C_{5} \exp(-\varphi \lambda Th)}{\lambda h} +\exp\big(-C_{6}T^{1+1/\eta_{1}-2d_{1}}(\lambda h)^{2+1/\eta_{1}-2d_{1}}\big)\Big]\\
& \simeq Td \Big[\frac{L(T\log T)}{(T\log T)^{\eta_{2}}} 
+\frac{C_{7}}{ T^{(d_{1}-1/\eta_{1})(\eta_{2}-1)} (\lambda h)^{1+(d_{1}-1/\eta_{1})(\eta_{2}-1)} }L( (\lambda T )^{d_{1}-1/\eta_{1}})\\
&\quad\qquad + \frac{C_{8} T \exp(-\varphi \lambda Th)}{\lambda h} +\exp\big(-C_{9}T^{1+1/\eta_{1}-2d_{1}}(\lambda h)^{2+1/\eta_{1}-2d_{1}}\big)\Big]\\
& \simeq d \Big[ \frac{C_{7}}{ T^{(d_{1}-1/\eta_{1})(\eta_{2}-1)-1} (\lambda h)^{1+(d_{1}-1/\eta_{1})(\eta_{2}-1)} }\\
&\quad\qquad + \frac{C_{8} T\exp(-\varphi \lambda Th)}{\lambda h} +T\exp\big(-C_{9}T^{1+1/\eta_{1}-2d_{1}}(\lambda h)^{2+1/\eta_{1}-2d_{1}}\big)\Big],
\end{align*}
with $\varphi>0$, $\eta_{1}>1$, $\eta_{2}\geq 2$, $0<\vartheta<\frac{(\eta_{1}-1)(\eta_{2}-1)}{1+(2\eta_{1}-1)\eta_{2}}$, $\frac{\vartheta}{(1-\vartheta)(\eta_{2}-1)}+\frac{1}{\eta_{1}} < d_{1}< \frac{1-2\vartheta}{2(1-\vartheta)}+\frac{1}{2\eta_{1}}$, then we have $1-\vartheta > 1+1/\eta_{1}-2d_{1}-\vartheta(2+1/\eta_{1}-2d_{1})$, which means that
\begin{equation*}
\frac{C_{8} T\exp(-\varphi \lambda Th)}{\lambda h} < T\exp\big(-C_{9}T^{1+1/\eta_{1}-2d_{1}}(\lambda h)^{2+1/\eta_{1}-2d_{1}}\big)\Big],
\end{equation*}
then we can have that
\begin{align*}
\mathbb P[({\mathscr{U}_{T}^{\lambda}})^{\complement}]
& \simeq d \Big[ \frac{C_{7}}{ T^{(d_{1}-1/\eta_{1})(\eta_{2}-1)-1} (\lambda h)^{1+(d_{1}-1/\eta_{1})(\eta_{2}-1)}}\\
&\quad\qquad + T\exp\big(-C_{10}T^{1+1/\eta_{1}-2d_{1}}(\lambda h)^{2+1/\eta_{1}-2d_{1}}\big)\Big],
\end{align*}
where $C_{K}=C_{K_{1}}C_{K_{2}}$, the constant $C_{K,L}$ depend on kernel bound and Lipschiz constant, the constants $C_{1}$, $C_{2}$ and $C_{3}$ depend on $C_{K}$, the constants $C_{4}$, $C_{5}$ and $C_{6}$ depend on $C_{K,L}$, the constants $C_{7}$ depends on $C_{K}$, $C_{K,L}$ and the bound of $L(\cdot)$, $C_{8}$, $C_{9}$ and $C_{10}$ depend on $C_{K}$ and $C_{K,L}$.

Note that we want each term above tends to $0$ when $T$ is very large. For the first term, we want $(d_{1}-1/\eta_{1})(\eta_{2}-1)-1-\vartheta(1+(d_{1}-1/\eta_{1})(\eta_{2}-1))>0$, which means that $\frac{1+\vartheta}{(1-\vartheta)(\eta_{2}-1)}+\frac{1}{\eta_{1}} < d_{1}< \frac{1-2\vartheta}{2(1-\vartheta)}+\frac{1}{2\eta_{1}}$ with $0<\vartheta< \min\{\frac{(\eta_{1}-1)(\eta_{2}-1)-2\eta_{1}}{1+(2\eta_{1}-1)\eta_{2}}, \frac{(\eta_{1}-1)(\eta_{2}-1)}{1+(2\eta_{1}-1)\eta_{2}}\}=\frac{(\eta_{1}-1)(\eta_{2}-1)-2\eta_{1}}{1+(2\eta_{1}-1)\eta_{2}}$ and $\eta_{2}>\frac{3\eta_{1}-1}{\eta_{1}-1}$. Obviously, the second term tends to $0$ when $T$ is very large.

~\\
If we set
\begin{equation*}
\lambda \geq \max \Big\{ \frac{d^{\frac{c}{(d_{1}-1/\eta_{1})(\eta_{2}-1)+1}}}{T^{\frac{(d_{1}-1/\eta_{1})(\eta_{2}-1)-1}{(d_{1}-1/\eta_{1})(\eta_{2}-1)+1}-\xi}}, \quad  \frac{(c\log d + \log T)^{\frac{1}{2+1/\eta_{1}-2d_{1}}}}{T^{\frac{1+1/\eta_{1}-2d_{1}}{2+1/\eta_{1}-2d_{1}}-\xi}} \Big\},
\end{equation*}
then the probability above is at most $d\exp(-c \log d )=d^{1-c}$.
Note that the constant $c>1$ can be made arbitrarily large but affects the constants $C_{7}$ and $C_{10}$ above. As $\frac{(d_{1}-1/\eta_{1})(\eta_{2}-1)-1}{(d_{1}-1/\eta_{1})(\eta_{2}-1)+1} > \frac{1+1/\eta_{1}-2d_{1}}{2+1/\eta_{1}-2d_{1}}$, we want
\begin{equation*}
\frac{d^{\frac{c}{(d_{1}-1/\eta_{1})(\eta_{2}-1)+1}}}{T^{\frac{(d_{1}-1/\eta_{1})(\eta_{2}-1)-1}{(d_{1}-1/\eta_{1})(\eta_{2}-1)+1}-\xi}} \leq \frac{(c\log d + \log T)^{\frac{1}{2+1/\eta_{1}-2d_{1}}}}{T^{\frac{1+1/\eta_{1}-2d_{1}}{2+1/\eta_{1}-2d_{1}}-\xi}},
\end{equation*}
which is implied by 
\begin{equation*}
T > \big( \frac{d^{c(2+1/\eta_{1}-2d_{1})}}{(c\log d+\log T)^{(d_{1}-1/\eta_{1})(\eta_{2}-1)+1}}\big)^{\frac{1}{(\eta_{2}+3)d_{1}-(\eta_{2}+1)/\eta_{1}-3}},
\end{equation*}
then we can get that
\begin{equation*}
\lambda \geq \frac{(c\log d + \log T)^{\frac{1}{2+1/\eta_{1}-2d_{1}}}}{T^{\frac{1+1/\eta_{1}-2d_{1}}{2+1/\eta_{1}-2d_{1}}-\xi}},
\end{equation*}
obviously, if 
$$T > \big( \frac{d^{c(2+1/\eta_{1}-2d_{1})}}{(c\log d)^{(d_{1}-1/\eta_{1})(\eta_{2}-1)+1}}\big)^{\frac{1}{(\eta_{2}+3)d_{1}-(\eta_{2}+1)/\eta_{1}-3}},$$
the condition $\lambda \geq \max\{\frac{2C_{K,L}(2T+1)}{T^{1+\vartheta}h},\frac{(c\log d + \log T)^{\frac{1}{2+1/\eta_{1}-2d_{1}}}}{T^{\frac{1+1/\eta_{1}-2d_{1}}{2+1/\eta_{1}-2d_{1}}-\xi}}\}=\frac{2C_{K,L}(2T+1)}{T^{1+\vartheta}h}$, then we have
\begin{equation*}
\frac{1}{T}\norm{\bM^\star-\bK \hat{\bm{\theta}}}_{2}^{2}
\leq \frac{1}{T}\norm{\bM^\star-\bK \bm{\theta}}_{2}^{2}+2 \lambda \norm{\bm{\theta}}_1.
\end{equation*}
It means that
\begin{equation*}
R(\hat m, m^\star) \leq \inf_{\bm{\theta} \in\R^{T d}} \big\{R(m_{\theta}, m^\star) + 2 \lambda \norm{\bm{\theta}}_1\big\}.
\end{equation*}

\subsection{Proof of Corollary~\ref{corol:pareto}: Pareto distribution with Lasso}
We consider $\{\varepsilon_{t,T}\}_{t=1}^{T}$ follows the Pareto distribution as Example~\ref{example4-(ii)} with $\eta_{2}=4$ and $L(v)=u^{4}$, the constant $u > 0$. The sequence $W_{t,r,T}^{j}$ be defined in (\ref{equation-W_t,T}). Assumption~\ref{Ass:1}-\ref{Ass:KB2} are satisfied with $\eta_{1}=4$. Let $h=\bigO(T^{-\xi})$ with $0<\xi<\vartheta$, for any $\lambda>\frac{2C_{K,L}(2T+1)}{T^{1+\vartheta}h}$, similar as the proof of Theorem~\ref{theorem:least_sq_oracle_ineq_lasso_pen_ragularvarying}, we have 
\begin{align*}
\mathbb P[({\mathscr{U}_{T}^{\lambda}})^{\complement}]
&\leq \sum_{r=1}^{T} \sum_{j=1}^{d} \mathbb P \bigg(\frac{1}{T}\sum_{t=1}^{T}  |K_{h,1}(\frac{t}{T}-\frac{r}{T}) K_{h,2}(X_{t,T}^{j}-X_{r,T}^{j})\varepsilon_{t,T} | \leq \lambda \bigg)\\
& \simeq d \Big[ \frac{C_{1}}{ T^{(d_{1}-1/\eta_{1})(\eta_{2}-1)-1} (\lambda h)^{1+(d_{1}-1/\eta_{1})(\eta_{2}-1)} }L( (\lambda T )^{d_{1}-1/\eta_{1}}) \\
&\quad\qquad + T\exp\big(-C_{2}T^{1+1/\eta_{1}-2d_{1}}(\lambda h)^{2+1/\eta_{1}-2d_{1}}\big)\Big],\\
& \simeq d \Big[ \frac{C_{1}u^{4}}{ T^{(3d_{1}-7/4)} (\lambda h)^{3d_{1}+1/4}} + T\exp\big(-C_{2}T^{5/4-2d_{1}}(\lambda h)^{9/4-2d_{1}}\big)\Big],
\end{align*}
where $0<\vartheta<1/29$, $d_{1} \in (\frac{1+\vartheta}{3(1-\vartheta)}+\frac{1}{4},\frac{1-2\vartheta}{2(1-\vartheta)}+\frac{1}{8})$, $C_{K}=C_{K_{1}}C_{K_{2}}$, $\varphi>0$ and the constant $C_{K,L}$ depend on kernel bound and Lipschiz constant, $C_{1}$ depends on $C_{K}$, $C_{K,L}$ and $u$, $C_{2}$ depends on $C_{K}$ and $C_{K,L}$. If we set
\begin{equation*}
\lambda \geq \max \Big\{ \frac{d^{\frac{c}{(3d_{1}+1/4)}}}{T^{\frac{3d_{1}-7/4}{3d_{1}+1/4}-\xi}}, \quad  \frac{(c\log d + \log T)^{\frac{1}{9/4-2d_{1}}}}{T^{\frac{5/4-2d_{1}}{9/4-2d_{1}}-\xi}} \Big\},
\end{equation*}
then the probability above is at most $d\exp(-c \log d )=d^{1-c}$.
Note that the constant $c>1$ can be made arbitrarily large but affects the constants $C_{h}$, $C_{5}$ and $C_{6}$ above. As $\frac{3d_{1}-7/4}{3d_{1}+1/4}>\frac{5/4-2d_{1}}{9/4-2d_{1}} $, we want
\begin{equation*}
\frac{d^{\frac{c}{(3d_{1}+1/4)}}}{T^{\frac{3d_{1}-7/4}{3d_{1}+1/4}-\xi}} \leq \frac{(c\log d + \log T)^{\frac{1}{9/4-2d_{1}}}}{T^{\frac{5/4-2d_{1}}{9/4-2d_{1}}-\xi}},
\end{equation*}
which is implied by 
\begin{equation*}
T > \big( \frac{d^{c(9/4-2d_{1})}}{(c\log d+\log T)^{3d_{1}+1/4}}\big)^{1/(7d_{1}-17/4)},
\end{equation*}
then we can get that
\begin{equation*}
\lambda \geq \frac{(c\log d + \log T)^{\frac{1}{9/4-2d_{1}}}}{T^{\frac{5/4-2d_{1}}{9/4-2d_{1}}-\xi}},
\end{equation*}
obviously, if $T > \big( \frac{d^{c(9/4-2d_{1})}}{(c\log d)^{3d_{1}+1/4}}\big)^{1/(7d_{1}-17/4)}$, the regularized parameter $\lambda$ satisfies 
\begin{align*}
    \lambda \geq \max\{\frac{2C_{K,L}(2T+1)}{T^{1+\vartheta}h},\frac{(c\log d + \log T)^{\frac{1}{9/4-2d_{1}}}}{T^{\frac{5/4-2d_{1}}{9/4-2d_{1}}-\xi}}\}=\frac{2C_{K,L}(2T+1)}{T^{1+\vartheta}h}
\end{align*}
then we have
\begin{equation*}
\frac{1}{T}\norm{\bM^\star-\bK \hat{\bm{\theta}}}_{2}^{2}
\leq \frac{1}{T}\norm{\bM^\star-\bK \bm{\theta}}_{2}^{2}+2 \lambda \norm{\bm{\theta}}_1.
\end{equation*}
It means that
\begin{equation*}
R(\hat m, m^\star) \leq \inf_{\bm{\theta} \in\R^{T d}} \big\{R(m_{\theta}, m^\star) + 2 \lambda \norm{\bm{\theta}}_1\big\}.
\end{equation*}

\subsection{Proof of Theorem~\ref{theorem:TV_Regularly_V_H}: oracle inequality for regular varying heavy-tailed distribution with weighted total variation penalization}
Consider $\{\varepsilon_{t, T}\}_{t=1}^{T}$ follows the regularly varying heavy-tailed with bounded slowly varying function $L(\cdot)$, we get the following result similar to the proof of Theorem~\ref{theorem:least_sq_oracle_ineq_TV_pen_subweibull} and Theorem~\ref{theorem:least_sq_oracle_ineq_lasso_pen_ragularvarying}, i,e, assume the sample size satisfies 
$ T > \big( \frac{d^{c(2+1/\eta_{1}-2d_{1})}}{(c\log d)^{(d_{1}-1/\eta_{1})(\eta_{2}-1)+1}}\big)^{\frac{1}{(\eta_{2}+3)d_{1}-(\eta_{2}+1)/\eta_{1}-3}},$
and $\lambda_{j} \geq (d-j+1) \frac{2C_{K,L}(2T+1)}{T^{1+\vartheta}h}$
with the bandwidth $h=\bigO(T^{-\xi})$ with $0<\xi<\vartheta$, the probability of the complementary event $\mathscr{U}_{T}^{\lambda_{j}}$, i.e. $\mathbb P[(\mathscr{U}_{T}^{\lambda_{j}})^{\complement}]$, with a probability larger than $1 - d^{1 - c}$, we have  
\begin{equation*}
R(\hat m, m^\star) \leq \inf_{\bm{\theta} \in\R^{T d}} \big\{R(m_{\theta}, m^\star) + 2 \lambda\norm{\bm{\theta}}_{\TV,\lambda}\big\}.
\end{equation*}

\subsection{Proof of Theorem~\ref{theorem:fast_lasso_pen_SubWeibull}: fast oracle inequality for sub-Weibull distribution with Lasso}
\textbf{Step 1.} 
From the definition of $\hat{\bm{\theta}}$, we have
\begin{equation*}
\frac{1}{T}\|\bm{Y}-\bm{K} \hat{\bm{\theta}}\|_{2}^{2}+\lambda \norm{\hat{\bm{\theta}}}_1 \leq \frac{1}{T}\|\bm{Y}-\bm{K}\bm{\theta}\|_{2}^{2}+\lambda \norm{\bm{\theta}}_1,
\end{equation*}
and
\begin{equation*}
\frac{1}{T}\|\bm{Y}-\bm{K} \hat{\bm{\theta}}\|_{2}^{2} -\frac{1}{T}\|\bm{Y}-\bm{K}\bm{\theta}\|_{2}^{2} \geq \frac{1}{T} \|\bm{K}(\bm{\theta} - \hat{\bm{\theta}})\|_2^2 -\frac{2}{T}\langle \varepsilon^{\top} \bm{K}, (\hat{\bm{\theta}} - \bm{\theta})\rangle,
\end{equation*}
it follows that
\begin{equation*}
\frac{1}{T} \|\bm{K}(\bm{\theta} - \hat{\bm{\theta}})\|_2^2 + \lambda \|\hat{\bm{\theta}}\|_1
\leq \lambda \|\bm{\theta} \|_1 + \frac{2}{T}\langle \varepsilon^{\top} \bm{K}, (\hat{\bm{\theta}} - \bm{\theta})\rangle,
\end{equation*}
Consider the event $\mathscr{U}_{T}^{\lambda}=\bigcap_{r=1}^{T} \bigcap_{j=1}^{d}\mathscr{U}_{r,j}^{\lambda},$where
\begin{equation*}
\mathscr{U}_{r,j}^{\lambda} = \Big\{\frac{2}{T}\sum_{t=1}^{T} \big|K_{h,1}\big(\frac{t}{T}-\frac{r}{T}\big) K_{h,2}(X_{t,T}^{j}-X_{r,T}^{j})\varepsilon_{t,T}\big| \leq \frac{\lambda}{2} \Big\},
\end{equation*}
we have $\frac{2}{T}\langle \varepsilon^{\top} \bm{K},(\hat{\bm{\theta}} - \bm{\theta})\rangle \leq \lambda/2 \|\hat{\bm{\theta}} - \bm{\theta}\|_1$, it follows that
\begin{equation*}
\frac{1}{T} \|\bm{K}(\bm{\theta} - \hat{\bm{\theta}})\|_2^2 
\leq \frac{\lambda}{2} \|\hat{\bm{\theta}} - \bm{\theta}\|_1+ \lambda \|\bm{\theta} \|_1 - \lambda \|\hat{\bm{\theta}} \|_1.
\end{equation*}
Adding $\frac{\lambda}{2} \|\hat{\bm{\theta}} - \bm{\theta}\|_1$ to both sides we get
\begin{align*}
\frac{1}{T} \|\bm{K}(\bm{\theta} - \hat{\bm{\theta}})\|_2^2 +\frac{\lambda}{2} \|\hat{\bm{\theta}} - \bm{\theta}\|_1
&\leq \lambda \big( \|\hat{\bm{\theta}} - \bm{\theta}\|_1+ \|\bm{\theta} \|_1 - \|\hat{\bm{\theta}} \|_1 \big)\\
&\leq \lambda \sum_{r=1}^{T} \sum_{j=1}^{d} (|\hat{\theta}_{r, j} - \theta_{r, j}| + |\theta_{r, j}| - |\hat{\theta}_{r, j}|)\\
&\leq \lambda \sum_{r \in J_{r}} (\|\hat{\theta}_{r\bullet} - \theta_{r\bullet}\| + \|\theta_{r\bullet}\| - \|\hat{\theta}_{r\bullet}\|)\\
&\leq 2\lambda \sum_{r \in J_{r}} \|\hat{\theta}_{r\bullet} - \theta_{r\bullet}\|\\
&= 2\lambda \| [\hat{\bm{\theta}}-\bm{\theta}]_{J} \|.
\end{align*}
It follows that $\frac{\lambda}{2} \|\hat{\bm{\theta}} - \bm{\theta}\|_1
\leq 2\lambda \| [\hat{\bm{\theta}}-\bm{\theta}]_{J} \|_{1}$, i.e.,
\begin{equation*}
\frac{\lambda}{2} \| [\hat{\bm{\theta}} - \bm{\theta}]_{J^{\complement}} \|_{1}
+\frac{\lambda}{2} \| [\hat{\bm{\theta}} - \bm{\theta}]_{J} \|_{1}
\leq 2\lambda \| [\hat{\bm{\theta}}-\bm{\theta}]_{J} \|_{1} ,  
\end{equation*}
then we have $\| [\hat{\bm{\theta}} - \bm{\theta}]_{J^{\complement}} \|_{1}
\leq 3 \| [\hat{\bm{\theta}}-\bm{\theta}]_{J} \|_{1},$
by Assumption~\ref{Assumption_RE_condition}-(i), $(\hat{\bm{\theta}} - \bm{\theta} ) \in S_{J}$. Let $\Delta=\hat{\bm{\theta}} - \bm{\theta}$, it also have that
\begin{equation*}
\frac{1}{T} \| \bm{K}\Delta \|_{2}^{2} \leq \frac{3}{2} \lambda \| \Delta_{J} \|_{1} \leq \frac{3}{2} \lambda \sqrt{J^{\star}} \| \Delta_{J} \|_{2}.
\end{equation*}
From the definition of $\kappa(\bm{K}, J(\bm{\theta}))$,
\begin{equation*}
\| \Delta_{J} \|_{2}^{2} \leq \frac{1}{\kappa^{2}(\bm{K}, J(\bm{\theta}))} \frac{\| \bm{K}\Delta \|_{2}^{2}}{T} \leq \frac{3 \lambda \sqrt{J^{\star}} \| \Delta_{J} \|_{2}}{2 \kappa^{2}(\bm{K}, J(\bm{\theta})) }. 
\end{equation*}
Therefore $\| \Delta_{J} \|_{2} \leq \frac{3 \lambda \sqrt{J^{\star}} }{2 \kappa^{2}(\bm{K}, J(\bm{\theta})) }$, then
\begin{equation*}
\| \Delta \|_{2} \leq \| \Delta_{J} \|_{2} +\| \Delta_{J^{\complement}} \|_{2} \leq \sqrt{\| \Delta_{J^{\complement}} \|_{1} \| \Delta_{J^{\complement}} \|_{\infty}} + \| \Delta_{J} \|_{2}.
\end{equation*}
From $\Delta \in S_{J}$, $\| \Delta_{J^{\complement}} \|_{1} \leq 3\| \Delta_{J} \|_{1}$. Since $\Delta_{J}$ spans the largest coordinates of $\Delta$ in absolute value, $\| \Delta_{J^{\complement}} \|_{\infty} \leq \| \Delta_{J} \|_{1} / J^{\star} $, we get
\begin{equation}\label{deltalasso}
\| \Delta \|_{2} \leq \sqrt{\frac{3}{J^{\star}}} \| \Delta_{J} \|_{1} + \| \Delta_{J} \|_{2} \leq (\sqrt{3}+1)\| \Delta_{J} \|_{2} \leq \frac{3 (\sqrt{3}+1) \lambda \sqrt{J^{\star}} }{2 \kappa^{2}(\bm{K}, J(\bm{\theta}))}
\end{equation}
and
\begin{equation}\label{Kdeltalasso}
\frac{1}{T} \| \bm{K}\Delta \|_{2}^{2} \leq \frac{9 \lambda^{2} J^{\star}} {4 \kappa^{2}(\bm{K}, J(\bm{\theta}))}.
\end{equation}

\textbf{Step 2.}
Recall that for all $\bm{\theta} \in \mathbb{R}^{Td}$,
\begin{equation*}
\hat{\bm{\theta}} = \mathop{\arg\min}\limits_{\bm{\theta} \in \mathbb{R}^{Td}} \big\{\frac{1}{T} \norm{\bY -\bK\bm{\theta}}_{2}^{2}+\norm{\bm{\theta}}_{1}\big\}.    
\end{equation*}
By Lemma~\ref{lemma_RE_subgradients_lasso}, there is a subgradient 
$\hat h = [\hat h _{r, \bullet}]_{r=1, \ldots, T} \in \partial \norm{\hat {\bm{\theta}}}_{ 1}$ such that 
\begin{equation*}
\langle \frac{2}{T} \bm{K}^{\top}\big(\bm{K} \hat{\bm{\theta}} - \bm{Y} \big) + \lambda \hat{h}, \hat{\bm{\theta}}-\bm{\theta} \rangle= 0, \text{ for all } \bm{\theta} \in \mathbb{R}^{Td},
\end{equation*}
it follows that
\begin{equation*}
\langle \frac{2}{T} \bm{K}^{\top}\big(\bm{K}\hat{\bm{\theta}} - M^{\star} \big) - \frac{2}{T} \bm{K}^{\top}\big(\bm{Y} - M^{\star} \big), \hat{\bm{\theta}}-\bm{\theta} \rangle + \lambda \langle \hat{h}, \hat{\bm{\theta}}-\bm{\theta} \rangle= 0.
\end{equation*}
Since the subdifferential mapping is monotone~\citep{rockafellar1997convex}, $\langle \hat{h}-h, \hat{\bm{\theta}}-\bm{\theta} \rangle \geq 0$, then we have
\begin{equation*}
\frac{2}{T} \langle \bm{K}\hat{\bm{\theta}} - M^{\star}, \bm{K}(\hat{\bm{\theta}}-\bm{\theta}) \rangle  - \frac{2}{T} \langle \bm{K}^{\top}\big(\bm{Y} - M^{\star} \big), \hat{\bm{\theta}}-\bm{\theta} \rangle + \lambda \langle h, \hat{\bm{\theta}}-\bm{\theta} \rangle \leq 0.
\end{equation*}
i.e.,
\begin{equation*}
\frac{2}{T} \langle \bm{K}\hat{\bm{\theta}} - M^{\star}, \bm{K}(\hat{\bm{\theta}}-\bm{\theta}) \rangle  \leq \frac{2}{T} \langle \bm{K}^{\top}\big(\bm{Y} - M^{\star} \big), \hat{\bm{\theta}}-\bm{\theta} \rangle - \lambda \langle h, \hat{\bm{\theta}}-\bm{\theta} \rangle .
\end{equation*}
For the left-hand side,
\begin{align*}
&\frac{2}{T} \langle \bm{K}\hat{\bm{\theta}} - M^{\star}, \bm{K}(\hat{\bm{\theta}}-\bm{\theta}) \rangle \\
&=\frac{1}{T}\norm{\bm{K}\hat{\bm{\theta}} - M^{\star}}_{2}^{2} + \frac{1}{T}\norm{\bm{K} (\hat{\bm{\theta}}-\bm{\theta})}_{2}^{2}-\frac{1}{T}\norm{\bm{K}\bm{\theta} - M^{\star}}_{2}^{2},
\end{align*}
then
\begin{align}\label{eqution_inequation-lasso}
&\frac{1}{T}\norm{\bm{K}\hat{\bm{\theta}} - M^{\star}}_{2}^{2} + \frac{1}{T}\norm{\bm{K} (\hat{\bm{\theta}}-\bm{\theta})}_{2}^{2} \nonumber \\
& \leq \frac{1}{T}\norm{\bm{K}\bm{\theta} - M^{\star}}_{2}^{2} + \frac{2}{T} \langle \bm{K}^{\top}\big(\bm{Y} - M^{\star} \big), \hat{\bm{\theta}}-\bm{\theta} \rangle - \lambda \langle h, \hat{\bm{\theta}}-\bm{\theta} \rangle.
\end{align}

If $\langle \bm{K}\hat{\bm{\theta}} - M^{\star}, \bm{K}(\hat{\bm{\theta}}-\bm{\theta}) \rangle < 0$, we have $\frac{1}{T}\norm{\bm{K}\hat{\bm{\theta}} - M^{\star}}_{2}^{2} < \frac{1}{T}\norm{\bm{K}\bm{\theta} - M^{\star}}_{2}^{2}$, which yiled the result. If $\langle \bm{K}\hat{\bm{\theta}} - M^{\star}, \bm{K}(\hat{\bm{\theta}}-\bm{\theta}) \geq 0$, it follows that 
\begin{equation*}
\frac{2}{T} \langle \bm{K}^{\top}\big(\bm{Y} - M^{\star} \big), \hat{\bm{\theta}}-\bm{\theta} \rangle - \lambda \langle h, \hat{\bm{\theta}}-\bm{\theta} \rangle \geq 0.
\end{equation*} 
Since
\begin{align*}
&\frac{2}{T} \langle \bm{K}^{\top}\big(\bm{Y} - M^{\star} \big), \hat{\bm{\theta}}-\bm{\theta} \rangle \\
&= \frac{2}{T} \sum_{r=1}^{T} \langle K_{r\bullet}^{\top} \big(\bm{Y} - M^{\star} \big), (\hat{\theta}_{r\bullet}-\theta_{r\bullet}) \rangle \\
& \leq \sum_{r=1}^{T} \sum_{j=1}^{d} \frac{2}{T} \sum_{t=1}^{T} \big|K_{h,1}\big(\frac{t}{T}-\frac{r}{T}\big) K_{h,2}(X_{t,T}^{j}-X_{r,T}^{j}) \varepsilon_{t,T}\big|\big|\hat{\theta}_{r, j}-\theta_{r, j}\big|..
\end{align*}
We consider the event $\mathscr{U}_{T}^{\lambda}=\bigcap_{r=1}^{T} \bigcap_{j=1}^{d}\mathscr{U}_{r,j}^{\lambda},$where
\begin{equation*}
\mathscr{U}_{r,j}^{\lambda} = \Big\{\frac{2}{T}\sum_{t=1}^{T} \big|K_{h,1}\big(\frac{t}{T}-\frac{r}{T}\big) K_{h,2}(X_{t,T}^{j}-X_{r,T}^{j})\varepsilon_{t,T}\big| \leq \frac{\lambda}{2} \Big\},
\end{equation*}
then we have
\begin{equation}\label{equ_interaction_lasso}
\frac{2}{T} \langle \bm{K}^{\top}\big(\bm{Y} - M^{\star} \big), \hat{\bm{\theta}}-\bm{\theta} \rangle \leq \frac{\lambda}{2} \sum_{r=1}^{T} \norm{\hat{\theta}_{r\bullet}-\theta_{r\bullet}}_{1}.
\end{equation}
From the definition of the subgradient $g = [g _{r, \bullet}]_{r=1, \ldots, T} \in \partial \norm{\bm{\theta}}_{1}$, see Lemma~\ref{lemma_RE_subgradients_lasso}, we can choose $g$ such that
\begin{align*}
h_{r, \bullet} &=  \sgn(\theta_{r \bullet})_{r \in \{1, \ldots, J_{r}(\bm{\theta})\}} \\
h_{r, \bullet} &= \sgn(
\hat{\theta}_{r \bullet} )_{r \in \{1, \ldots, 
J_{r}^{\complement}(\bm{\theta}) \}} = \sgn(\hat{\theta}_{r \bullet} - \theta_{r \bullet} )_{r \in \{1, \ldots, 
J_{r}^{\complement}(\bm{\theta}) \}}.
\end{align*}
This gives
\begin{align*}
&- \lambda \langle g, \hat{\bm{\theta}} - \bm{\theta} \rangle\\
&= - \lambda  \sum_{r=1}^{T} \langle h_{r, \bullet}, \hat{\theta}_{r \bullet} - \theta_{r \bullet} \rangle \\
&= \lambda  \sum_{r=1}^{T} \langle (-h_{r \bullet} )_{J_{r}(\bm{\theta})}, (\hat{\theta}_{r \bullet} - \theta_{r \bullet})_{J_{r}(\bm{\theta})} \rangle - \lambda  \sum_{r=1}^{T} \langle (h_{r, \bullet})_{J_{r}^{\complement}(\bm{\theta})}, (\hat{\theta}_{r \bullet} - \theta_{r  \bullet})_{J_{r}^{\complement}(\bm{\theta})} \rangle\\
&=\lambda  \sum_{r=1}^{T} \langle (- \sgn( \theta_{r \bullet}))_{J_{r}(\bm{\theta})},  (\hat{\theta}_{r \bullet} - \theta_{r \bullet})_{J_{r}(\bm{\theta})} \rangle - \lambda \sum_{r=1}^{T} \langle ( \sgn( \theta_{r \bullet}))_{J_{r}^{\complement}(\bm{\theta})}, (\hat{\theta}_{r \bullet} - \theta_{r \bullet})_{J_{r}^{\complement}(\bm{\theta})} \rangle.
\end{align*}

Using a triangle inequality and the fact that $\langle \sgn(x), x \rangle = \norm{x}_{1}$, imply that

\begin{equation}
\label{equ_interaction_subgradient_lasso}
-\lambda \langle g, \hat{\bm{\theta}} - \bm{\theta} \rangle
 \leq \lambda \sum_{r=1}^{T} \| (\hat{\theta}_{r \bullet} - \theta_{r \bullet})_{J_{r}(\bm{\theta})} \|_{1} - \lambda \sum_{r=1}^{T} \|  (\hat{\theta}_{r \bullet} - \theta_{r \bullet})_{J_{r}^{\complement}(\bm{\theta})} \|_{1}
\end{equation}

Note on $\mathscr{U}_{T}^{\lambda }$ with equation~(\ref{equ_interaction_lasso}) and~(\ref{equ_interaction_subgradient_lasso}), we have
\begin{equation*}
\frac{\lambda }{2} \sum_{r=1}^{T} \norm{\hat{\theta}_{r\bullet}-\theta_{r\bullet}}_{1} 
+ \lambda \sum_{r=1}^{T} \|(\hat{\theta}_{r \bullet} - \theta_{r \bullet})_{J_{r}(\bm{\theta})} \|_{1} 
- \lambda \sum_{r=1}^{T} \|(\hat{\theta}_{r \bullet} - \theta_{r \bullet})_{J_{r}^{\complement}(\bm{\theta})} \|_{1} \geq 0,    
\end{equation*}
i.e.,
\begin{equation*}
3\sum_{r=1}^{T} \|(\hat{\theta}_{r \bullet} - \theta_{r \bullet})_{J_{r}(\bm{\theta})} \|_{1} \geq \sum_{r=1}^{T} \|(\hat{\theta}_{r \bullet} - \theta_{r \bullet})_{J_{r}^{\complement}(\bm{\theta})} \|_{1}.
\end{equation*}
By Assumption~\ref{Assumption_RE_condition}-(i), we have $\hat{\bm{\theta}}-\bm{\theta} \in S_{J}$, then the Equation~(\ref{eqution_inequation-lasso}) follows
\begin{align*}
&\frac{1}{T}\norm{\bm{K}\hat{\bm{\theta}} - M^{\star}}_{2}^{2} + \frac{1}{T}\norm{\bm{K} (\hat{\bm{\theta}}-\bm{\theta})}_{2}^{2} \\
& \leq \frac{1}{T}\norm{\bm{K}\bm{\theta} - M^{\star}}_{2}^{2} + \frac{2}{T} \langle \bm{K}^{\top}\big(\bm{Y} - M^{\star} \big), \hat{\bm{\theta}}-\bm{\theta} \rangle - \lambda \langle g, \hat{\bm{\theta}}-\bm{\theta} \rangle\\
& \leq  \frac{1}{T}\norm{\bm{K}\bm{\theta} - M^{\star}}_{2}^{2} + \frac{\lambda}{2} \sum_{r=1}^{T} \norm{\hat{\theta}_{r\bullet}-\theta_{r\bullet}}_{1}\\
& \quad + \lambda \sum_{r=1}^{T} \|(\hat{\theta}_{r \bullet} - \theta_{r \bullet})_{J_{r}(\bm{\theta})} \|_{1} - \lambda \sum_{r=1}^{T} \|(\hat{\theta}_{r \bullet} - \theta_{r \bullet})_{J_{r}^{\complement}(\bm{\theta})} \|_{1}\\
& \leq  \frac{1}{T}\norm{\bm{K}\bm{\theta} - M^{\star}}_{2}^{2} + \frac{\lambda}{2} \sum_{r=1}^{T} \norm{(\hat{\theta}_{r\bullet}-\theta_{r\bullet})_{J_{r}(\bm{\theta})}}_{1}
+\frac{\lambda}{2} \sum_{r=1}^{T} \norm{(\hat{\theta}_{r\bullet}-\theta_{r\bullet})_{J_{r}^{\complement}(\bm{\theta})}}_{1}\\
& \quad + \lambda \sum_{r=1}^{T} \|(\hat{\theta}_{r \bullet} - \theta_{r \bullet})_{J_{r}(\bm{\theta})} \|_{1} - \lambda \sum_{r=1}^{T} \|(\hat{\theta}_{r \bullet} - \theta_{r \bullet})_{J_{r}^{\complement}(\bm{\theta})} \|_{1}\\
& \leq  \frac{1}{T}\norm{\bm{K}\bm{\theta} - M^{\star}}_{2}^{2} + \frac{3\lambda}{2} \sum_{r=1}^{T} \norm{(\hat{\theta}_{r\bullet}-\theta_{r\bullet})_{J_{r}(\bm{\theta})}}_{1}
-\frac{\lambda}{2} \sum_{r=1}^{T} \norm{(\hat{\theta}_{r\bullet}-\theta_{r\bullet})_{J_{r}^{\complement}(\bm{\theta})}}_{1}\\
& \leq \frac{1}{T}\norm{\bm{K}\bm{\theta} - M^{\star}}_{2}^{2} + \frac{3\lambda}{2} \sum_{r=1}^{T} \norm{(\hat{\theta}_{r\bullet}-\theta_{r\bullet})_{J_{r}(\bm{\theta})}}_{1}\\
& \leq \frac{1}{T}\norm{\bm{K}\bm{\theta} - M^{\star}}_{2}^{2} + \frac{3\lambda \sqrt{J^{\star}}}{2} \norm{(\hat{\bm{\theta}} -\bm{\theta})_{J_{r}(\bm{\theta})}}_{2},
\end{align*}
then we have
\begin{equation*}
\frac{1}{T}\norm{\bm{K}\hat{\bm{\theta}} - M^{\star}}_{2}^{2} + \frac{1}{T}\norm{\bm{K} (\hat{\bm{\theta}}-\bm{\theta})}_{2}^{2} 
\leq \frac{1}{T}\norm{\bm{K}\bm{\theta} - M^{\star}}_{2}^{2} + \frac{3\lambda \sqrt{J^{\star}}}{2} \frac{\norm{\bm{K}(\hat{\bm{\theta}}-\bm{\theta})}_{2}}{\sqrt{T} \kappa(\bm{K}, J(\bm{\theta}))}.
\end{equation*}
Since the fact $2uv \leq u^{2}+v^{2}$,
\begin{equation*}
\frac{1}{T}\norm{\bm{K}\hat{\bm{\theta}} - M^{\star}}_{2}^{2} + \frac{1}{T}\norm{\bm{K} (\hat{\bm{\theta}}-\bm{\theta})}_{2}^{2} 
\leq \frac{1}{T}\norm{\bm{K}\bm{\theta} - M^{\star}}_{2}^{2} + \frac{9 \lambda^{2} J^{\star}}{16 \kappa^{2}(\bm{K}, J(\bm{\theta}))} + \frac{1}{T}\norm{\bm{K} (\hat{\bm{\theta}}-\bm{\theta})}_{2}^{2} ,
\end{equation*}
i.e.,
\begin{equation*}
\frac{1}{T}\norm{\bm{K}\hat{\bm{\theta}} - M^{\star}}_{2}^{2} 
\leq \frac{1}{T}\norm{\bm{K}\bm{\theta} - M^{\star}}_{2}^{2} + \frac{9 \lambda^{2} J^{\star}}{16 \kappa^{2}(\bm{K}, J(\bm{\theta}))}. 
\end{equation*}
It means as
\begin{equation}\label{fastboundlasso}
R(\hat m, m^\star) \leq \inf_{\bm{\theta} \in\R^{T d}} \big\{R(m_{\theta}, m^\star) + \frac{9 \lambda^{2} J^{\star}}{16 \kappa^{2}(\bm{K}, J(\bm{\theta}))} \big\}. 
\end{equation}

From the proof of Theorem~\ref{theorem:least_sq_oracle_ineq_lasso_pen_subweibull}, if the sample size satisfies $ T \geq c(\log d)^{\frac{2}{\eta}-1}$ with $1/2 \leq \eta<1$ and $c>1$, set $\lambda = \bigO(\sqrt{\frac{c\log d +\log T}{T^{1-2\xi}}})$, and the bandwidth $g=\bigO(T^{-\xi})$ with $0<\xi<1/2$, we have $\frac{2}{T}\langle \bK(\hat{\bm{\theta}}- \bm{\theta}),\beps \rangle \leq \frac{\lambda}{2} \| \hat{\bm{\theta}}- \bm{\theta} \|_{1} $ with a probability larger than $1 - d^{1 - c}$. Combined with the (\ref{deltalasso}), (\ref{Kdeltalasso}) and (\ref{fastboundlasso}), we get the result.

\subsection{Proof of Theorem~\ref{theorem:fast_TV_pen_SubWeibull}: fast oracle inequality for sub-Weibull distribution with weighted total variation penalization}
\textbf{Step 1.}
Recall that for all $\bm{\theta} \in \mathbb{R}^{Td}$,
\begin{equation*}
\hat{\bm{\theta}} = \mathop{\arg\min}\limits_{\bm{\theta} \in \mathbb{R}^{Td}} \big\{\frac{1}{T} \norm{\bY -\bK\bm{\theta}}_{2}^{2}+\norm{\bm{\theta}}_{\TV,\lambda}\big\}.    
\end{equation*}
By Lemma~\ref{lemma_RE_subgradients_TV}, there is a subgradient 
$\hat g = [\hat g _{r, \bullet}]_{r=1, \ldots, T} \in \partial \norm{\hat {\bm{\theta}}}_{ \TV, \lambda}$ such that 
\begin{equation*}
\langle \frac{2}{T} \bm{K}^{\top}\big(\bm{K} \hat{\bm{\theta}} - \bm{Y} \big) + \hat{g}, \hat{\bm{\theta}}-\bm{\theta} \rangle= 0, \text{ for all } \bm{\theta} \in \mathbb{R}^{Td},
\end{equation*}
it follows that
\begin{equation*}
\langle \frac{2}{T} \bm{K}^{\top}\big(\bm{K}\hat{\bm{\theta}} - M^{\star} \big) - \frac{2}{T} \bm{K}^{\top}\big(\bm{Y} - M^{\star} \big), \hat{\bm{\theta}}-\bm{\theta} \rangle + \langle \hat{g}, \hat{\bm{\theta}}-\bm{\theta} \rangle= 0.
\end{equation*}
Since the subdifferential mapping is monotone~\citep{rockafellar1997convex}, $\langle \hat{g}-g, \hat{\bm{\theta}}-\bm{\theta} \rangle \geq 0$, then we have
\begin{equation*}
\frac{2}{T} \langle \bm{K}\hat{\bm{\theta}} - M^{\star}, \bm{K}(\hat{\bm{\theta}}-\bm{\theta}) \rangle  - \frac{2}{T} \langle \bm{K}^{\top}\big(\bm{Y} - M^{\star} \big), \hat{\bm{\theta}}-\bm{\theta} \rangle + \langle g, \hat{\bm{\theta}}-\bm{\theta} \rangle \leq 0.
\end{equation*}
i.e.,
\begin{equation*}
\frac{2}{T} \langle \bm{K}\hat{\bm{\theta}} - M^{\star}, \bm{K}(\hat{\bm{\theta}}-\bm{\theta}) \rangle  \leq \frac{2}{T} \langle \bm{K}^{\top}\big(\bm{Y} - M^{\star} \big), \hat{\bm{\theta}}-\bm{\theta} \rangle - \langle g, \hat{\bm{\theta}}-\bm{\theta} \rangle .
\end{equation*}
For the left-hand side,
\begin{align*}
&\frac{2}{T} \langle \bm{K}\hat{\bm{\theta}} - M^{\star}, \bm{K}(\hat{\bm{\theta}}-\bm{\theta}) \rangle \\
&=\frac{1}{T}\norm{\bm{K}\hat{\bm{\theta}} - M^{\star}}_{2}^{2} + \frac{1}{T}\norm{\bm{K} (\hat{\bm{\theta}}-\bm{\theta})}_{2}^{2}-\frac{1}{T}\norm{\bm{K}\bm{\theta} - M^{\star}}_{2}^{2},
\end{align*}
then
\begin{align}\label{eqution_inequation-TV}
&\frac{1}{T}\norm{\bm{K}\hat{\bm{\theta}} - M^{\star}}_{2}^{2} + \frac{1}{T}\norm{\bm{K} (\hat{\bm{\theta}}-\bm{\theta})}_{2}^{2} \nonumber \\
& \leq \frac{1}{T}\norm{\bm{K}\bm{\theta} - M^{\star}}_{2}^{2} + \frac{2}{T} \langle \bm{K}^{\top}\big(\bm{Y} - M^{\star} \big), \hat{\bm{\theta}}-\bm{\theta} \rangle - \langle g, \hat{\bm{\theta}}-\bm{\theta} \rangle.
\end{align}

If $\langle \bm{K}\hat{\bm{\theta}} - M^{\star}, \bm{K}(\hat{\bm{\theta}}-\bm{\theta}) \rangle < 0$, we have $\frac{1}{T}\norm{\bm{K}\hat{\bm{\theta}} - M^{\star}}_{2}^{2} < \frac{1}{T}\norm{\bm{K}\bm{\theta} - M^{\star}}_{2}^{2}$, which yiled the result. If $\langle \bm{K}\hat{\bm{\theta}} - M^{\star}, \bm{K}(\hat{\bm{\theta}}-\bm{\theta}) \geq 0$, it follows that 
\begin{equation*}
\frac{2}{T} \langle \bm{K}^{\top}\big(\bm{Y} - M^{\star} \big), \hat{\bm{\theta}}-\bm{\theta} \rangle - \langle g, \hat{\bm{\theta}}-\bm{\theta} \rangle \leq 0.
\end{equation*} 
Let $D^{-1}=V$, we have
\begin{align*}
\frac{2}{T} \langle \bm{K}^{\top}\big(\bm{Y} - M^{\star} \big), \hat{\bm{\theta}}-\bm{\theta} \rangle
&= \frac{2}{T} \langle (\bm{K} V)^{\top} \big(\bm{Y} - M^{\star} \big), D\hat{\bm{\theta}}-\bm{\theta} \rangle \\
&= \frac{2}{T} \sum_{r=1}^{T} \langle (K_{r\bullet}V_{r})^{\top} \big(\bm{Y} - M^{\star} \big), D_{r}(\hat{\theta}_{r\bullet}-\theta_{r\bullet}) \rangle \\
&= \frac{2}{T} \sum_{r=1}^{T} \sum_{j=1}^{d} \bigg( (K_{r\bullet}V_{r})^{\top} \big(\bm{Y} - M^{\star} \big)\bigg)_{j} \bigg( D_{r}(\hat{\theta}_{r\bullet}-\theta_{r\bullet})\bigg)_{j} \\
&\leq \frac{2}{T} \sum_{r=1}^{T} \sum_{j=1}^{d} \bigg| \varepsilon ^{\top} (K_{r\bullet}V_{r})_{j} \bigg| \bigg| \bigg( D_{r}(\hat{\theta}_{r\bullet}-\theta_{r\bullet})\bigg)_{j} \bigg|.
\end{align*}
We consider the event 
\begin{equation}\label{equation_event_RE_TV}
\mathscr{U}_{T}^{\lambda_{j}}=\bigcap_{r=1}^{T} \bigcap_{j=1}^{d}\mathscr{U}_{r,j}^{\omega},
 \text{ where }
\mathscr{U}_{r,j}^{\omega} = \Big\{\frac{1}{T} \big|\varepsilon^{T}(K_{r\bullet}V_{r})_{j}\big| \leq \frac{\lambda_{j}}{4} \Big\},
\end{equation}
then we have
\begin{align}\label{equ_interaction}
\frac{2}{T} \langle \bm{K}^{\top}\big(\bm{Y} - M^{\star} \big), \hat{\bm{\theta}}-\bm{\theta} \rangle &\leq \frac{1}{2} \sum_{r=1}^{T} \sum_{j=1}^{d} \lambda_{j} \bigg| \bigg( D_{r}(\hat{\theta}_{r\bullet}-\theta_{r\bullet})\bigg)_{j} \bigg| \nonumber\\
& =\sum_{r=1}^{T}\sum_{j=1}^{d}\lambda_{j}\norm{D_{r}(\hat{\theta}_{r\bullet}-\theta_{r\bullet})}_{1} \nonumber\\
& =\frac{1}{2} \sum_{r=1}^{T} \norm{\hat{\theta}_{r\bullet}-\theta_{r\bullet}}_{\TV,\lambda}.
\end{align}
From the definition of the subgradient $g = [g _{r, \bullet}]_{r=1, \ldots, T} \in \partial \norm{\bm{\theta}}_{\TV, \lambda}$, see Lemma~\ref{lemma_RE_subgradients_TV}, we can choose $g$ such that
\begin{align*}
h_{r, \bullet} &= \bigg( D_{r}^{\top} (\lambda_{j} \odot \sgn(D_{r} 
\theta_{r \bullet})) \bigg)_{r \in \{1, \ldots, J_{r}(\bm{\theta})\}} \\
h_{r, \bullet} &= \bigg( D_{r}^{\top} (\lambda_{j} \odot \sgn(D_{r} 
\hat{\theta}_{r \bullet} )) \bigg)_{r \in \{1, \ldots, 
J_{r}^{\complement}(\bm{\theta}) \}} \\
&= \bigg( D_{r}^{\top} (\hat{\lambda}_{j} \odot \sgn(D_{r} (\theta_{r \bullet}- 
\theta_{r \bullet}))) \bigg)_{r \in \{1, \ldots, J_{r}^{\complement}
(\bm{\theta}) \}}.
\end{align*}
This gives
\begin{align*}
&-\langle g, \hat{\bm{\theta}} - \bm{\theta} \rangle\\
&= -\sum_{r=1}^{T} \langle h_{r, \bullet}, \hat{\theta}_{r \bullet} - \theta_{r \bullet} \rangle \\
&= \sum_{r=1}^{T} \langle (-h_{r \bullet} )_{J_{r}(\bm{\theta})}, (\hat{\theta}_{r \bullet} - \theta_{r \bullet})_{J_{r}(\bm{\theta})} \rangle - \sum_{r=1}^{T} \langle (h_{r, \bullet})_{J_{r}^{\complement}(\bm{\theta})}, (\hat{\theta}_{r \bullet} - \theta_{r  \bullet})_{J_{r}^{\complement}(\bm{\theta})} \rangle\\
&=\sum_{r=1}^{T} \langle (-\lambda_{j} \odot \sgn(D_{r} \theta_{r \bullet}))_{J_{r}(\bm{\theta})}, D_{r} (\hat{\theta}_{r \bullet} - \theta_{r \bullet})_{J_{r}(\bm{\theta})} \rangle \\
& \qquad - \sum_{r=1}^{T} \langle (\lambda_{j} \odot \sgn(D_{r} \theta_{r \bullet}))_{J_{r}^{\complement}(\bm{\theta})}, D_{r} (\hat{\theta}_{r \bullet} - \theta_{r \bullet})_{J_{r}^{\complement}(\bm{\theta})} \rangle.
\end{align*}

Using a triangle inequality and the fact that $\langle \sgn(x), x \rangle = \norm{x}_{1}$, imply that
\begin{align}\label{equ_interaction_subgradient}
-\langle g, \hat{\bm{\theta}} - \bm{\theta} \rangle 
& \leq \sum_{r=1}^{T} \| ( \lambda_{j} \odot D_{r} (\hat{\theta}_{r \bullet} - \theta_{r \bullet}))_{J_{r}(\bm{\theta})} \|_{1} - \sum_{r=1}^{T} \| ( \lambda_{j} \odot D_{r} (\hat{\theta}_{r \bullet} - \theta_{r \bullet}))_{J_{r}^{\complement}(\bm{\theta})} \|_{1} \nonumber\\
& = \sum_{r=1}^{T} \|(\hat{\theta}_{r \bullet} - \theta_{r \bullet})_{J_{r}(\bm{\theta})} \|_{\TV,\lambda} - \sum_{r=1}^{T} \|(\hat{\theta}_{r \bullet} - \theta_{r \bullet})_{J_{r}^{\complement}(\bm{\theta})} \|_{\TV,\lambda}
\end{align}
Note on $\mathscr{U}_{T}^{\lambda_{j}}$ with equation~(\ref{equ_interaction}) and~(\ref{equ_interaction_subgradient}), we have
\begin{equation*}
\frac{1}{2} \sum_{r=1}^{T} \norm{\hat{\theta}_{r\bullet}-\theta_{r\bullet}}_{\TV,\lambda} 
+ \sum_{r=1}^{T} \|(\hat{\theta}_{r \bullet} - \theta_{r \bullet})_{J_{r}(\bm{\theta})} \|_{\TV,\lambda} 
- \sum_{r=1}^{T} \|(\hat{\theta}_{r \bullet} - \theta_{r \bullet})_{J_{r}^{\complement}(\bm{\theta})} \|_{\TV,\lambda} \geq 0,    
\end{equation*}
i.e.,
\begin{equation*}
3\sum_{r=1}^{T} \|(\hat{\theta}_{r \bullet} - \theta_{r \bullet})_{J_{r}(\bm{\theta})} \|_{\TV,\lambda} \geq \sum_{r=1}^{T} \|(\hat{\theta}_{r \bullet} - \theta_{r \bullet})_{J_{r}^{\complement}(\bm{\theta})} \|_{\TV,\lambda},
\end{equation*}
it also means that
\begin{equation*}
3\sum_{r=1}^{T}\sum_{j=1}^{d}\lambda_{j}\norm{D_{r}(\hat{\theta}_{r \bullet} - \theta_{r \bullet})_{J_{r}}}_{1} \geq \sum_{r=1}^{T}\sum_{j=1}^{d}\lambda_{j}\norm{D_{r}(\hat{\theta}_{r \bullet} - \theta_{r \bullet})_{J_{r}^{\complement}}}_{1}.
\end{equation*}
By Assumption~\ref{Assumption_RE_condition}-(ii) and Lemma~\ref{lemma-compatibility-1_lambda}, we have
\begin{equation*}
\hat{\bm{\theta}}-\bm{\theta} \in S_{\TV,J} \text{ and } D(\hat{\bm{\theta}}-\bm{\theta}) \in S_{1,J}.
\end{equation*}
The Equation~(\ref{eqution_inequation-TV}) follows
\begin{align*}
&\frac{1}{T}\norm{\bm{K}\hat{\bm{\theta}} - M^{\star}}_{2}^{2} + \frac{1}{T}\norm{\bm{K} (\hat{\bm{\theta}}-\bm{\theta})}_{2}^{2} \\
& \leq \frac{1}{T}\norm{\bm{K}\bm{\theta} - M^{\star}}_{2}^{2} + \frac{2}{T} \langle \bm{K}^{\top}\big(\bm{Y} - M^{\star} \big), \hat{\bm{\theta}}-\bm{\theta} \rangle - \langle g, \hat{\bm{\theta}}-\bm{\theta} \rangle\\
& \leq  \frac{1}{T}\norm{\bm{K}\bm{\theta} - M^{\star}}_{2}^{2} + \frac{1}{2} \sum_{r=1}^{T} \norm{\hat{\theta}_{r\bullet}-\theta_{r\bullet}}_{\TV,\lambda}\\
& \quad + \sum_{r=1}^{T} \|(\hat{\theta}_{r \bullet} - \theta_{r \bullet})_{J_{r}(\bm{\theta})} \|_{\TV,\lambda} - \sum_{r=1}^{T} \|(\hat{\theta}_{r \bullet} - \theta_{r \bullet})_{J_{r}^{\complement}(\bm{\theta})} \|_{\TV,\lambda}\\
& \leq  \frac{1}{T}\norm{\bm{K}\bm{\theta} - M^{\star}}_{2}^{2} + \frac{1}{2} \sum_{r=1}^{T} \norm{(\hat{\theta}_{r\bullet}-\theta_{r\bullet})_{J_{r}(\bm{\theta})}}_{\TV,\lambda}
+\frac{1}{2} \sum_{r=1}^{T} \norm{(\hat{\theta}_{r\bullet}-\theta_{r\bullet})_{J_{r}^{\complement}(\bm{\theta})}}_{\TV,\lambda}\\
& \quad + \sum_{r=1}^{T} \|(\hat{\theta}_{r \bullet} - \theta_{r \bullet})_{J_{r}(\bm{\theta})} \|_{\TV,\lambda} - \sum_{r=1}^{T} \|(\hat{\theta}_{r \bullet} - \theta_{r \bullet})_{J_{r}^{\complement}(\bm{\theta})} \|_{\TV,\lambda}\\
& \leq  \frac{1}{T}\norm{\bm{K}\bm{\theta} - M^{\star}}_{2}^{2} + \frac{3}{2} \sum_{r=1}^{T} \norm{(\hat{\theta}_{r\bullet}-\theta_{r\bullet})_{J_{r}(\bm{\theta})}}_{\TV,\lambda}
-\frac{1}{2} \sum_{r=1}^{T} \norm{(\hat{\theta}_{r\bullet}-\theta_{r\bullet})_{J_{r}^{\complement}(\bm{\theta})}}_{\TV,\lambda},
\end{align*}
it also means that
\begin{align*}
&\frac{1}{T}\norm{\bm{K}\hat{\bm{\theta}} - M^{\star}}_{2}^{2} + \frac{1}{T}\norm{\bm{K} (\hat{\bm{\theta}}-\bm{\theta})}_{2}^{2} \\
& \leq  \frac{1}{T}\norm{\bm{K}\bm{\theta} - M^{\star}}_{2}^{2} + \frac{3}{2} \sum_{r=1}^{T}\sum_{j=1}^{d}\lambda_{j}\norm{D_{r}(\hat{\theta}_{r\bullet}-\theta_{r\bullet})_{J_{r}(\bm{\theta})}}_{1}\\
& \quad -\frac{1}{2} \sum_{r=1}^{T}\sum_{j=1}^{d}\lambda_{j}\norm{D_{r}(\hat{\theta}_{r\bullet}-\theta_{r\bullet})_{J_{r}^{\complement}(\bm{\theta})}}_{1}\\
&\leq \frac{1}{T}\norm{\bm{K}\bm{\theta} - M^{\star}}_{2}^{2} + \frac{3}{2} \sum_{r=1}^{T}\sum_{j=1}^{d}\lambda_{j}\norm{D_{r}(\hat{\theta}_{r\bullet}-\theta_{r\bullet})_{J_{r}(\bm{\theta})}}_{1}.
\end{align*}
By using Lemma~\ref{lemma-compatibility-1_lambda}, we have
\begin{equation*}
\frac{1}{T}\norm{\bm{K}\hat{\bm{\theta}} - M^{\star}}_{2}^{2} + \frac{1}{T}\norm{\bm{K} (\hat{\bm{\theta}}-\bm{\theta})}_{2}^{2} 
\leq \frac{1}{T}\norm{\bm{K}\bm{\theta} - M^{\star}}_{2}^{2} + 2 \frac{\norm{\bm{K}(\hat{\bm{\theta}}-\bm{\theta})}_{2}}{\sqrt{T} \kappa_{V,\gamma}(J(\bm{\theta})) \kappa(\bm{K}, J(\bm{\theta}))},  
\end{equation*}
where $ \gamma= (\gamma_{1,1}, \dots, \gamma_{1,d}, \dots, \gamma_{T,1}, \dots, \gamma_{T,d})^{\top} \in \mathbb{R}^{Td}_{+}$ such that
\begin{equation*}
\forall r = 1, \ldots, T, \gamma_{r,j} = 
\begin{cases} 
\frac{3}{2} \lambda_{j}, & \text{if } j \in J_{r}(\bm{\theta}), \\
0, & \text{if } j \in J_{r}^{\complement}(\bm{\theta}),
\end{cases}    
\end{equation*}
and
\begin{equation*}
\kappa_{V,\gamma}(J) = \bigg\{ 32  \sum_{r=1}^{T}\sum_{j=1}^{d} | \gamma_{r,j} -\gamma_{r,j-1}|^{2}+  2|J_{r}|\norm{\gamma_{r,\bullet}}_\infty^{2}\Lambda_{\min, J_{r}}^{-1}\bigg\}^{-1/2}.
\end{equation*}
Since the fact $2uv \leq u^{2}+v^{2}$,
\begin{align*}
&\frac{1}{T}\norm{\bm{K}\hat{\bm{\theta}} - M^{\star}}_{2}^{2} + \frac{1}{T}\norm{\bm{K} (\hat{\bm{\theta}}-\bm{\theta})}_{2}^{2} \\
&\leq \frac{1}{T}\norm{\bm{K}\bm{\theta} - M^{\star}}_{2}^{2} + 2 \frac{\norm{\bm{K}(\hat{\bm{\theta}}-\bm{\theta})}_{2}}{\sqrt{T} \kappa_{V,\gamma}(J(\bm{\theta})) \kappa(\bm{K}, J(\bm{\theta}))}\\
&\leq \frac{1}{T}\norm{\bm{K}\bm{\theta} - M^{\star}}_{2}^{2} + \frac{1}{ \kappa_{V,\gamma}^{2}(J(\bm{\theta})) \kappa^{2}(\bm{K}, J(\bm{\theta}))}+\frac{1}{T}\norm{\bm{K}(\hat{\bm{\theta}}-\bm{\theta})}_{2}^{2},
\end{align*}
i.e.,
\begin{equation*}
\frac{1}{T}\norm{\bm{K}\hat{\bm{\theta}} - M^{\star}}_{2}^{2} 
\leq \frac{1}{T}\norm{\bm{K}\bm{\theta} - M^{\star}}_{2}^{2} + \frac{1}{ \kappa_{V,\gamma}^{2}(J(\bm{\theta})) \kappa^{2}(\bm{K}, J(\bm{\theta}))}. 
\end{equation*}
Obviously, 
\begin{equation*}
\frac{1}{ \kappa_{V,\gamma}^{2}(J(\bm{\theta}))} =32  \sum_{r=1}^{T}\sum_{j=1}^{d} | \gamma_{r,j} -\gamma_{r,j-1}|^{2}+  2|J_{r}|\norm{\gamma_{r,\bullet}}_\infty^{2}\Lambda_{\min, J_{r}}^{-1},
\end{equation*}
We write set $ J_{r}(\bm{\theta}) = \{j_{r}^{1}, \dots, j_{r}^{|J_{r}(\bm{\theta})|} \}$ and we set $B_{s} = \{ j_{r}^{s-1}+1, \dots, j_{r}^{s}\}$ for $ s \in \{1, \dots, |J_r(\bm{\theta})|\} $ with the convention that $ j_{r}^{0} = 0 $. Then
\begin{align*}
\sum_{j=1}^{d} |\gamma_{r,j} - \gamma_{r,j-1}|^2 
&= \sum_{s=1}^{|J_{r}(\bm{\theta})|} \sum_{j \in B_{s}} |\gamma_{r,j} - \gamma_{r,j-1}|^2\\
&= \sum_{s=1}^{|J_j(\theta)|} \bigg\{ |\gamma_{r,j_{r}^{s-1}+1} - \gamma_{r,j_{r}^{s-1}}|^2 + |\gamma_{r,j_{r}^{s}} - \gamma_{r,j_{r}^{s}-1}|^2 \bigg\}\\
&= \sum_{r=1}^{|J_j(\theta)|} \bigg\{ \gamma_{r,j_{r}^{s-1}}^{2} + \gamma_{r,j_{r}^{s}}^{2} \bigg\}\\
&= \sum_{r=1}^{|J_{r}(\bm{\theta})|} 2\gamma_{r,j_{r}^{s}}^{2}
\leq \frac{9}{2}|J_{r}(\bm{\theta})| \|(\lambda_{j})_{J_{r}(\bm{\theta})}\|_{\infty}^{2}.
\end{align*}
Therefore
\begin{align*}
&\frac{1}{ \kappa_{V,\gamma}^{2}(J(\bm{\theta}))} \\
& \leq 32  \sum_{r=1}^{T}\big(\big\{\frac{9}{2}|J_{r}(\bm{\theta})| \|(\lambda_{j})_{J_{r}(\bm{\theta})}\|_{\infty}^{2}\big\}+  \frac{9}{2}|J_{r}(\bm{\theta})| \|(\lambda_{j})_{J_{r}(\bm{\theta})}\|_{\infty}^{2}\Lambda_{\min, J_{r}}^{-1}\big)\\
&\leq 32  \sum_{r=1}^{T}\big(\big\{\frac{9}{2}+\frac{9}{2\Lambda_{\min, J_{r}}} \big\} |J_{r}(\bm{\theta})| \|(\lambda_{j})_{J_{r}(\bm{\theta})}\|_{\infty}^{2} \big)\\
& \leq 288 J^{\star} \max_{r=1,\dots,T}\|(\lambda_{j})_{J_{r}(\bm{\theta})}\|_{\infty}^{2}.
\end{align*}
Then we have
\begin{equation*}
\frac{1}{T}\norm{\bm{K}\hat{\bm{\theta}} - M^{\star}}_{2}^{2} 
\leq \frac{1}{T}\norm{\bm{K}\bm{\theta} - M^{\star}}_{2}^{2} + \frac{288 J^{\star}}{ \kappa^{2}(\bm{K}, J(\bm{\theta}))}\max_{r=1,\dots,T}\|(\lambda_{j})_{J_{r}(\bm{\theta})}\|_{\infty}^{2}. 
\end{equation*}
It means as
\begin{equation}\label{fastboundTV}
R(\hat m, m^\star) \leq \inf_{\bm{\theta} \in\R^{T d}} \big\{R(m_{\theta}, m^\star)  + \frac{288 J^{\star}}{ \kappa^{2}(\bm{K}, J(\bm{\theta}))}\max_{r=1,\dots,T}\|(\lambda_{j})_{J_{r}(\bm{\theta})}\|_{\infty}^{2} \big\}. 
\end{equation}

\textbf{Step 2.} 
From the definition of $\hat{\bm{\theta}}$, we have
\begin{equation*}
\frac{1}{T}\|\bm{Y}-\bm{K} \hat{\bm{\theta}}\|_{2}^{2}+ \norm{\hat{\bm{\theta}}}_{\TV,\lambda} \leq \frac{1}{T}\|\bm{Y}-\bm{K}\bm{\theta}\|_{2}^{2}+ \norm{\bm{\theta}}_{\TV,\lambda},
\end{equation*}
and
\begin{equation*}
\frac{1}{T}\|\bm{Y}-\bm{K} \hat{\bm{\theta}}\|_{2}^{2} -\frac{1}{T}\|\bm{Y}-\bm{K}\bm{\theta}\|_{2}^{2} \geq \frac{1}{T} \|\bm{K}(\bm{\theta} - \hat{\bm{\theta}})\|_2^2 -\frac{2}{T}\langle \varepsilon^{\top} \bm{K}, (\hat{\bm{\theta}} - \bm{\theta})\rangle,
\end{equation*}
it follows that
\begin{equation*}
\frac{1}{T} \|\bm{K}(\bm{\theta} - \hat{\bm{\theta}})\|_2^2 + \|\hat{\bm{\theta}}\|_{\TV,\lambda}
\leq \|\bm{\theta} \|_{\TV,\lambda} + \frac{2}{T}\langle \varepsilon^{\top} \bm{K}, (\hat{\bm{\theta}} - \bm{\theta})\rangle.
\end{equation*}
Let $D^{-1}=V$, we have
\begin{align*}
\frac{2}{T} \langle \bm{K}^{\top} \varepsilon , \hat{\bm{\theta}}-\bm{\theta} \rangle 
&= \frac{2}{T} \langle (\bm{K} V)^{\top} \varepsilon, D\hat{\bm{\theta}}-\bm{\theta} \rangle \\
&= \frac{2}{T} \sum_{r=1}^{T} \langle (K_{r\bullet}V_{r})^{\top} \varepsilon, D_{r}(\hat{\theta}_{r\bullet}-\theta_{r\bullet}) \rangle \\
&= \frac{2}{T} \sum_{r=1}^{T} \sum_{j=1}^{d} \bigg( (K_{r\bullet}V_{r})^{\top} \varepsilon\bigg)_{j} \bigg( D_{r}(\hat{\theta}_{r\bullet}-\theta_{r\bullet})\bigg)_{j} \\
&\leq \frac{2}{T} \sum_{r=1}^{T} \sum_{j=1}^{d} \bigg| \varepsilon ^{\top} (K_{r\bullet}V_{r})_{j} \bigg| \bigg| \bigg( D_{r}(\hat{\theta}_{r\bullet}-\theta_{r\bullet})\bigg)_{j} \bigg|.
\end{align*}
We consider the event (\ref{equation_event_RE_TV}),
then we have
\begin{equation*}
\frac{1}{T} \|\bm{K}(\bm{\theta} - \hat{\bm{\theta}})\|_2^2 
\leq \frac{1}{2} \|\hat{\bm{\theta}} - \bm{\theta}\|_{\TV,\lambda}+ \|\bm{\theta} \|_{\TV,\lambda} - \lambda \|\hat{\bm{\theta}} \|_{\TV,\lambda}.
\end{equation*}
Adding $\frac{1}{2} \|\hat{\bm{\theta}} - \bm{\theta}\|_{\TV,\lambda}$ to both sides we get
\begin{align*}
\frac{1}{T} \|\bm{K}(\bm{\theta} - \hat{\bm{\theta}})\|_2^2 +\frac{1}{2} \|\hat{\bm{\theta}} - \bm{\theta}\|_{\TV,\lambda}
&\leq \|\hat{\bm{\theta}} - \bm{\theta}\|_{\TV,\lambda}+ \|\bm{\theta} \|_{\TV,\lambda} - \|\hat{\bm{\theta}} \|_{\TV,\lambda}\\
&\leq (\|\hat{\bm{\theta}} - \bm{\theta}\|_{\TV,\lambda}+ \|\bm{\theta} \|_{\TV,\lambda} - \|\hat{\bm{\theta}} \|_{\TV,\lambda})\\
&\leq \sum_{r=1}^{T} (\|\hat{\theta}_{r \bullet} - \theta_{r \bullet}\|_{\TV,\lambda} + \|\theta_{r\bullet}\|_{\TV,\lambda} - \|\hat{\theta}_{r\bullet}\|_{\TV,\lambda})\\
&\leq \sum_{r \in J_{r}} (\|\hat{\theta}_{r\bullet} - \theta_{r\bullet}\|_{\TV,\lambda} + \|\theta_{r\bullet}\|_{\TV,\lambda} - \|\hat{\theta}_{r\bullet}\|_{\TV,\lambda})\\
&\leq 2 \sum_{r \in J_{r}} \|\hat{\theta}_{r\bullet} - \theta_{r\bullet}\|_{\TV,\lambda}\\
&= 2 \| [\hat{\bm{\theta}}-\bm{\theta}]_{J} \|_{\TV,\lambda},
\end{align*}
similarly,
\begin{equation*}
\frac{1}{T} \|\bm{K}(\bm{\theta} - \hat{\bm{\theta}})\|_2^2 +\frac{1}{2} \sum_{r=1}^{T}\sum_{j=1}^{d}\lambda_{j}\norm{D_{r}(\hat{\theta}_{r \bullet} - \theta_{r \bullet})}_{1} \leq 2 \sum_{r=1}^{T}\sum_{j=1}^{d}\lambda_{j}\norm{D_{r}(\hat{\theta}_{r \bullet} - \theta_{r \bullet})_{J_{r}}}_{1}.
\end{equation*}
It follows that $\frac{1}{2} \|\hat{\bm{\theta}} - \bm{\theta}\|_{\TV,\lambda}
\leq 2 \| [\hat{\bm{\theta}}-\bm{\theta}]_{J} \|_{\TV,\lambda}$, i.e.,
\begin{equation*}
\frac{1}{2} \| [\hat{\bm{\theta}} - \bm{\theta}]_{J^{\complement}} \|_{\TV,\lambda}
+\frac{1}{2} \| [\hat{\bm{\theta}} - \bm{\theta}]_{J} \|_{\TV,\lambda}
\leq 2 \| [\hat{\bm{\theta}}-\bm{\theta}]_{J} \|_{\TV,\lambda},  
\end{equation*}
then we have
\begin{equation*}
\| [\hat{\bm{\theta}} - \bm{\theta}]_{J^{\complement}} \|_{\TV,\lambda}
\leq 3 \| [\hat{\bm{\theta}}-\bm{\theta}]_{J} \|_{\TV,\lambda},  
\end{equation*}
similarly,
\begin{equation*}
\sum_{r=1}^{T}\sum_{j=1}^{d}\lambda_{j}\norm{D_{r}(\hat{\theta}_{r \bullet} - \theta_{r \bullet})_{J_{r}^{\complement}}}_{1} \leq 3\sum_{r=1}^{T}\sum_{j=1}^{d}\lambda_{j}\norm{D_{r}(\hat{\theta}_{r \bullet} - \theta_{r \bullet})_{J_{r}}}_{1}.
\end{equation*}
By Assumption~\ref{Assumption_RE_condition}-(ii), we have
\begin{equation*}
\hat{\bm{\theta}}-\bm{\theta} \in S_{\TV,J} \text{ and } D(\hat{\bm{\theta}}-\bm{\theta}) \in S_{1,J}.
\end{equation*}
Let $\Delta=\hat{\bm{\theta}} - \bm{\theta}$, it also have that
\begin{align*}
\frac{1}{T} \| \bm{K}\Delta \|_{2}^{2} &\leq \frac{3}{2} \sum_{r=1}^{T}\sum_{j=1}^{d}\lambda_{j}\norm{D_{r}(\hat{\theta}_{r \bullet} - \theta_{r \bullet})_{J_{r}}}_{1}\\
&\leq \frac{\norm{\bm{K} \Delta}_{2}}{\sqrt{T} \kappa_{V,\gamma}(J(\bm{\theta})) \kappa(\bm{K}, J(\bm{\theta}))},  
\end{align*}
where $ \gamma= (\gamma_{1,1}, \dots, \gamma_{1,d}, \dots, \gamma_{T,1}, \dots, \gamma_{T,d})^{\top} \in \mathbb{R}^{Td}_{+}$ such that
\begin{equation*}
\forall r = 1, \ldots, T, \gamma_{r,j} = 
\begin{cases} 
\frac{3}{2} \lambda_{j}, & \text{if } j \in J_{r}(\bm{\theta}), \\
0, & \text{if } j \in J_{r}^{\complement}(\bm{\theta}),
\end{cases}    
\end{equation*}
and
\begin{equation*}
\kappa_{V,\gamma}(J) = \bigg\{ 32  \sum_{r=1}^{T}\sum_{j=1}^{d} | \gamma_{r,j} -\gamma_{r,j-1}|^{2}+  2|J_{r}|\norm{\gamma_{r,\bullet}}_\infty^{2}\Lambda_{\min, J_{r}}^{-1}\bigg\}^{-1/2}.
\end{equation*}
It follows that
\begin{equation}\label{KdeltaTV}
\frac{1}{T} \| \bm{K}\Delta \|_{2}
\leq \frac{ 1 }{ \sqrt{T} \kappa_{V,\gamma}(J(\bm{\theta})) \kappa(\bm{K}, J(\bm{\theta}))} \leq \frac{ \sqrt{288 J^{\star} }\max\limits_{r=1,\dots,T}\|(\lambda_{j})_{J_{r}(\bm{\theta})}\|_{\infty}}{ \sqrt{T} \kappa(\bm{K}, J(\bm{\theta})) },  
\end{equation}
From the definition of $\kappa(\bm{K}, J(\bm{\theta}))$ in Assumption~\ref{Assumption_RE_condition}-(ii),
\begin{equation*}
\| \Delta_{J} \|_{2} \leq \frac{1}{\kappa(\bm{K}, J(\bm{\theta}))} \frac{\| \bm{K}\Delta \|_{2} }{\sqrt{T}} \leq \frac{1}{ \kappa^{2}(\bm{K}, J(\bm{\theta})) \kappa_{V,\gamma}(J(\bm{\theta}))},
\end{equation*}
then
\begin{equation*}
\| \Delta \|_{2} = \| \Delta_{J} \|_{2} +\| \Delta_{J^{\complement}} \|_{2} \leq \sqrt{\| \Delta_{J^{\complement}} \|_{1} \| \Delta_{J^{\complement}} \|_{\infty}} + \| \Delta_{J} \|_{2}.
\end{equation*}
From $\Delta \in S_{J}$, $\| \Delta_{J^{\complement}} \|_{1} \leq 3\| \Delta_{J} \|_{1}$. Since $\Delta_{J}$ spans the largest coordinates of $\Delta$ in absolute value, $\| \Delta_{J^{\complement}} \|_{\infty} \leq \| \Delta_{J} \|_{1} / J^{\star} $, we get
\begin{equation}\label{deltaTV}
\| \Delta \|_{2} \leq \sqrt{\frac{3}{J^{\star}}} \| \Delta_{J} \|_{1} + \| \Delta_{J} \|_{2} \leq (\sqrt{3}+1)\| \Delta_{J} \|_{2} \leq \frac{(\sqrt{3}+1)  \sqrt{288 J^{\star} }\max\limits_{r=1,\dots,T}\|(\lambda_{j})_{J_{r}(\bm{\theta})}\|_{\infty}}{ \kappa^{2}(\bm{K}, J(\bm{\theta})) }.
\end{equation}

From the proof of Theorem~\ref{theorem:least_sq_oracle_ineq_TV_pen_subweibull}, if the sample size satisfies $ T \geq c(\log d)^{\frac{2}{\eta}-1}$ with $1/2 \leq \eta<1$ and $c>1$, set $\lambda_{j} = (d-j+1)\sqrt{\frac{c\log d +\log T}{T^{1-2\xi}}}$, and the bandwidth $h=\bigO(T^{-\xi})$ with $0<\xi<1/2$, the event on Equation~(\ref{equation_event_RE_TV}) satisfied with a probability larger than $1 - d^{1 - c}$, $c>1$, the constant $C_{1}$, $C_{2}$ depend only on $c$. Combined with the Combined with Combined with the (\ref{deltaTV}), (\ref{KdeltaTV}) and (\ref{fastboundTV}), we get the result., we get the result.

\subsection{Proof of Theorem~\ref{theorem:fast_lasso_pen_ragularvarying}: fast oracle inequality for regular varying heavy-tailed distribution with Lasso penalization}
From the proof of Theorem~\ref{theorem:least_sq_oracle_ineq_lasso_pen_ragularvarying}, if the sample size satisfies 
$$ T > \big( \frac{d^{c(2+1/\eta_{1}-2d_{1})}}{(c\log d)^{(d_{1}-1/\eta_{1})(\eta_{2}-1)+1}}\big)^{\frac{1}{(\eta_{2}+3)d_{1}-(\eta_{2}+1)/\eta_{1}-3}},$$
set 
$$\lambda = \frac{2C_{K,L}(2T+1)}{T^{1+\vartheta-\xi}}$$
with the bandwidth $h=\bigO(T^{-\xi})$, $0<\xi<\vartheta$, $c>1$, $0<\vartheta< \frac{(\eta_{1}-1)(\eta_{2}-1)-2\eta_{1}}{1+(2\eta_{1}-1)\eta_{2}}$, $\varphi, \eta_{1}>1$, $\eta_{2}>\frac{3\eta_{1}-1}{\eta_{1}-1}$, $\frac{1+\vartheta}{(1-\vartheta)(\eta_{2}-1)}+\frac{1}{\eta_{1}} < d_{1}< \frac{1-2\vartheta}{2(1-\vartheta)}+\frac{1}{2\eta_{1}}$, the constant $C_{K,L}$ depend on kernel bound and Lipschiz constant, then we have $\frac{2}{T}\langle \bK(\hat{\bm{\theta}}- \bm{\theta}),\beps \rangle \leq \frac{\lambda}{2} \| \hat{\bm{\theta}}- \bm{\theta} \|_{1} $ with a probability larger than $1 - d^{1 - c}$. Similar to the Proof of Theorem~\ref{theorem:fast_lasso_pen_SubWeibull}, combined with the (\ref{deltalasso}), (\ref{Kdeltalasso}) and (\ref{fastboundlasso}), we get the result.

\subsection{Proof of Theorem~\ref{theorem:fast_TV_pen_ragularvarying}: fast oracle inequality for regular varying heavy-tailed distribution with weighted total variation penalization}
From the proof of Theorem~\ref{theorem:TV_Regularly_V_H}, if the sample size satisfies 
$$ T > \big( \frac{d^{c(2+1/\eta_{1}-2d_{1})}}{(c\log d)^{(d_{1}-1/\eta_{1})(\eta_{2}-1)+1}}\big)^{\frac{1}{(\eta_{2}+3)d_{1}-(\eta_{2}+1)/\eta_{1}-3}},$$
and 
$$\lambda_{j} \geq (d-j+1) \frac{2C_{K,L}(2T+1)}{T^{1+\vartheta-\xi}}$$
with the bandwidth $h=\bigO(T^{-\xi})$ with $0<\xi<\vartheta$, the event on Equation~(\ref{equation_event_RE_TV}) satisfied
with a probability larger than $1 - d^{1 - c}$, $c>1$, $0<\vartheta< \frac{(\eta_{1}-1)(\eta_{2}-1)-2\eta_{1}}{1+(2\eta_{1}-1)\eta_{2}}$, $\varphi, \eta_{1}>1$, $\eta_{2}>\frac{3\eta_{1}-1}{\eta_{1}-1}$, $\frac{1+\vartheta}{(1-\vartheta)(\eta_{2}-1)}+\frac{1}{\eta_{1}} < d_{1}< \frac{1-2\vartheta}{2(1-\vartheta)}+\frac{1}{2\eta_{1}}$, the constant $C_{K,L}$ depend on kernel bound and Lipschiz constant. Similar to the Proof of Theorem~\ref{theorem:fast_TV_pen_SubWeibull}, combined with the (\ref{deltaTV}), (\ref{KdeltaTV}) and (\ref{fastboundTV}), we get the result.

\section{Useful lemmas} \label{sec:useful_lemmas}
In this section, we give the lemmas that are useful for the proof of the technical lemmas in Section \ref{sec:Proofs_of_Proposition}. 
First, we illustrate the concentration inequality for sums of stationary $\beta$-mixing sub-Weibull random variables in Lemma 13 of \citep{Wong2020}. We will need this result to give the concentration inequality for locally stationary random variables.

\begin{lemma}[Stationary sub-Weibull distribution (Lemma 13 in \cite{Wong2020}]\label{lemma.sumofmixingsubweibull}
Let $\{Z_{t,T}\}_{t=1}^{T}$ be a strictly stationary $\beta$-mixing sequence of zero mean random variables with {\it $\beta$-mixing} coefficient $ \beta(k) \leq \exp \bigg(-\varphi k^{\eta_{1}}\bigg) $, for some $ \varphi, \eta_{1}>1$. If it follows the {\it sub-Weibull ($\eta_{2}$)} with sub-Weibull constant $C_{\varepsilon}$ and $\eta$ be a parameter given by
\begin{equation*}
\frac{1}{\eta} = \frac{1}{\eta_{1}} + \frac{1}{\eta_{2}}, \quad \eta < 1.
\end{equation*}
Then, for $T \geq 4$ and any $\varrho > 1/T$,
\begin{equation*}
\mathbb{P}\Big(\frac{1}{T}\Big|\sum_{t=1}^{T} Z_{t, T}\Big|>\varrho \Big) \leq T \exp \bigg( -\frac{(\varrho T)^{\eta}}{C_{\varepsilon}^{\eta}C_{1}} \bigg) + \exp \bigg( -\frac{\varrho^{2}T}{C_{\varepsilon}^{2}C_{2}} \bigg),  
\end{equation*}
where the constants $C_{1}$, $C_{2}$ depend only on $\eta_{1}$, $\eta_{2}$ and $\varphi$.
\end{lemma}

Next, we present the lemmas that are useful for the concentration inequality for stationary, regularly varying, heavy-tailed random variables.
\begin{lemma}[Karamata's theorem \citep{Bingham1989}]\label{Karamata_theorem}
Let $\textsc{L}$ is slowly varying function and locally bounded on $[a, \infty)$, $a \geq 0$, $\eta_{2}>1$. Then,
\begin{equation*}
\int_{r}^{\infty} v^{-\eta_{2}} \textsc{L}(v) \mathrm{d} v \sim -(1-\eta_{2})^{-1}r^{1-\eta_{2}} \textsc{L}(r), \quad r \rightarrow \infty.
\end{equation*}
\end{lemma}

\begin{lemma}[Lemma 5 in \cite{Dedecker2004}]\label{lemma_lemma5_of_Dedecker}
Let $ ( \Omega, \mathscr{A} , \mathbb{P} ) $ be a probability space, $X$ an integrable real-valued random variable, and $\mathscr{M}$ a $\sigma$-algebra of $\mathscr{A}$. Assume that there exists a random variable $\delta$ uniformly distributed over [0,1], independent of the $\sigma$-algebra generated by $X$ and $\mathscr{M}$. Then there exists a random variable $X^{*}$, measurable with respect to $\mathscr{M}\lor\sigma(X)\lor\sigma(\delta)$, independent of $\mathscr{M}$ and distributed as $X$, such that
\begin{equation*}
\|X-X^{*}\|_1=\tau(\mathscr{M},X),
\end{equation*}
the coefficient $\tau$ is now defined by
\begin{equation*}
\tau(\mathscr{M},X)=\|W(P_{X|\mathscr{M}})\|_1,  
\end{equation*}
where
\begin{equation*}
W(P_{X\mid\mathscr{M}})=\sup\bigg\{\bigg|\int f(x)P_{X\mid\mathscr{M}}(dx)-\int f(x)P_{X}(dx)\bigg|,\:f\in\Lambda_{1}(\mathbb{R})\bigg\},
\end{equation*}
$\Lambda_{1}(\mathbb{R})$ is the class of $1$-Lipschitz functions from $\mathbb{R}$ to $\mathbb{R}$.
\end{lemma}

\begin{lemma}[Hoeffding's inequality]\label{lemma_Hoeffding's inequality}
Let $ X_1, X_2, \ldots, X_n$ be independent random variables such that $a_i \leq X_i \leq b_i$ almost surely. Consider the sum of these variables by $S_n = \sum_{i=1}^n X_i$ and its expected value by $\mathbb{E}[S_n]$, then Hoeffding's inequality states that for any $t \geq 0$,
\begin{equation*}
\mathbb P(S_n - \mathbb{E}[S_n] \geq t) \leq \exp\bigg(-\frac{2t^2}{\sum_{i=1}^n (b_i - a_i)^2}\bigg).
\end{equation*}
\end{lemma}

\begin{lemma}\label{lemma_sub_beta_mixing}
If the joint sequence $\{(X_{t,T}, Y_{t,T})\}_{t=1}^{T}$ is $\beta$-mixing, then $\{Y_{t,T}\}_{t=1}^{T}$ is also $\beta$-mixing.
\end{lemma}
\begin{proof}
A stochastic process $\{Z_t\}$ is called $\beta$-mixing if its mixing coefficients $\beta(k)$ approach zero as $k \to \infty$, where
\begin{equation*}
\beta(k) = \sup_t \mathbb{E}\bigg[\sup_{B \in \mathcal{G}_t^k} \bigg| \mathbb{P}(B|\mathcal{F}_t) - \mathbb{P}(B) \bigg|\bigg],
\end{equation*}
where $\mathcal{F}_t = \sigma(Z_1, \ldots, Z_t)$ and $\mathcal{G}_t^k = \sigma(Z_{t+k}, Z_{t+k+1}, \ldots)$.

We have that the sigma algebra generated by a subset of variables is a sub-algebra of the sigma algebra generated by the entire, see \cite{halmos2013measure}. As the sequence $\{(X_{t,T}, Y_{t,T})\}$ is $\beta$-mixing, which means
\begin{equation*}
\beta_{XY}(k) = \sup_t \mathbb{E}\bigg[\sup_{B \in \sigma(X_{t+k,T}, Y_{t+k,T}, \ldots)} \bigg|\mathbb{P}(B|\sigma(X_{1,T}, Y_{1,T}, \ldots, X_{t,T}, Y_{t,T})) - \mathbb{P}(B)\bigg|\bigg] \rightarrow 0 \text{ as } k \rightarrow \infty.
\end{equation*}
Applying the sub-$\sigma$-algebra property, the $\sigma$-algebra generated by $\{Y_{t,T}\}$, denoted $\mathcal{F}_t^Y = \sigma(Y_{1,T}, \ldots, Y_{t,T})$, is a sub-$\sigma$-algebra of $\mathcal{F}_t^{XY} = \sigma(X_{1,T}, Y_{1,T}, \ldots, X_{t,T}, Y_{t,T})$. For $\{Y_{t,T}\}$, we can show that:
\begin{equation*}
\beta_Y(k) = \sup_t \mathbb{E}\bigg[\sup_{B \in \sigma(Y_{t+k,T}, \ldots)} \bigg|\mathbb{P}(B|\mathcal{F}_t^Y) - \mathbb{P}(B)\bigg|\bigg] \text{ as } k \rightarrow \infty.
\end{equation*}
Given that $\mathcal{F}_t^Y \subseteq \mathcal{F}_t^{XY}$, any set $B$ in $\sigma(Y_{t+k,T}, \ldots)$ is also in $\sigma(X_{t+k,T}, Y_{t+k,T}, \ldots)$. Thus,
\begin{equation*}
\sup_{B \in \sigma(Y_{t+k,T}, \ldots)}\bigg|\mathbb{P}(B|\mathcal{F}_t^Y) - \mathbb{P}(B)\bigg| \leq \sup_{B \in \sigma(X_{t+k,T}, Y_{t+k,T}, \ldots)} \bigg|\mathbb{P}(B|\mathcal{F}_t^{XY}) - \mathbb{P}(B)\bigg|.
\end{equation*}
It follows that
\begin{equation*}
\beta_Y(k) \leq \beta_{XY}(k),
\end{equation*}
and since $\beta_{XY}(k) \to 0$ as $n \to \infty$, we have $\beta_Y(k) \to 0$ as well, then we have $\{Y_{t,T}\}$ is also $\beta$-mixing.
\end{proof}

\begin{lemma}[Matrix Inversion Lemma (Woodbury matrix identities, \cite{bach2021learning})]
\label{lemma_Matrix_Inversion}
Let $A$ and $D$ be invertible matrices, $B$ and $C$ are matrices of conformable size. The lemma states that the inverse of the matrix $A - B D^{-1} C$ can be expressed as:
\begin{equation*}
(A - B D^{-1} C)^{-1} = A^{-1} + A^{-1} B (D- C A^{-1} B)^{-1} C A^{-1}.  
\end{equation*}
Multiply $B$ on each side of the equation
\begin{align*}
(A - B D^{-1} C)^{-1}B&=A^{-1}B + A^{-1} B (D-C A^{-1} B)^{-1} C A^{-1}B\\
&=A^{-1}B (I+(D-C A^{-1} B)^{-1} C A^{-1}B),
\end{align*}
recognize that
\begin{align*}
I&=(D-C A^{-1} B)^{-1} (D-C A^{-1} B)\\
&= (D-C A^{-1} B)^{-1} D - (D-C A^{-1} B)^{-1}C A^{-1} B,
\end{align*}
Then, we have the classical formulation
\begin{equation*}
(A - BD^{-1}C)^{-1}B = A^{-1}B(D - CA^{-1}B)^{-1}D.
\end{equation*}
\end{lemma}

Lemma~\ref{lemma_Matrix_Inversion} is often applied when $C = B^{\top}$, $A = I$, and $D = -I$, which lead to
\begin{equation*}
(I + BB^\top)^{-1} = I - B(I + B^\top B)^{-1}B^{\top}, 
\end{equation*}
and
\begin{equation*}
(I + BB^\top)^{-1}B = B(I + B^\top B)^{-1}.
\end{equation*}

\begin{lemma}[\cite{hastie2015statistical}]\label{lemma_RE_subgradients_lasso}
A vector $\hat{\bm{\theta}} = [{\hat{\theta}}_{1,\bullet}^{\top}, \cdots, {\hat{\theta}}_{T,\bullet}^{\top}]^{\top} \in \R^{Td}$ is 
an optimum of the objective function~(\ref{estimator}) with $\Omega(\bm{\theta})=\norm{\bm{\theta}}_{1}$ is the total variation penalization, if and only if there is a subgradient 
$\hat g = [\hat g _{r, \bullet}]_{r=1, \ldots, T} \in \partial \norm{\hat {\bm{\theta}}}_{1}$ such that 
\begin{equation*}
\nabla R_{T}(\hat{\theta}_{r,\bullet})+ \lambda{\hat{g}}_{r,\bullet} = \mathbf{0_{d}}, 
\end{equation*}
where
\begin{equation}
\label{subdifferential_of_lasso}
\bigg\{
\begin{array}{ll}
{\hat{g}}_{r,\bullet} = \sgn({\hat{\theta}}_{r,\bullet})  & \text{ if }  r \in  J(\hat{\bm{\theta}}),\\
{\hat{g}}_{r,\bullet} \in  {[-1,+1]}^{d} & \text{ if }  r \in J^{\complement}(\hat{\bm{\theta}}),
\end{array} 
\bigg.
\end{equation}
$J(\hat{\bm{\theta}})$ is the support set of $\hat{\bm{\theta}}$. 
For the problem~(\ref{estimator}), we have 
\begin{equation}
\label{optimality_condition_lasso}
\frac{2}{T} \big(K_{r\bullet}\big)^{\top}\big(K \hat{\bm{\theta}} - Y \big) + \lambda {\hat{g}}_{r,\bullet}= \mathbf{0_{d}}.
\end{equation}
\end{lemma}

\begin{lemma}[\cite{alaya2019binarsity}]\label{lemma_RE_subgradients_TV}
A vector $\hat{\bm{\theta}} = [{\hat{\theta}}_{1,\bullet}^{\top}, \cdots, {\hat{\theta}}_{T,\bullet}^{\top}]^{\top} \in \R^{Td}$ is 
an optimum of the objective function~(\ref{estimator}) with $\Omega_{\lambda}(\bm{\theta})=\norm{\bm{\theta}}_{\TV,\lambda}$ is the total variation penalization, if and only if there is a subgradient 
$\hat g = [\hat g _{r, \bullet}]_{r=1, \ldots, T} \in \partial \norm{\hat {\bm{\theta}}}_{ \TV, \lambda}$ such that 
\begin{equation*}
\nabla R_{T}(\hat{\theta}_{r,\bullet})+{\hat{g}}_{r,\bullet} = \mathbf{0_{d}}, 
\end{equation*}
where
\begin{equation}
\label{subdifferential_of_TVlambda}
\bigg\{
\begin{array}{ll}
{\hat{g}}_{r,\bullet} = D_{r}^{\top} \big(\hat{\lambda}_{j} \odot \sgn(D_{r}
{\hat{\theta}}_{r,\bullet})\big) & \text{ if }  r \in  J(\hat{\bm{\theta}}),\\
{\hat{g}}_{r,\bullet} \in  D_{r}^\top \big( \hat{\lambda}_{j} \odot 
{[-1,+1]}^{d}\big) & \text{ if }  r \in J^{\complement}(\hat{\bm{\theta}}),
\end{array} 
\bigg.
\end{equation}
and $D_{r}$ is defined by~(\ref{equation_D_r}), $J(\hat{\bm{\theta}})$ is the support set of $\hat{\bm{\theta}}$. 
For the problem~(\ref{estimator}), we have 
\begin{equation}
\label{optimality_condition}
\frac{2}{T} \big(K_{r\bullet}\big)^{\top}\big(K \hat{\bm{\theta}} - Y \big) + {\hat{g}}_{r,\bullet}= \mathbf{0_{d}}.
\end{equation}
\end{lemma}

Let us recall the block diagonal matrix $D =diag(D_{1}, \ldots, 
D_{T})$ with $D_{r}$ defined in~(\ref{equation_D_r}). The matrix 
$V$ as the inverse of matrix $D$, i.e., $VD=I$, 
where $V={\rm diag}(V_{1},\ldots,V_{T})$ is the $Td \times Td$ matrix 
with the $(d \times d)$ lower triangular matrix $V_{r}$, and the entries 
$ \Big(V_{r}\Big)_{s, j}=0$ if $ s<j $ and $ \Big(V_{r}\Big)_{s, j}=1$ 
otherwise.
To prove Theorem~\ref{theorem:fast_TV_pen_SubWeibull} and \ref{theorem:fast_TV_pen_ragularvarying}, we need the following results, which give a compatibility property for the matrix $V$.
For any concatenation of subsets $ J=[J_{1}, \ldots, J_{T}],$ we set
\begin{equation}
\label{definition_J_concatenation}
J_{r} = \{\tau_{r}^{1}, \ldots, \tau_{r}^{b_{r}}\} \subset \{1, \ldots, d\}
\end{equation}
for all $r=1, \ldots, T$ with the convention that $\tau_{r}^{0} = 0$ and $\tau_{r}^{b_{r} +1} = d+1$.

\begin{lemma}[\cite{alaya2019binarsity}]
\label{lemma-compatibility-TV}
Let $ \gamma= (\gamma_{1,1}, \dots, \gamma_{1,d}, \dots, \gamma_{T,1}, \dots, \gamma_{T,d})^{\top} \in \mathbb{R}^{Td}_{+}$ be a given vector as the "weights", $\odot$ is the Hadamard product (elementwise) and $J = [J_{1}, \ldots, J_{T}]$ with $J_{r}$ given by~(\ref{definition_J_concatenation}) for all $r=1, \ldots, T$.
Then, for every $\Delta \in \mathbb{R}^{Td} \backslash \{\mathbf{0}\}$, we have 
\begin{equation*}
\frac{\norm{V \Delta}_2}{|\norm{\Delta_{J}\odot \gamma_{J}}_{1} - \norm{\Delta_{J^{\complement}} \odot \gamma_{J^{\complement}}}_1|}  \geq  \kappa_{V,\gamma}(J),
\end{equation*}
where 
\begin{equation*}
\kappa_{V,\gamma}(J) = \bigg\{ 32  \sum_{r=1}^{T}\sum_{j=1}^{d} | \gamma_{r,j} -\gamma_{r,j-1}|^{2}+  2|J_{r}|\norm{\gamma_{r,\bullet}}_\infty^{2}\Lambda_{\min, J_{r}}^{-1}\bigg\}^{-1/2},
\end{equation*}
and $\Lambda_{\min, J_{r}} = \min_{l=1, \ldots b^{r}}| \tau_{r}^{l_{r}} - \tau_{r}^{l_{r} -1 }|$.
\end{lemma}

\begin{lemma}[\cite{alaya2019binarsity}]\label{lemma-compatibility-1_lambda}
Let $ \gamma= (\gamma_{1,1}, \dots, \gamma_{1,d}, \dots, \gamma_{T,1}, \dots, \gamma_{T,d})^{\top} \in \mathbb{R}^{Td}_{+}$ be a given vector as the "weights", $J = [J_{1}, \ldots, J_{T}]$ with $J_{r}$ given by~(\ref{definition_J_concatenation}) for all $r=1, \ldots, T$ and the integer $s$ is an upper bound on the sparsity $J(\bm{\theta})$ of a vector of coefficients $\bm{\theta}$. Then we have
\begin{equation*}
\inf_{\Delta \in S_{1, J_{0}}} \bigg\{ \frac{\|K V \Delta\|_2}{\sqrt{T} \| \Delta_{J} \circ \gamma_{J} \|_{1} - \| \Delta_{J^{\complement}} \circ \gamma_{J^{\complement}} \|_{1}} \bigg\} \geq \kappa_{V, \gamma}(J(\bm{\theta})) \kappa(\bm{K}, J(\bm{\theta})),
\end{equation*}
where
\begin{equation*}
S_{1, J_{0}} = \bigg\{ \Delta \in \mathbb{R}^{Td} \setminus \{0\} \mid \sum_{r=1}^{T} \|( \Delta_{r\bullet} )_{J_{0}^{\complement}}\|_{1,\gamma} \leq 3 \sum_{r=1}^{T} \|( \Delta_{r\bullet})_{J_{0}}\|_{1,\gamma} \bigg\}.
\end{equation*}
\end{lemma}

\end{document}